\documentclass[11pt]{article}

\usepackage[a4paper,margin=1in]{geometry}

\usepackage[utf8]{inputenc}
\usepackage[T1]{fontenc}
\usepackage{lmodern}
\usepackage{microtype}

\usepackage{amsmath}
\usepackage{amssymb}
\usepackage{amsfonts}
\usepackage{amsthm}
\usepackage{mathtools}
\usepackage{bm}
\usepackage{nicefrac}

\usepackage{graphicx}
\usepackage{subcaption}
\usepackage{booktabs}
\usepackage{multirow}
\usepackage{float}
\usepackage{xcolor}

\usepackage{algorithm}
\usepackage{algorithmic}

\usepackage[numbers,sort&compress]{natbib}
\usepackage{url}
\usepackage{etoc}

\usepackage{aliascnt}

\theoremstyle{plain}

\newtheorem{theorem}{Theorem}[section]

\newaliascnt{proposition}{theorem}
\newtheorem{proposition}[proposition]{Proposition}
\aliascntresetthe{proposition}

\newaliascnt{lemma}{theorem}

\aliascntresetthe{lemma}

\newaliascnt{corollary}{theorem}
\newtheorem{corollary}[corollary]{Corollary}
\aliascntresetthe{corollary}

\theoremstyle{definition}

\newaliascnt{definition}{theorem}

\aliascntresetthe{definition}

\newaliascnt{assumption}{theorem}

\aliascntresetthe{assumption}

\theoremstyle{remark}

\newaliascnt{remark}{theorem}
\newtheorem{remark}[remark]{Remark}
\aliascntresetthe{remark}

\usepackage[
    colorlinks=true,
    linkcolor=blue,
    citecolor=blue,
    urlcolor=blue
]{hyperref}

\usepackage[capitalize,noabbrev]{cleveref}

\crefname{proposition}{Proposition}{Propositions}
\crefname{lemma}{Lemma}{Lemmas}
\crefname{corollary}{Corollary}{Corollaries}
\crefname{definition}{Definition}{Definitions}
\crefname{assumption}{Assumption}{Assumptions}
\crefname{remark}{Remark}{Remarks}

\newcommand{\appsection}[2]{%
  \noindent
  \textbf{\hyperref[#2]{#1}}
  \dotfill
  \pageref{#2}
  \par\vspace{0.5em}
}

\newcommand{\appsubsection}[2]{%
  \noindent
  \hspace{1.5em}
  \hyperref[#2]{#1}
  \dotfill
  \pageref{#2}
  \par\vspace{0.3em}
}

\title{%
  \Large\bfseries
  TracingFlow: A Simulation-Free Trajectory Inference Framework\\
  Based on Second-Order Dynamics
}

\author{
\normalsize
  Yuhao Sun\textsuperscript{1}\thanks{Equal contribution} ,
  Zekun Wu\textsuperscript{2}\footnotemark[1] ,
  Zixun Huang\textsuperscript{2} ,
  {Peijie Zhou}\textsuperscript{1,3,4,5}
  \thanks{Corresponding author: \texttt{pjzhou@pku.edu.cn}}
  \\
  \\
  \normalsize\textsuperscript{1}Center for Machine Learning Research, Peking University
  \\
  \normalsize\textsuperscript{2}School of Mathematical Sciences, Peking University
  \\
  \normalsize\textsuperscript{3}Center for Quantitative Biology, Peking University
  \\
  \normalsize\textsuperscript{4}National Engineering Laboratory for Big Data Analysis and Applications, Beijing
  \\
  \normalsize\textsuperscript{5}AI for Science Institute, Beijing
  \\
}

\date{}

\begin{document}

\maketitle

\begingroup
\renewcommand{\thefootnote}{}
\footnotetext{

  \medskip
\textit{Preprint. August 21, 2026.}
}
\addtocounter{footnote}{-1}
\endgroup

\begin{abstract}
  Inferring continuous system evolution from sparse temporal snapshots is a key challenge in generative modeling and single-cell omics. While Optimal Transport (OT) is popular, existing frameworks are largely restricted to first-order dynamics, assuming memoryless velocity fields. This limits expressiveness, as first-order systems fail to account for regulatory momentum and time-delayed responses inherent in processes like cell differentiation. Here, we introduce TracingFlow, a simulation-free Flow Matching framework generalizing to second-order dynamics. By using neural networks to regress the acceleration field, TracingFlow provides an exact, efficient solution to the Dynamical Optimal Acceleration Transport (DOAT) problem. Unlike first-order methods yielding over-smoothed trajectories, our second-order formulation captures high-curvature transitions and nonlinear evolutions by learning the underlying force fields. Evaluated on complex synthetic and large-scale scRNA-seq datasets, TracingFlow achieves superior accuracy in distributional reconstruction and trajectory faithfulness. Moreover, by integrating lineage tracing priors, it recovers dynamical structures that are both mathematically optimal and biologically plausible.
\end{abstract}

\section{Introduction}








Recovering underlying dynamics from discrete observations is a pivotal task in single-cell omics, known as Trajectory Inference (TI) \citep{chen2018neural,zheng2017massively}. To identify the least costly evolution between distributions, Optimal Transport (OT) frameworks have emerged \citep{bunne2024optimal,zhang2025review}. These generally categorize into Static OT methods \citep{waddingot,moscot,halmos2025dest}, which learn direct mappings, and Dynamical OT methods \citep{trajectorynet,mioflow,scnode,peng2024stvcr}, which capture continuous evolution via flow maps.

Early Dynamical OT frameworks relied on Neural ODEs \citep{chen2018neural}, incurring high computational costs due to numerical integration. To address this, Flow Matching \citep{cfm_lipman,liu2022flow,pooladian2023multisample} was proposed as a simulation-free alternative, regressing velocity fields to push source distributions to targets. By decomposing transport costs into single-particle trajectories, Flow Matching efficiently solves Dynamical OT problems \citep{cfm_tong,klein2024genot,rohbeck2025modeling}, including its various variants \citep{sflowmatch,eyring2024unbalancedness,cao2025taming,WFRFlowMatching} .

However, most TI frameworks assume first-order dynamics ($\boldsymbol{\dot x} = \boldsymbol{v} (\boldsymbol{x}, t)$), inherently restricting $\boldsymbol{v}$ to be single-valued w.r.t $\boldsymbol{x}$. This limits expressiveness when modeling complex biological priors, such as lineage tracing \citep{Kretzschmar2012, Mao2025}, potentially yielding dynamics that contradict with additional biological prior information. Furthermore, \citep{gorin2020protein, li2023multi, hong2025multi} pointed out that if more modalities such as chromatin accessibility or proteomics are considered, the biological kinetics may follow higher-order dynamics or could be modelled in augmented space. While 3MSBM \citep{theodoropoulos2025momentum} recently introduced a second-order Momentum Schrödinger Bridge for smoother trajectories, it relies on iterative retraining the acceleration field similar to Rectified Flow \citep{liu2022flow} to determine optimal couplings. Furthermore, it employs a heuristic method to estimate the initial velocity, rather than incorporating this estimation as part of the optimal control problem.

To overcome these limitations, we introduce \textbf{TracingFlow}, a simulation-free framework utilizing second-order dynamical systems. We propose the Dynamical Optimal Acceleration Transport (DOAT) problem, which determines evolution by minimizing acceleration costs. TracingFlow directly regresses the acceleration field and initial velocity, avoiding ODE simulation. Our experiments show that TracingFlow achieves competitive and often improved reconstruction accuracy relative to existing simulation-free frameworks \citep{cfm_tong, sflowmatch, hierachicalRFM, park2024constant}, while effectively incorporating biological priors. Our contributions include:

\begin{itemize}
    \item We propose TracingFlow, a simulation-free TI framework based on second-order dynamics. It solves the proposed DOAT problem by regressing acceleration fields,  reducing computational costs compared to ODE-based methods.
    \item We provide theoretical guarantees by decoupling the design of the optimal transport plan and single-particle control. We further introduce an iterative strategy to transform the real-world Velocity Missed-DOAT (VM-DOAT) problem into a standard DOAT problem, enabling approximate solutions.
    \item We demonstrate TracingFlow's effectiveness on multi-time-point real world datasets. Results indicate superior distribution reconstruction accuracy and enhanced capability in preserving biological priors compared to state-of-the-art methods.
\end{itemize}

\begin{figure}
    \centering
    \includegraphics[width=0.5\linewidth]{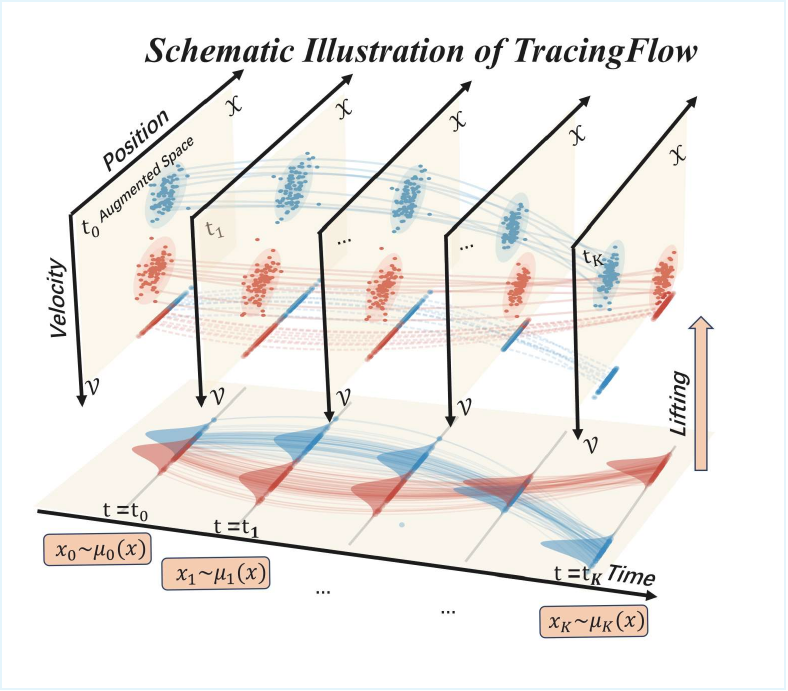}
    \caption{An illustration figure for TracingFlow}
    \label{fig:TracingFlow}
\end{figure}

\section{Related Works}

\paragraph{Solving Optimal Transport via Flow Matching.} 
Flow Matching \citep{cfm_lipman,liu2022flow,pooladian2023multisample,albergo2022building} transports probability distributions via learned flow maps, offering scalable simulation-free training. As an optimal control problem, Optimal Transport (OT)~\cite{peyre2019computational} decomposes into Optimal Coupling and single-particle geodesics, facilitating solutions via Flow Matching \citep{cfm_tong,klein2024genot,rohbeck2025modeling,eyring2024unbalancedness}. This approach extends to various variants and generalizations \citep{sflowmatch, eyring2024unbalancedness, cao2025taming, corso2025composing, wang2025joint, WFRFlowMatching, kapusniak2024metric, zhang2024trajectory, meta_flow_matching, petrovic2025curly}. However, these typically employ first-order dynamics. While second-order approaches have emerged \citep{theodoropoulos2025momentum}, they require iterative training of the acceleration field to achieve optimal coupling. Moreover, they rely on heuristic methods for initial velocity estimation, rather than integrating it as a component of the optimal control problem. TracingFlow addresses this by decomposing the Dynamical Acceleration Optimal Transport (DOAT) problem into optimal coupling and single particle optimal trajectory, achieving a truly simulation-free process for DOAT.

\paragraph{Overcoming Single-Valued Velocity Constraint in Flow Matching.} 
In standard Flow Matching, the velocity field is modeled as a function solely dependent on $\boldsymbol{x}$ and $t$. This single-valued dependency prevents the model from representing intersecting trajectories with distinct velocities, often resulting in curved inference paths and increased computational costs. To decouple the velocity from the current position, the velocity network can incorporate additional inputs like class labels \citep{switchedFM,diversifiedFM}, initial positions \citep{augmentedFM}, or hidden states \citep{variationalRFM}. Additionally, \citep{hierachicalRFM} employs hierarchical generation to mitigate this issue. However, these methods do not extend to second-order dynamics. While \citep{park2024constant} utilized a second-order system, it is limited to constant-acceleration paths. In contrast, TracingFlow offers a flexible second-order setting. By designing minimal-cost paths, it effectively resolves the limitation of the single-valued velocity field.

\paragraph{Single-cell Trajectory Inference and Lineage Integration.} 
Optimal Transport (OT) is a robust framework for inferring cellular dynamics from scRNA-seq data \citep{waddingot, moscot, trajectorynet, mioflow, action_matching,tong_action,albergo2023stochastic,Linwei2024governing,maddu2024inferring,PRESCIENT,GeofftrajectoryMFLP,GeofftrajectoryVentre,gwot,shi2024diffusion,dyn_sb_koshizuka2023neural,bunne_dynam_SB,chen2022likelihood,jiang2024physics,DeepRUOT,sun2025variational, peng2024stvcr, yang2025topological,TIGON, chen2019multi}.
However, transcriptomic similarity alone often fails to resolve complex trajectories. To mitigate this, studies incorporate lineage tracing priors such as clonal barcode information. Existing approaches integrate such info via regularization \citep{Forrow2021, prasad2020}, structural alignment \citep{Lange2024}, velocity mapping \citep{Wang2023}, or sparse transition modeling \citep{Wang2022, Guo2025, Gao2025}. While improving accuracy, these generally focus on discrete couplings or static matrices. TracingFlow advances this by embedding priors into a continuous OT-based second-order flow matching framework, capturing complex dynamics consistent with lineage information.

\section{Mathematical Background}

In this section, we provide the mathematical formulation of the problem addressed by TracingFlow.

\paragraph{Dynamical Optimal Acceleration Transport (DOAT) Problem}

Let $\mathcal{X}\subset\mathbb{R}^d$ and $\mathcal{V}\subset\mathbb{R}^d$ denote the \emph{position} and \emph{velocity} spaces, respectively. 
  We define the \emph{augmented space} as $\mathcal{S}\coloneqq \mathcal{X}\times\mathcal{V}$, so that the full system state is $\bm s=(\bm x,\bm v)\in\mathcal S$.
 Consider the second-order dynamic system:
\begin{equation}\label{dyn}
        \dot{\bm {x}}=\bm {v} \quad 
        \dot{\bm {v}}=\bm {a}(\bm {x}, \bm {v}, t), \quad t\in[t_0, t_K].
\end{equation}
Modeling a large ensemble of such particles via a probability density $\rho_t$ in $\mathcal{S}$, we follow \cite{benamou2019second, chen2019multi} to formulate the following optimal control problem:
\begin{equation}
    \label{eq:doat_problem}
    \small{\min_{\boldsymbol{a}, \rho} \mathcal{J}_{\text{DOAT}} (\boldsymbol{a}, \rho) =  \int_{t_0}^{t_K } \int_{\mathcal{S}} \frac{1}{2} \|\boldsymbol{a}(\boldsymbol{x}, \boldsymbol{v},t) \|^{2} \rho_{t}(\boldsymbol{x}, \boldsymbol{v}) \, \mathrm{d} \boldsymbol{x} \mathrm{d} \boldsymbol{v} \mathrm{d} t }
\end{equation}
\begin{align}
    \text{s.t.} \quad & \partial_{t} \rho_{t} + \nabla_{\boldsymbol{x} }\cdot (\boldsymbol{v} \rho_{t}) +  \nabla_{\boldsymbol{v}}\cdot( \boldsymbol{a} \rho_{t}) =  0, \label{eq:continuity_eq} \\
    & \rho_{t_{i}}(\boldsymbol{x}, \boldsymbol{v}) =  \mu_{i}(\boldsymbol{x}, \boldsymbol{v}), \quad i = 0, \dots ,K.
\end{align}
Here, constraints are imposed at $K+1$ timestamps $t_i$ with given densities $\mu_i$. \cref{eq:continuity_eq} represents the continuity equation in the augmented space. Analogous to standard Dynamical Optimal Transport \citep{peyre2019computational}, we term this the \emph{Dynamical Optimal Acceleration Transport (DOAT)} problem. To facilitate the subsequent numerical solution via particle methods, we assume the existence of an optimal flow map.

\paragraph{Velocity-Missed Dynamical Optimal Acceleration Transport (VM-DOAT) Problem}

However, in practical applications such as single-cell datasets, only particle positions are observable, rendering velocities as unknown quantities. This presents a discrepancy with the DOAT formulation. 

Consequently, TracingFlow addresses a relaxation of the problem: 
\begin{equation}
    \small{\min_{\boldsymbol{a}, \rho} \mathcal{J}_{\text{VM-DOAT}} (\boldsymbol{a}, \rho) =  \int_{t_0}^{t_K} \int_{\mathcal{S}} \dfrac{1}{2} \|\boldsymbol{a}(\boldsymbol{x}, \boldsymbol{v},t) \|^{2} \rho_{t}(\boldsymbol{x}, \boldsymbol{v}) \, \mathrm{d} \boldsymbol{x} \mathrm{d} \boldsymbol{v} \mathrm{d} t}
\end{equation}
\begin{align}
 \text{s.t.}  & \quad \partial_{t} \rho_{t}(\boldsymbol{x}, \boldsymbol{v}) + \nabla_{\boldsymbol{x} } \cdot(\boldsymbol{v} \rho_{t}) +  \nabla_{\boldsymbol{v}}\cdot( \boldsymbol{a} \rho_{t}) =  0,  \\  & \int  \rho_{t_{i}}(\boldsymbol{x},   \boldsymbol{v}) \, \mathrm{d} \boldsymbol{v} =  \mu_{i}^{(\text{pos})}(\boldsymbol{x}), \ i = 0,1, \dots ,K. 
\end{align}
Here, the constraints are imposed only on the marginal distribution of positions, $\int \rho_{t_{j}} (\boldsymbol{x}, \boldsymbol{v}) \, \mathrm{d} \boldsymbol{v}$, at each time point, with no constraints on the velocity distribution. We refer to this problem as \emph{Velocity Missed-Dynamical Optimal Acceleration Transport (VM-DOAT)}.

\section{Methodology of TracingFlow }

In this section, we first introduce the key properties of the DOAT problem that are essential for our algorithm design. Subsequently, drawing inspiration from Flow Matching approaches for standard Dynamical OT, we devise an algorithm to exactly solve the DOAT problem by combining Optimal Coupling with single-particle optimal trajectories. Finally, we introduce a global path-based velocity completion scheme that effectively bridges the gap between VM-DOAT and the standard DOAT framework. By inferring missing velocities through the lens of global trajectory consistency, our approach enables the application of optimal transport theory to solve the VM-DOAT problem in a principled manner.

\subsection{Key Properties of DOAT Problem}

\label{article:key_properties_of_doat}

\paragraph{Optimal Single Particle Trajectory} 
The DOAT problem is an optimal control problem over distributions. We begin by considering the corresponding single-particle optimal control problem. Specifically, the optimal trajectory is a \emph{cubic interpolation} determined by the boundary states.

\begin{proposition}
\label{proposition: optimal_traj_of_particle}
Let $\gamma:[0,T]\to\mathbb{R}^d$ be a twice continuously differentiable curve satisfying the boundary conditions $\gamma(0)=\bm{x}_0$, $\gamma'(0)=\bm{v}_0$, $\gamma(T)=\bm{x}_T$, and $\gamma'(T)=\bm{v}_T$. The variational problem
\begin{equation}
\label{eq:single_particle_problem}
    \min_{\gamma} \int_0^T \frac{1}{2} \|\gamma''(t)\|_2^2\,\mathrm{d} t
\end{equation}
admits a unique minimizer, which is a cubic polynomial of the form
\begin{equation}
\label{eq:single_particle_minimizer}
    \gamma(t)=c_3 t^3 + c_2 t^2 + c_1 t + c_0,
\end{equation}
where the coefficients are detailed in \cref{app:proof_optimal_traj_of_particle}.This result is got by \emph{Euler-Lagrange} formula. See a detailed proof in \cref{app:proof_optimal_traj_of_particle}.
\end{proposition}

\paragraph{Optimal Coupling} 
By utilizing the single-particle optimal trajectory, we can immediately obtain the minimum cost for transporting a particle from $(\boldsymbol{x}_{0}, \boldsymbol{v}_{0})$ to $(\boldsymbol{x}_{T}, \boldsymbol{v}_{T})$:

\begin{corollary}\label{cost}
    Substitute \cref{eq:single_particle_minimizer} to \cref{eq:single_particle_problem}, we get the minimum cost :
    \begin{align}
    \label{eq:single_particle_cost}
    \mathcal{C}_{[0, T]} & =  \frac{2}{T^3} \bigg\{  (\|\bm{v}_0\|^2+\langle \bm{v}_0,\bm{v}_T\rangle+ \|\bm{v}_T\|^2)T^2   + 3\langle \bm{v}_0+\bm{v}_T, \bm{x}_0-\bm{x}_T\rangle T+3\|\bm{x}_0-\bm{x}_T\|^2 \bigg\}
 \end{align}
\end{corollary}

The minimum cost allows us to define an optimal transport problem between $t=0$ and $t=T$:
\begin{equation}
\label{eq:SOAT_problem}
\pi^*_{[0,T]} = \min_{\pi_{[0,T]}} \int  \mathcal{C}_{[0,T]} [(\boldsymbol{x}_{0} , \boldsymbol{v}_{0}), (\boldsymbol{x}_{T} , \boldsymbol{v}_{T})] \mathrm{d} \pi _{[0,T]}
\end{equation}
where $\pi_{[0,T]}[(\boldsymbol{x}_{0} , \boldsymbol{v}_{0} ), (\boldsymbol{x}_{T}, \boldsymbol{v}_{T})]$ is any coupling satisfying the constraints $\int \pi_{[0,T]} \mathrm{d} \boldsymbol{x}_{0} \mathrm{d}\boldsymbol{v}_{0} = \mu_{T}(\boldsymbol{x}, \boldsymbol{v})$ and $\int \pi_{[0,T]} \mathrm{d} \boldsymbol{x}_{T} \mathrm{d}\boldsymbol{v}_{T} = \mu_{0}(\boldsymbol{x}, \boldsymbol{v})$ , $\mu_{0}(\boldsymbol{x}, \boldsymbol{v})$ and $\mu_{T}(\boldsymbol{x}, \boldsymbol{v})$ are two distributions in augmented space. We term this formulation the \textbf{Static Optimal Acceleration Transport (SOAT)} problem.

It is worth noting that while DOAT permits crossings in the projected position space, it prohibits trajectory crossing in the augmented space, which guarentees the neatness of trajectories. The following remark provides the mathematical justification for this non-crossing property in the augmented space, demonstrating that any such intersection implies a strictly higher transport cost.

\begin{remark}
\label{thm:non_collision}
    Consider two pairs of points in $\mathbb{R}^{2d}$, $(\bm{x}_0,\bm{v}_0), (\bm{x}_T, \bm{v}_T )\text{ and } (\bm{\tilde{x}}_{0}, \bm{\tilde{v}}_{0}), (\bm{\tilde{x}}_T, {\bm{\tilde{v}}}_T)$ with no points overlapping. Denote the cubic interpolation curves between each pair by $\gamma(t), {\tilde\gamma}(t) (t\in[0, T])$, respectively. If the curves cross with each other at time point $t_0\in (0, T)$, then we have
    \begin{equation*}
    \begin{aligned}
        \mathcal{C}_{[0, T]}(\bm{x}_0, \bm{v}_0, \bm{x}_T, \bm{v}_{T})+\mathcal{C}_{[0, T]}(\bm{\tilde{x}}_0, \bm{\tilde{v}}_0, \bm{\tilde{x}}_T, \bm{\tilde{v}}_{T})>
        \mathcal{C}_{[0, T]}(\bm{x}_{0}, \bm{v}_{0}, \bm{\tilde{x}}_{T}, \bm{\tilde{v}}_{T})+\mathcal{C}_{[0, T]}(\bm{\tilde{x}}_{0}, \bm{\tilde{v}}_{0}, \bm{x}_T, \bm{v}_T).
    \end{aligned}
    \end{equation*}
    See a detailed proof in ~\cref{app:Non_collision}.
\end{remark}

\cref{eq:doat_problem} represents a multi-marginal optimal transport problem over the interval $[t_{0}, t_{K}]$. To address this, we formulate sequential SOAT problems between adjacent time steps $t_i$ and $t_{i+1}$ by applying the substitutions $\boldsymbol{x}_{0}\rightarrow \boldsymbol{x}_{i}$, $\boldsymbol{x}_{T}\rightarrow \boldsymbol{x}_{i+1}$, $\boldsymbol{v}_{0}\rightarrow \boldsymbol{v}_{i}$, $\boldsymbol{v}_{T}\rightarrow \boldsymbol{v}_{i+1}$, and $T \rightarrow t_{i+1} - t_{i}$ to \cref{eq:SOAT_problem}. For simplicity, let $\mathcal{C}_{i \rightarrow i+1}$ denote the resulting pairwise cost and $\pi^*_{i \rightarrow i+1}$ the optimal plan. Due to the Markovian nature of the dynamics, the full multi-marginal plan decomposes into the sequence of these local plans $\{{\pi^{*}_{i \rightarrow i+1}}\}$ for $i=0, \dots, K-1$.

\subsection{Conditional and Marginal Probability Path}

\label{article:acc_and_prob_path_construct}

\paragraph{Obtaining Marginal Probability Path via Mixing Conditional Probability Paths} 
Similar to the approach in standard Flow Matching \cite{cfm_lipman, cfm_tong}, directly constructing the Marginal Probability Path is difficult. Therefore, we condition on $\boldsymbol{z}$ and define the Marginal Probability Path as a mixture of Conditional Probability Paths:
\begin{equation}
\label{eq:marginal_prob_path}
    \rho_{t}(\boldsymbol{x}, \boldsymbol{v}) = \int \rho_{t}(\boldsymbol{x}, \boldsymbol{v}|\boldsymbol{z}) q(\boldsymbol{z}) \mathrm{d} \boldsymbol{z}
\end{equation}
The condition $\boldsymbol{z}$ consists of $K+1$ points selected respectively from the snapshot data at $K+1$ time points $t=0, 1, \cdots, K$. If $\rho_{t}(\boldsymbol{x}, \boldsymbol{v}|\boldsymbol{z})$ is generated by a conditional acceleration field $\boldsymbol{a}(\boldsymbol{x}, \boldsymbol{v}, t|\boldsymbol{z})$ starting from the initial condition $\rho_{0}(\boldsymbol{x}, \boldsymbol{v}|\boldsymbol{z})$, then the marginal acceleration field:
\begin{equation}
    \boldsymbol{a}(\boldsymbol{x}, \boldsymbol{v}, t) := \mathbb{E}_{q(\boldsymbol{z})}\left\{ \dfrac{\boldsymbol{a}(\boldsymbol{x}, \boldsymbol{v}, t|\boldsymbol{z}) \rho_{t}(\boldsymbol{x}, \boldsymbol{v}|\boldsymbol{z})}{\rho_{t}(\boldsymbol{x}, \boldsymbol{v})} \right\}
    \label{eq:marginal_acc_field}
\end{equation}
will generate the marginal probability path $\rho_{t}(\boldsymbol{x}, \boldsymbol{v})$.

\begin{theorem}
The marginal acceleration field in Eq.~\eqref{eq:marginal_acc_field} generates the marginal probability path. See the proof in \cref{app:Marginal}.
\end{theorem}

\paragraph{Equivalence between Regressing Conditional Acceleration Fields and Marginal Acceleration Fields} 
Assuming the marginal acceleration field $\boldsymbol{a}(\boldsymbol{x}, \boldsymbol{v}, t)$ is known and we can sample from the marginal probability path $\rho_{t}(\boldsymbol{x}, \boldsymbol{v})$, we can directly regress $\boldsymbol{a}(\boldsymbol{x}, \boldsymbol{v}, t)$ using a neural network. Let $\boldsymbol{a}_{\boldsymbol{\theta}}(\cdot, \cdot, \cdot): \mathbb{R}^{d} \times \mathbb{R}^{d} \times \mathbb{R}^{1} \rightarrow \mathbb{R}^{d}$ be a time-dependent acceleration field parameterized by a neural network with parameters $\theta$. The acceleration flow matching loss is
\begin{equation*}
\begin{split}
    \mathcal{L}_{\text{AFM}} = \mathbb{E}_{(\boldsymbol{x}, \boldsymbol{v}) \sim \rho_{t}(\boldsymbol{x}, \boldsymbol{v}), t \sim \mathcal{U}[t_{0}, t_{K}]} 
    \left\{ \| \boldsymbol{a}_{\boldsymbol{\theta} }- \boldsymbol{a}(\boldsymbol{x}, \boldsymbol{v}, t) \|^{2} \right\}
\end{split}
\end{equation*}
However, $\boldsymbol{a}(\boldsymbol{x}, \boldsymbol{v}, t)$ is generally intractable because it is defined via an expectation (integral), and its denominator $\rho_{t}(\boldsymbol{x}, \boldsymbol{v})$ also requires integration to compute. In contrast, the conditional acceleration field $\boldsymbol{a}(\boldsymbol{x}, \boldsymbol{v}, t|z)$ and the conditional probability path $\rho_{t}(\boldsymbol{x}, \boldsymbol{v}|z)$ have simple forms. Therefore, in actual training, we use the following Conditional Acceleration Flow Matching (CAFM) objective:
\begin{equation*}
\begin{split}
    \mathcal{L}_{\text{CAFM}} = \mathbb{E}_{\substack{(\boldsymbol{x}, \boldsymbol{v}) \sim \rho_{t}(\boldsymbol{x}, \boldsymbol{v}|z), \\ t \sim \mathcal{U}[t_{0}, t_{K}], z \sim q(z)}} \Big[ \| \boldsymbol{a}_\theta(\boldsymbol{x}, \boldsymbol{v}, t) 
    - \boldsymbol{a}(\boldsymbol{x}, \boldsymbol{v}, t|z) \|^{2} \Big]
\end{split}
\end{equation*}
These two objectives are equivalent in training, as described by the following theorem:

\begin{theorem}
If $\rho_{t}(\boldsymbol{x}, \boldsymbol{v}) > 0$ for any $\boldsymbol{x} \in \mathbb{R}^{d}, \boldsymbol{v} \in \mathbb{R}^{d}, t \in [t_{0}, t_{K}]$, then $\mathcal{L}_{\text{AFM}}$ and $\mathcal{L}_{\text{CAFM}}$ differ only by a constant independent of $\boldsymbol{\theta}$. In other words, $\nabla_{\boldsymbol{\theta}}\mathcal{L}_{\text{AFM}} = \nabla_{\boldsymbol{\theta}}\mathcal{L}_{\text{CAFM}}$. The proof is left to \cref{app:AFM_and_CAFM}.
\end{theorem}

\paragraph{Design of the Conditional Probability Path}  
Consider the joint distributions $\mu_{0}(\boldsymbol{x}, \boldsymbol{v})$, $\mu_{1}(\boldsymbol{x}, \boldsymbol{v})$, ..., $\mu_{K}(\boldsymbol{x}, \boldsymbol{v})$ at $K+1$ time points. The optimal transport plans calculated via solving SOAT problem  are $\pi^*_{0 \rightarrow 1}, \pi^*_{1 \rightarrow 2}, \cdots, \pi^*_{K-1 \rightarrow K}$. Taking the condition variable $\boldsymbol{z} = [(\boldsymbol{x}_{0}, \boldsymbol{v}_{0}), \cdots, (\boldsymbol{x}_{K}, \boldsymbol{v}_{K})]$, and similar to standard Flow Matching, we choose a path with time-varying mean and invariant variance as the conditional probability path, i.e.:
\begin{equation}
\begin{split}
    \rho_{t}(\boldsymbol{x}, \boldsymbol{v}|z) = \mathcal{N}(\boldsymbol{x}|\boldsymbol{m}_{x}(t), \sigma_{x}^{2}I) 
    \cdot \mathcal{N}(\boldsymbol{v}| \boldsymbol{m}_{v}(t), \sigma_{v}^{2}I)
\end{split}
\end{equation}
Here, $\boldsymbol{m}_{x}, \boldsymbol{m}_{v}$ are given by the Minimum Cost Trajectory described in \cref{proposition: optimal_traj_of_particle}, where for $t \in [t_{i}, t_{i+1}]$:
\begin{align*}
    \boldsymbol{m}_{x}(t) = c_{3}(t-t_{i})^{3} + c_{2}(t-t_{i})^{2} + c_{1}(t-t_{i}) + c_{0} \quad 
    \boldsymbol{m}_{v}(t)  3c_{3}(t-t_{i})^{2} + 2c_{2}(t-t_{i}) + c_{1}
\end{align*}
,  the coefficients are identical to those given in \cref{app:proof_optimal_traj_of_particle}, with $\boldsymbol{x}_{0}, \boldsymbol{v}_{0}$ replaced by $\boldsymbol{x}_{i}, \boldsymbol{v}_{i}$, $\boldsymbol{x}_{T}, \boldsymbol{v}_{T}$ replaced by $\boldsymbol{x}_{i+1}, \boldsymbol{v}_{i+1}$, and $T$ replaced by $t_{i+1} - t_{i}$.
Based on the optimal transport plans, we define the optimal propagator (transition kernel) from time $t_{i}$ to $t_{i+1}$ as the probability of a particle at $(\boldsymbol{x}_{i}, \boldsymbol{v}_{i})$ at time $t_{i}$ being transported to $(\boldsymbol{x}_{i+1}, \boldsymbol{v}_{i+1})$ at time $t_{i+1}$:
\begin{equation}
    \mathcal{K}^*_{i \rightarrow i+1}[(\boldsymbol{x}_{i}, \boldsymbol{v}_{i}), (\boldsymbol{x}_{i+1}, \boldsymbol{v}_{i+1})] = \frac{\pi^*_{i \rightarrow i+1}[(\boldsymbol{x}_{i}, \boldsymbol{v}_{i}), (\boldsymbol{x}_{i+1}, \boldsymbol{v}_{i+1})]}{\mu_{i}(\boldsymbol{x}_{i}, \boldsymbol{v}_{i})}
\end{equation}

Then,
\begin{equation*}
    \boldsymbol{z} \sim \mu_{0}(\boldsymbol{x}_{0}, \boldsymbol{v}_{0}) \cdot \prod_{i=0}^{K-1} \mathcal{K}^*_{i \rightarrow i+1}[(\boldsymbol{x}_{i}, \boldsymbol{v}_{i}), (\boldsymbol{x}_{i+1}, \boldsymbol{v}_{i+1})]
\end{equation*}

In the following, we demonstrate that the constructed marginal probability path recovers the true distribution at every given time point.

\begin{theorem}
\label{thm:marginal_reconstruct}
When $\sigma_{x} \rightarrow 0, \sigma_{v} \rightarrow 0$, the marginal probability density constructed by \cref{eq:marginal_prob_path} satisfies $\rho_{t_{i}}(\boldsymbol{x}, \boldsymbol{v}) \rightarrow \mu_{i}(\boldsymbol{x}, \boldsymbol{v})$ at the $K+1$ given time points. Therefore, the Marginal Probability Path constructed by mixing Conditional Probability Paths can correctly reconstruct the marginal distributions. See  the proof in \cref{app:Marginal_prob_path}.
\end{theorem}

Furthermore, we demonstrate that the design yields an exact solution to the DOAT problem.

\begin{proposition}
    As $\sigma_{x} \rightarrow 0$ and $\sigma_{v} \rightarrow 0$, the marginal probability path and the marginal acceleration field will converge to the solution to the DOAT problem. See the proof in \cref{app:exactly_solve}.
\end{proposition}

\subsection{Transform VM-DOAT  to DOAT }

\label{article:VM-DOAT to DOAT}

\paragraph{Relations between DOAT and VM-DOAT} 

The optimal transport cost of the DOAT problem is a functional of the sequence of joint distributions $\mu_{i}(\boldsymbol{x}, \boldsymbol{v})$ for $i = 0, \dots, K$. Formally, we denote this dependency as $\mathcal{J}_{\text{DOAT}}[\{\mu_{i}(\boldsymbol{x}, \boldsymbol{v})\}|_{i=0}^{K}]$. 
In the VM-DOAT setting, we only have access to the marginal spatial distributions $\mu_{i}^{(\text{pos})}(\boldsymbol{x}) = \int \mu_{i}(\boldsymbol{x}, \boldsymbol{v})\mathrm{d}\boldsymbol{v}$. Consequently, the cost of VM-DOAT, denoted as $\mathcal{J}_{\text{VM-DOAT}}[\{\mu_{i}^{(\text{pos})}(\boldsymbol{x})\}|_{i=0}^{K}]$, relates to $\mathcal{J}_{\text{DOAT}}$ through the following minimization:
\begin{equation}
\mathcal{J}_{\text{VM-DOAT}} [\{\mu_{i}^{(\text{pos})}(\boldsymbol{x})\}] =\min_{\omega_{i} (\boldsymbol{v}| \boldsymbol{x})} \mathcal{J}_{\text{DOAT}}[\{\mu_{i}^{(\text{pos})}(\boldsymbol{x}) \cdot \omega_{i} (\boldsymbol{v} | \boldsymbol{x})\}],
\end{equation}
where $\omega_i(\boldsymbol{v} | \boldsymbol{x})$ represents the conditional distribution of velocity, such that the joint distribution factorizes as $\mu_i(\boldsymbol{x}, \boldsymbol{v}) = \mu_i^{(\text{pos})}(\boldsymbol{x}) \cdot \omega_i(\boldsymbol{v} | \boldsymbol{x})$. Thus, our objective is to determine the optimal $\omega_{i}(\boldsymbol{v}|\boldsymbol{x})$ that minimizes the expression above, thereby transforming the VM-DOAT problem into a solvable DOAT instance. 

This transformation rests on two key properties: 
1) The SOAT problem can be solved efficiently using the Sinkhorn method. Once the Transport Plan is computed, we can sample a trajectory $[(\boldsymbol{x}_{0}, \boldsymbol{v}_{0}),\cdots ,(\boldsymbol{x}_{K}, \boldsymbol{v}_{K})]$; 
2) Given the positions $\boldsymbol{x}_{0: K}$ of a trajectory, consider the following multi-marginal optimal control problem:
\begin{equation}
    \label{eq:multi_marginal_oc}
    \min \int_{t_0}^{t_K} \dfrac{1}{2} \|\boldsymbol{\gamma}'' (t) \|^2  \mathrm{d} t , \quad \text{s.t. } \boldsymbol{\gamma}(t_i)  = \boldsymbol{x}_i, \ i =  0 ,\dots, K.
\end{equation}
The velocities of the optimal trajectory, denoted as $V = [\boldsymbol{\gamma}'(t_0),\cdots, \boldsymbol{\gamma}'(t_K) ]$, can be obtained by solving a sparse linear system, as illustrated in the following proposition: 

\begin{proposition}
    \label{prop:cubic_spline}
    The velocities $V$ defined above are the solution to the sparse linear system $VA = B$,
    where $A \in \mathbb{R}^{(K+1)\times(K+1)}$ is a tridiagonal matrix and $B \in \mathbb{R}^{d\times(K+1)}$. A detailed derivation and the explicit expressions for $A$ and $B$ are provided in \cref{app:velocity}.
\end{proposition}
It is worth noting that solving this linear system requires $\mathcal{O}(K)$ time, ensuring high computational efficiency.

\paragraph{Iterative Scheme to Perform the Transformation} 
Based on this decomposition, we employ an iterative scheme to estimate the conditional velocity distributions across $t_{0}, \dots, t_{K}$. We initialize $\omega_{i}^{(0)} (\boldsymbol{v}|\boldsymbol{x}) = \delta(\boldsymbol{v} - \boldsymbol{0})$ and subsequently alternate between the following two steps:

\begin{itemize}
    \item \emph{Solve SOAT Plans}: Compute the $K$ optimal SOAT plans between the distributions $\mu_{i}^{(\text{pos})}(\boldsymbol{x}) \cdot \omega^{(m)}_{i}(\boldsymbol{v}|\boldsymbol{x})$ for adjacent time points. This yields the plans $\pi^{*(m)}_{i \rightarrow i+1}[(\boldsymbol{x}_{i}, \boldsymbol{v}_{i}), (\boldsymbol{x}_{i+1}, \boldsymbol{v}_{i+1})]$ for $i =0 ,\dots, K-1$, where $m$ denotes the iteration index. This step effectively provides the multi-marginal optimal transport plan given the current $\omega^{(m)}_{i}(\boldsymbol{v}|\boldsymbol{x})$.
    
    \item \emph{Update Velocity Distributions}: Update the conditional velocity distributions based on the computed SOAT plans:
    \begin{equation}
    \label{eq:velocity_iter_2}
    \omega^{(m+1)}_{i}(\boldsymbol{v} | \boldsymbol{x}_i) = \frac{\int q^{(m)}(\boldsymbol{z}) \delta(\boldsymbol{\hat v}_{i}[\boldsymbol{x}_0\cdots \boldsymbol{x}_K] - \boldsymbol{v}) \, \mathrm{d}\boldsymbol{z}_{\setminus i}}{\int q^{(m)}(\boldsymbol{z}) \, \mathrm{d}\boldsymbol{z}_{\setminus i}},
    \end{equation}
    where $\mathrm{d}\boldsymbol{z}_{\setminus i}$ denotes the volume element $\mathrm{d} \boldsymbol{x}_0 \mathrm{d} \boldsymbol{v}_0 \dots \mathrm{d} \boldsymbol{x}_K \mathrm{d} \boldsymbol{v}_K$ excluding $\mathrm{d} \boldsymbol{x}_i$ (i.e., integrating over all variables except $\boldsymbol{x}_i$). The term $\boldsymbol{\hat v}_{i}[\boldsymbol{x}_{0}, \dots, \boldsymbol{x}_{K}]$ denotes the optimal velocity at time $t_i$ derived from the sequence $\boldsymbol{x}_{0}, \dots, \boldsymbol{x}_{K}$ via \cref{prop:cubic_spline}. Intuitively, $\omega^{(m+1)}_{i}(\boldsymbol{v}|\boldsymbol{x}_i)$ is set to the distribution induced by the velocities of all optimal control trajectories passing through $\boldsymbol{x}_i$ at time $t_i$. Since the total transport cost is the sum of costs along individual trajectories, aligning $\boldsymbol{v}$ with the optimal velocity of each trajectory minimizes the global objective.
\end{itemize}

In the update step above, the trajectory variable $\boldsymbol{z} = [(\boldsymbol{x}_{0}, \boldsymbol{v}_{0}), \dots, (\boldsymbol{x}_{K}, \boldsymbol{v}_{K})]$ follows the joint distribution $q^{(m)}(\boldsymbol{z})$ defined as:
\begin{align}
q^{(m)}(\boldsymbol{z}) =  \ \mu_{0}^{(\text{pos})}(\boldsymbol{x}_{0}) \cdot \omega_{0}^{(m)}(\boldsymbol{v}_0|\boldsymbol{x}_0) \nonumber 
 \times \prod_{i=0}^{K-1} \mathcal{K}^{*(m)}_{i \rightarrow i+1}[(\boldsymbol{x}_{i}, \boldsymbol{v}_{i}), (\boldsymbol{x}_{i+1}, \boldsymbol{v}_{i+1})],
\end{align}
where the propagator $\mathcal{K}_{i \rightarrow i+1}^{*(m)}$ is given by:
\begin{align}
    \mathcal{K}^{*(m)}_{i \rightarrow i+1}[(\boldsymbol{x}_{i}, \boldsymbol{v}_{i}), (\boldsymbol{x}_{i+1}, \boldsymbol{v}_{i+1})]  \nonumber= \frac{\pi^{*(m)}_{i \rightarrow i+1}[(\boldsymbol{x}_{i}, \boldsymbol{v}_{i}), (\boldsymbol{x}_{i+1}, \boldsymbol{v}_{i+1})]}{\mu_{i}^{(\text{pos})}(\boldsymbol{x}_{i}) \cdot \omega_{i}^{(m)} (\boldsymbol{v}_{i}|\boldsymbol{x}_{i})}.
\end{align}

During the iteration process, the objective $\mathcal{L}_{\text{DOAT}}[\mu_{i}^{(\text{pos})}(\boldsymbol{x}) \cdot \omega_{i}^{(m)} (\boldsymbol{v} | \boldsymbol{x})]$ decreases monotonically. Since the cost is bounded below by $0$, the iterative algorithm is guaranteed to converge.

\section{TracingFlow Algorithm for Trajectory Inference}

In the practical implementation of TracingFlow, we first perform an approximation of the conditional velocity distribution to reduce the VM-DOAT problem to a DOAT problem. Subsequently, we employ flow matching to solve for the acceleration field of the DOAT problem. Notably, for lineage tracing data, we incorporate biological priors into this framework.

\subsection{Approximation of Conditional Velocity Distribution}

\label{article: approximation of conditional velocity distribution}

Determining the full conditional velocity distribution $\omega_{i}(\boldsymbol{v}| \boldsymbol{x})$ to transform VM-DOAT to DOAT (\cref{article:VM-DOAT to DOAT}) is computationally prohibitive, potentially requiring an auxiliary generative model. We therefore adopt a deterministic approximation, assigning a single velocity to each data point. Consequently, \cref{eq:velocity_iter_2} is modified to:
\begin{align}
\omega^{(m+1)}_{i}(\boldsymbol{v} | \boldsymbol{x}_{i}) &= \delta(\boldsymbol{v} - \boldsymbol{v}^{(m+1)}_{\text{exp}}(\boldsymbol{x}_{i})), \ \ 
\boldsymbol{v}_{\text{exp}}^{(m+1)}(\boldsymbol{x}_{i})= \frac{\int q^{(m)}(\boldsymbol{z}) \boldsymbol{\hat v}_{i}[\boldsymbol{x}_0\cdots \boldsymbol{x}_K]  \, \mathrm{d}\boldsymbol{z}_{\setminus i}}{\int q^{(m)}(\boldsymbol{z}) \, \mathrm{d}\boldsymbol{z}_{\setminus i}}.
\end{align}
Intuitively, this assigns $\boldsymbol{v}^{(m+1)}_{\text{exp}}(\boldsymbol{x}_{i})$ as the mean velocity of all optimal control trajectories passing through $\boldsymbol{x}_{i}$ at time $t_{i}$.

\subsection{MiniBatch-OT}
\label{article:minibatch-OT}

Solving SOAT repeatedly is computationally intensive for large-scale datasets. However, since the conditional velocity distribution $\omega_{i}(\boldsymbol{v}|\boldsymbol{x})$ is approximated as a Dirac delta, the joint empirical distribution $\mu_{i}(\boldsymbol{x}, \boldsymbol{v})$ effectively becomes a superposition of Dirac deltas. This allows us to employ a minibatch strategy: we partition $\mu_{i}$ and $\mu_{i+1}$ into $B$ corresponding minibatches $\{\mu_{i}^{(n)}\}_{n=1}^{B}$ and $\{\mu_{i+1}^{(n)}\}_{n=1}^{B}$. By solving the local plans $\pi_{i \rightarrow i+1}^{*(n)}$ independently and aggregating them as $\pi^*_{i \rightarrow i+1} = \oplus_{n=1}^{B }\pi_{i \rightarrow i+1}^{*(n)}$, we significantly enhance computational efficiency.

\subsection{Incorporation of Biological Priors}
\label{article:lineage}

In lineage tracing applications, let indices $l_i$ and $l_{i+1}$ denote individual samples at time points $t_i$ and $t_{i+1}$, respectively. Each data point $(\boldsymbol{x}_{i,l_i}, \boldsymbol{v}_{i,l_i})$ is associated with a barcode $b_{i,l_i} \in \mathbb{N}^{+}$. We incorporate this prior by modifying the cost matrix $\mathcal{C}_{i \rightarrow i+1}$. Specifically, when computing the cost between the $l$-th sample at $t_i$ and the $k$-th sample at $t_{i+1}$, we compare their barcodes: if $b_{i, l_i} \neq b_{i+1, l_{i+1}}$, the standard cost is multiplied by a penalty factor $p_0(p_0 >1)$. This soft constraint discourages transitions between distinct lineages, effectively embedding biological knowledge into the dynamics learning process.

\subsection{Algorithm of TracingFlow}

The TracingFlow algorithm comprises three main steps: First, we apply the iterative approximation from \cref{article: approximation of conditional velocity distribution} to determine initial velocities, effectively transforming the VM-DOAT problem into a DOAT problem. Second, we compute the Optimal Transport Plan for SOAT as defined in \cref{eq:SOAT_problem}. Third, we sample conditional probability paths to train the acceleration network $\boldsymbol{a}_{\boldsymbol{\theta}}(\boldsymbol{x}, \boldsymbol{v},t)$ via regression. Additionally, to provide initial conditions for inference, we parameterize the unique initial velocities derived in the first step using a neural network $\boldsymbol{v}_{0, \boldsymbol{\xi}}(\boldsymbol{x})$, trained similarly via regression. The algorithm's pseudocode is provided in \cref{app:pseudocode_for_algo}.

Theoretically, TracingFlow recovers the exact distribution if the acceleration and initial velocity are perfectly fitted. In practice, where losses are non-zero, we prove that the positional \(\mathcal{W}_2\) distance between the generated and true distributions is bounded by the sum of the velocity and AFM losses, and the relevant theorem is detailed in \cref{app:w2_bound}.

\section{Experiment Results}

To evaluate the effectiveness of our algorithm, TracingFlow, we conducted three categories of experiments: 1) verifying whether the acceleration field of TracingFlow can transport the initial data distribution $\mu_{0}^{(\text{pos})}(\boldsymbol{x}) = \int \rho_{t_{0}} (\boldsymbol{x}, \boldsymbol{v} ) \mathrm{d} \boldsymbol{v}$ to the data distribution $\mu_{i}^{(\text{pos})}(\boldsymbol{x})$ at other given times $t_{i}$; 2) verifying whether TracingFlow can effectively perform distribution interpolation and extrapolation; and 3) verifying whether TracingFlow can more effectively infer the dynamical laws in Lineage Tracing data given biological prior knowledge.

\paragraph{Distribution Transport}
We evaluated TracingFlow on a 2-dimensional synthetic dataset and real single-cell omics datasets of varying dimensions. The evaluation metrics were the 1-Wasserstein ($\mathcal{W}_{1}$) and 2-Wasserstein ($\mathcal{W}_{2}$) distances, measuring the discrepancy between the fitted and ground truth distributions. The second-order dynamics model of TracingFlow provides stronger expressive power than first-order algorithms, yielding lower average $\mathcal{W}_{1}$ and $\mathcal{W}_{2}$ values across all datasets (\cref{tab:exp_dist_matching}). We visualized the evolutionary trajectories learned by OT-CFM and TracingFlow on the synthetic dataset in \cref{fig:exp_dist_matching}. TracingFlow's second-order model allows it to learn trajectories that cross in the position space $\boldsymbol{x}_{1}, \boldsymbol{x}_{2}$, whereas OT-CFM is unable to do so.  In \cref{app:exp_on_datasets}, we present the distribution reconstruction accuracy at each time point, as well as experimental results on additional datasets (such as the Mexico Gulf dataset).

\begin{table}[t]
    \centering
    \caption{Average $\mathcal{W}_{1}$ and $\mathcal{W}_{2}$ distances between the generated and ground truth distributions at various time points on the 2D Simulation , Cite 5D, and Cite 100D datasets for different algorithms.}
    \label{tab:exp_dist_matching}
   
    \resizebox{0.90\linewidth}{!}{ 
    \begin{tabular}{llcccccc}
        \toprule
        \multirow{2}{*}{\textbf{Method}} & \multirow{2}{*}{\textbf{Dynamics}} & \multicolumn{2}{c}{\textbf{2D Simulation}} & \multicolumn{2}{c}{\textbf{Cite 5D}} & \multicolumn{2}{c}{\textbf{Cite 100D}} \\
        \cmidrule(lr){3-4} \cmidrule(lr){5-6} \cmidrule(lr){7-8}
         & & $\mathcal{W}_{1}$ & $\mathcal{W}_{2}$ & $\mathcal{W}_{1}$ & $\mathcal{W}_{2}$ & $\mathcal{W}_{1}$ & $\mathcal{W}_{2}$ \\
        \midrule
        OT-CFM & \multirow{2}{*}{1st Order} & 0.5270 {\scriptsize $\pm$ 0.0612} & 0.5922 {\scriptsize $\pm$ 0.0813} & 0.8298 {\scriptsize $\pm$ 0.0162} & 0.9155 {\scriptsize $\pm$ 0.0160} & 10.5498 {\scriptsize $\pm$ 0.0291} & 10.6220 {\scriptsize $\pm$ 0.0310} \\
        SF$^2$M & & 0.9926 {\scriptsize $\pm$ 0.1857} & 1.1318 {\scriptsize $\pm$ 0.1734} & 0.7650 {\scriptsize $\pm$ 0.0375} & 0.8956 {\scriptsize $\pm$ 0.0408} & 12.3426 {\scriptsize $\pm$ 0.0735} & 12.4757 {\scriptsize $\pm$ 0.0783} \\
        \midrule
        3MSBM & \multirow{5}{*}{2nd Order} & 1.1855 {\scriptsize $\pm$ 0.2636} & 1.2767 {\scriptsize $\pm$ 0.2215} & 3.6369 {\scriptsize $\pm$ 0.3508} & 3.7703 {\scriptsize $\pm$ 0.3178} & 15.7708 {\scriptsize $\pm$ 0.1473} & 15.8833 {\scriptsize $\pm$ 0.1832} \\
        MMFM & & 2.3689 {\scriptsize $\pm$ 0.1830} & 2.6466 {\scriptsize $\pm$ 0.1769} & 2.5346 {\scriptsize $\pm$ 0.1520} & 2.7576 {\scriptsize $\pm$ 0.1877} & 12.5353 {\scriptsize $\pm$ 0.6268} & 12.7868 {\scriptsize $\pm$ 0.7061} \\
        HRF & & 1.5390 {\scriptsize $\pm$ 0.1130} & 1.6631 {\scriptsize $\pm$ 0.1214} & 1.4016 {\scriptsize $\pm$ 0.0437} & 1.5313 {\scriptsize $\pm$ 0.0430} & 10.1532 {\scriptsize $\pm$ 0.0643} & 10.2336 {\scriptsize $\pm$ 0.0707} \\
        CAF & & 1.6385 {\scriptsize $\pm$ 0.2777} & 1.6431 {\scriptsize $\pm$ 0.2766} & 4.5589 {\scriptsize $\pm$ 0.1073} & 4.7411 {\scriptsize $\pm$ 0.1620} & 16.7429 {\scriptsize $\pm$ 0.4371} & 16.9837 {\scriptsize $\pm$ 0.4324} \\
        \textbf{TF (Ours)} & & \textbf{0.3748} {\scriptsize $\pm$ 0.0505} & \textbf{0.5180} {\scriptsize $\pm$ 0.0360} & \textbf{0.5861} {\scriptsize $\pm$ 0.0688} & \textbf{0.6386} {\scriptsize $\pm$ 0.0775} & \textbf{8.0700} {\scriptsize $\pm$ 0.1537} & \textbf{8.2530} {\scriptsize $\pm$ 0.1889} \\
        \bottomrule
    \end{tabular}
    }
\end{table}

\begin{figure}[hbt!]
    \centering
    \includegraphics[width=0.48\linewidth]{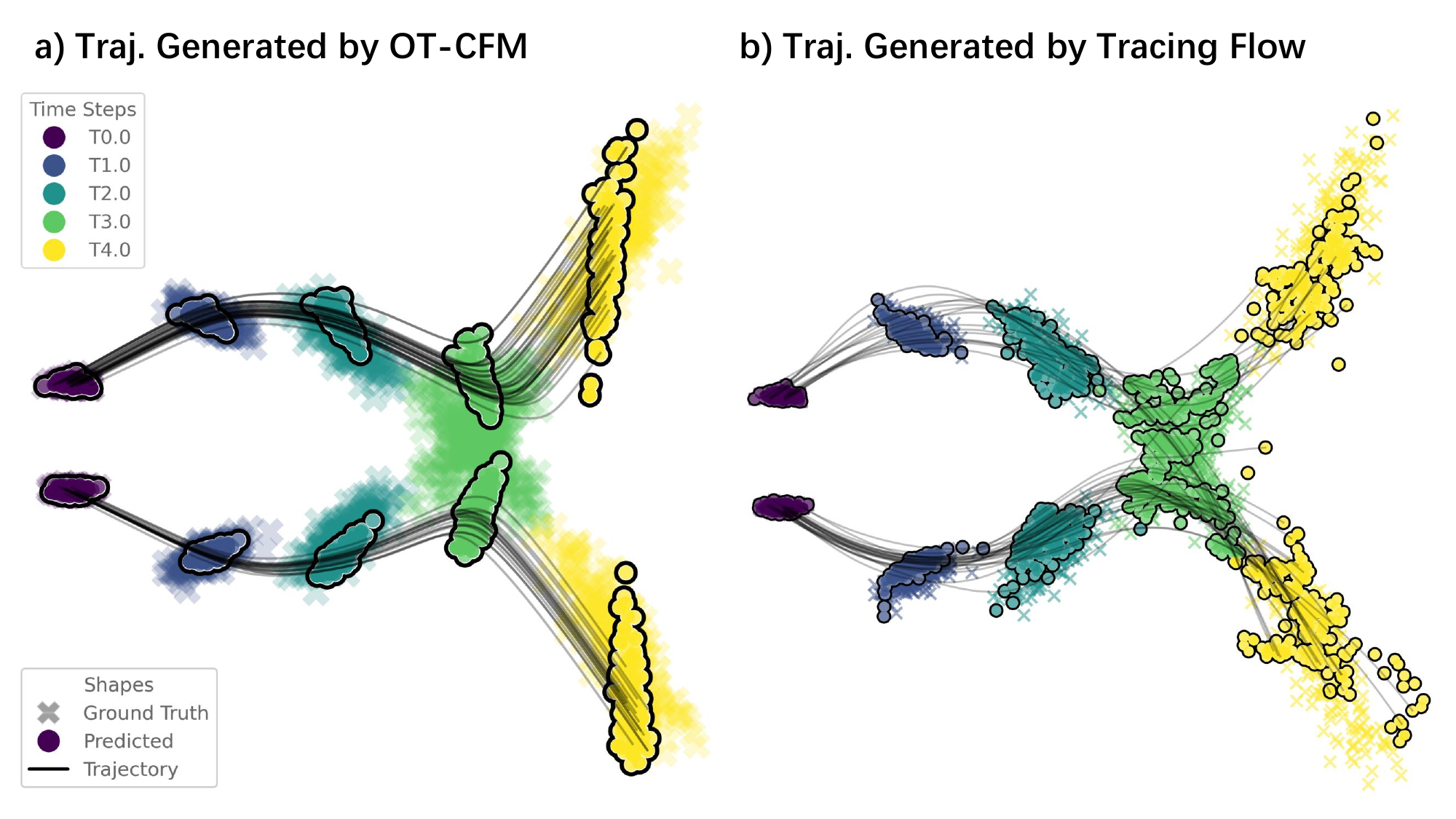}
    \caption{On the Simulation 2D dataset: a) non-crossing paths learned by OT-CFM, and b) crossing paths learned by TracingFlow.}
    \label{fig:exp_dist_matching}
\end{figure}

\paragraph{Interpolation and Extrapolation}
To verify TracingFlow's interpolation and extrapolation capabilities, we conducted a Hold-One-Out experiment on the 5-dimensional EB dataset (time points $t=0, \dots, 4$). In each iteration, we withheld one time point from the latter four for training and calculated the $\mathcal{W}_{1}$ and $\mathcal{W}_{2}$ distances between the interpolated and true distributions. As shown in \cref{tab:exp_hold_one_out}, TracingFlow achieved the highest accuracy, demonstrating that models based on second-order dynamics are superior in capturing high-curvature and non-linear trajectories.

\paragraph{Lineage Tracing Data}
To verify whether TracingFlow better handles lineage tracing data given biological priors, we experimented on a 3D Simulation Lineage Dataset and the Hematopoiesis Dataset. Following \cref{article:lineage}, we introduced biological priors into TracingFlow by modifying the transport cost matrix, other baselines also incorporates priors via conditional velocity fields. We used lineage-weighted $\mathcal{W}_{1}$ and $\mathcal{W}_{2}$ distances to evaluate consistency with biological priors while learning dynamics. Results in \cref{tab:exp_lineage_tracing} show that incorporating priors enables TracingFlow to outperform other algorithms. In \cref{fig:exp_lineage_tracing_1}, we visualized dynamics on the simulation dataset to demonstrate how TracingFlow learns trajectories conforming to priors. In \cref{fig:exp_lineage_tracing_2}, we plotted the PCA-reduced positions of two barcodes from the Hematopoiesis Dataset at $t=1$; the positions inferred by TracingFlow are significantly closer to the ground truth than those by OT-CFM.

\begin{table*}[t] 
    \centering
    
    \begin{minipage}[t]{0.37\textwidth}
        \centering
        \caption{Average $\mathcal{W}_{1}$ and $\mathcal{W}_{2}$ distances between the predicted and ground truth distributions at held-out time points on the EB 5D dataset for different algorithms.}
        \label{tab:exp_hold_one_out}
        \resizebox{\linewidth}{!}{ 
        \begin{tabular}{llcc}
            \toprule
            \multirow{2}{*}{\textbf{Method}} & \multirow{2}{*}{\textbf{Dynamics}} & \multicolumn{2}{c}{\textbf{EB 5D}} \\
            \cmidrule(lr){3-4}
             & & $\mathcal{W}_{1}$ & $\mathcal{W}_{2}$ \\
            \midrule
            OT-CFM & \multirow{2}{*}{1st Order} & 5.1205 {\scriptsize $\pm$ 0.2400} & 5.6910 {\scriptsize $\pm$ 0.2561} \\
            SF$^2$M & & 4.5602 {\scriptsize $\pm$ 0.1041} & 4.9905 {\scriptsize $\pm$ 0.1262} \\
            \midrule
            3MSBM & \multirow{3}{*}{2nd Order} & 7.1901 {\scriptsize $\pm$ 0.8667} & 7.4697 {\scriptsize $\pm$ 0.8588} \\
            MMFM & & 9.7777 {\scriptsize $\pm$ 0.9006} & 10.5980 {\scriptsize $\pm$ 0.9407} \\
            \textbf{TF (Ours)} & & \textbf{4.4974} {\scriptsize $\pm$ 0.0950} & \textbf{4.7859} {\scriptsize $\pm$ 0.1857} \\
            \bottomrule
        \end{tabular}
        }
    \end{minipage}%
    \hfill 
    \begin{minipage}[t]{0.60\textwidth}
        \centering
        \caption{Average lineage-weighted $\mathcal{W}_{1}$ and $\mathcal{W}_{2}$ distances between the generated and ground truth distributions at various time points on the 3D Simulation Lineage and Hematopoiesis datasets for different algorithms.}
        \label{tab:exp_lineage_tracing}
        \resizebox{\linewidth}{!}{ 
        \begin{tabular}{llcccc}
            \toprule
            \multirow{2}{*}{\textbf{Method}} & \multirow{2}{*}{\textbf{Dynamics}} & \multicolumn{2}{c}{\textbf{3D Sim-Lineage}} & \multicolumn{2}{c}{\textbf{Hematopoiesis}} \\
            \cmidrule(lr){3-4} \cmidrule(lr){5-6}
             & & $\mathcal{W}_{1}$ & $\mathcal{W}_{2}$ & $\mathcal{W}_{1}$ & $\mathcal{W}_{2}$ \\
            \midrule
            OT-CFM & \multirow{2}{*}{1st Order} & 2.2383 {\scriptsize $\pm$ 0.0205} & 2.3162 {\scriptsize $\pm$ 0.0216} & 14.9945 {\scriptsize $\pm$ 0.3442} & 15.3971 {\scriptsize $\pm$ 0.3180} \\
            SF$^2$M & & 1.5255 {\scriptsize $\pm$ 0.0513} & 1.5435 {\scriptsize $\pm$ 0.0515} & 18.7170 {\scriptsize $\pm$ 0.1124} & 18.9756 {\scriptsize $\pm$ 0.1168} \\
            \midrule
            3MSBM & \multirow{6}{*}{2nd Order} & 2.3126 {\scriptsize $\pm$ 0.4756} & 2.4905 {\scriptsize $\pm$ 0.4582} & 18.3313 {\scriptsize $\pm$ 0.9834} & 18.6427 {\scriptsize $\pm$ 0.9413} \\
            MMFM & & 2.6098 {\scriptsize $\pm$ 0.3149} & 3.0309 {\scriptsize $\pm$ 0.3265} & 27.0830 {\scriptsize $\pm$ 2.9299} & 29.5421 {\scriptsize $\pm$ 3.1420} \\
            HRF & & 1.7934 {\scriptsize $\pm$ 0.4043} & 1.8085 {\scriptsize $\pm$ 0.4022} & 15.4571 {\scriptsize $\pm$ 0.1266} & 15.7138 {\scriptsize $\pm$ 0.1282} \\
            CAF & & 2.6806 {\scriptsize $\pm$ 0.2777} & 2.7547 {\scriptsize $\pm$ 0.3521} & 15.2165 {\scriptsize $\pm$ 0.0685} & 15.4974 {\scriptsize $\pm$ 0.0576} \\
            TF w/o Bio-Prior & & 1.5292 {\scriptsize $\pm$ 0.1714}& 1.5532 {\scriptsize $\pm$ 0.2303} & 16.6166 {\scriptsize $\pm$ 0.1229}& 17.0110 {\scriptsize $\pm$ 0.1456}\\
            \textbf{TF (Ours)} & & \textbf{0.4538} {\scriptsize $\pm$ 0.2039} & \textbf{0.5328} {\scriptsize $\pm$ 0.2066} & \textbf{14.0470} {\scriptsize $\pm$ 0.1530} & \textbf{14.4589} {\scriptsize $\pm$ 0.1815} \\
            \bottomrule
        \end{tabular}
        }
    \end{minipage}
    
\end{table*}

\begin{figure}[hbt!]
    \centering
    \begin{minipage}[t]{0.48\linewidth}
        \centering
        \includegraphics[width=\linewidth, trim={0cm 0cm 0cm 2.1cm}, clip]{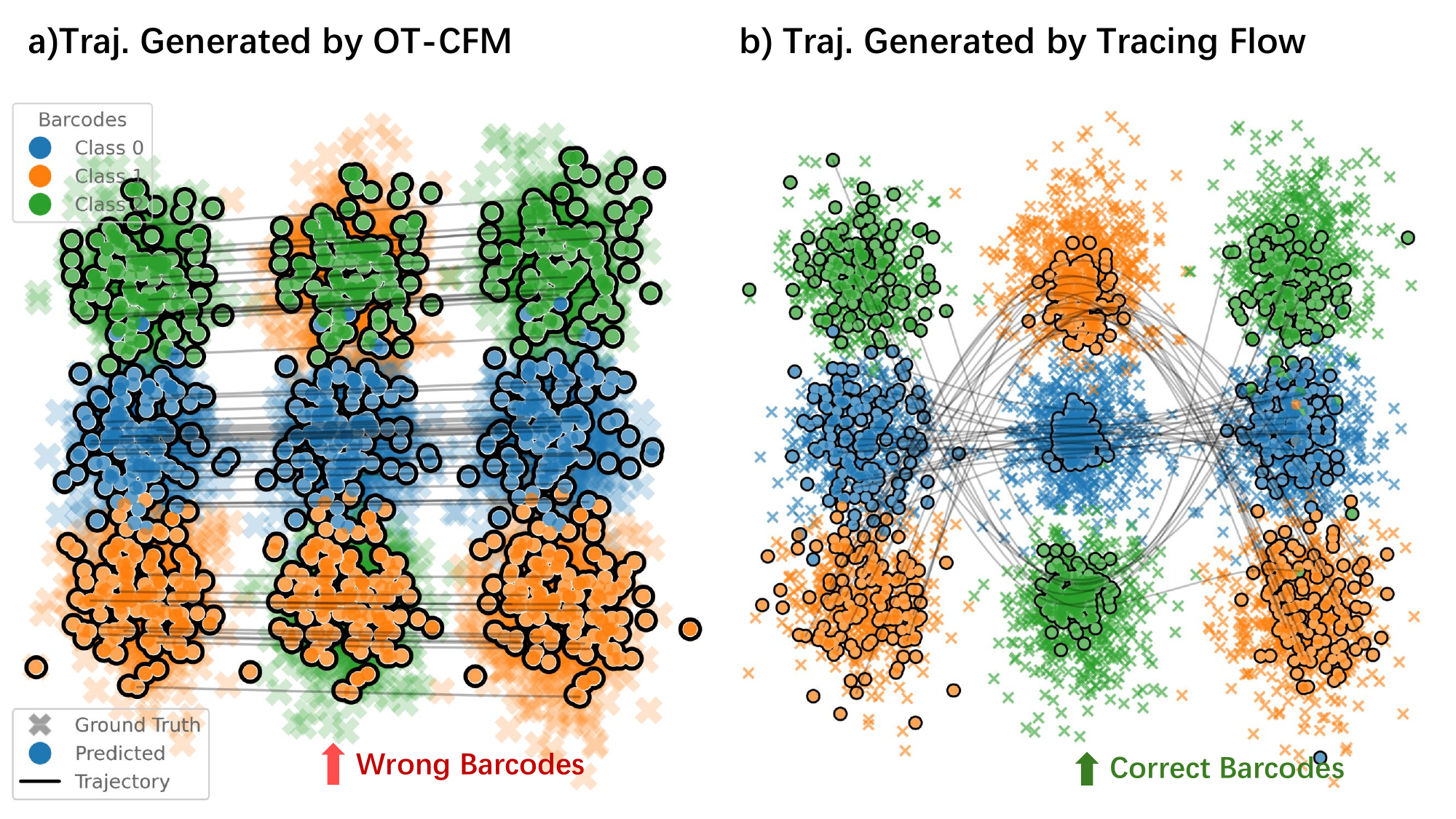}
        \caption{On the Sim-Lineage dataset: a) paths learned by OT-CFM without considering biological priors, and b) paths learned by TracingFlow that preserve biological priors.}
        \label{fig:exp_lineage_tracing_1}
    \end{minipage}
    \hfill
    \begin{minipage}[t]{0.48\linewidth}
        \centering
        \includegraphics[width=\linewidth]{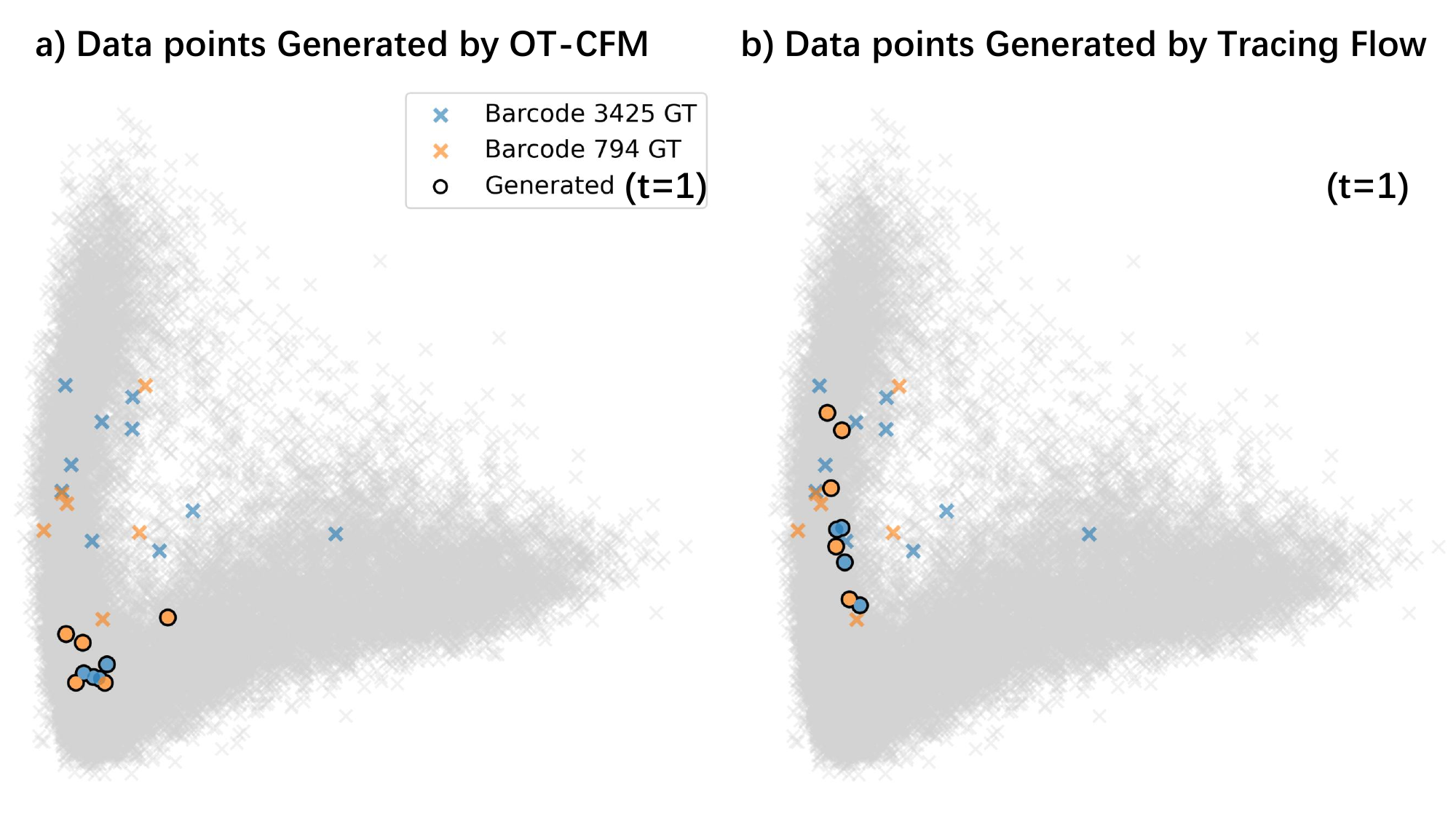}
        
        \caption{Comparison of generated (circles) and ground truth ('x') positions at $t=1$ for two barcodes on the hematopoiesis dataset \citep{weinreb2020lineage}. (a) OT-CFM. (b) TracingFlow. Visualized via 2D PCA.}
        \label{fig:exp_lineage_tracing_2}
    \end{minipage}
\end{figure}

\section{Conclusion and Limitation}

In this work, we introduce TracingFlow, a Flow Matching framework for the Dynamical Acceleration Optimal Transport (DOAT) problem. By pre-solving SOAT to learn acceleration and initial velocity, it enables efficient, simulation-free solutions for large-scale DOAT problem. Experiments on simulated and real world data demonstrate that second-order dynamics enhance expressivity, yielding precise temporal distribution recovery. Additionally, incorporating biological priors better preserves intercellular lineages.

A limitation is the computationally intensive SOAT pre-calculation, though this is mitigable via minibatch-OT. Furthermore, explicitly modeling the theoretical conditional velocity distribution per data point is prohibitive, so we approximate it using its expectation. The objective \(\int_{0}^{T} \frac{1}{2} \|\boldsymbol{a}\|^{2} \mathrm{d}t\) also lacks a clear physical interpretation, as standard classical actions exclude second time derivatives (see \cref{app:discussion_phase_space}). Despite current sequencing measurements lacking velocity data, priors suggest biological systems may follow higher-order dynamics due to the existence of complex regulations. Future work will extend TracingFlow to further model these dynamics.

\bibliographystyle{plainnat}
\bibliography{tracingflow_main}

\appendix

\newpage 

\section*{Contents of Appendix}

\appsection{A \quad Proofs of Main Theorems}{app:proof_of_main_thms}
\appsubsection{A.1 \quad Proof of Proposition 4.1 : Optimal Trajectory of Single Particle}{app:proof_optimal_traj_of_particle}

\appsubsection{A.3 \quad Proof of Remark 4.4 : No-Cross Property}{app:Non_collision}
\appsubsection{A.4 \quad Proof of Theorem 4.5 : Marginal Acceleration Field}{app:Marginal}
\appsubsection{A.5 \quad Proof of Theorem 4.6 : The Relation Between AFM Loss and CAFM Loss}{app:AFM_and_CAFM}
\appsubsection{A.6 \quad Proof of Theorem 4.7 : Marginal Probability Path}{app:Marginal_prob_path}
\appsubsection{A.7 \quad Proof of Proposition 4.8 : TracingFlow Exactly Solves DOAT Problem}{app:exactly_solve}
\appsubsection{A.8 \quad Proof of Proposition 4.9 : Solve Linear System to Get Optimal Velocity}{app:velocity}
\appsubsection{A.9 \quad Proof of Theorem A.1 : $\mathcal{W}_2$ Distance can be bounded by $\mathcal{L}_{\text{AFM}}$ and $\mathcal{L}_{v_0}$}{app:w2_bound}

\appsection{B \quad Datasets and Evaluation Metric}{app:datasets_and_eval_metric}
\appsubsection{B.1 \quad Experiment Setup}{app:experimental setup}
\appsubsection{B.2 \quad Evaluation Metrics}{app:evaluation metrics}
\appsubsection{B.3 \quad Datasets}{app:datasets}

\appsection{C \quad Experiment Details}{app:experiments}
\appsubsection{C.1 \quad Experiment Details across Various Datasets}{app:exp_on_datasets}
\appsubsection{C.2 \quad Training Time and Scalability of Tracing Flow}{app:scalability}
\appsubsection{C.3 \quad Sensitivity Analysis on Minibatch-OT}{app:minibatch}

\appsection{D \quad Pseudocode for Algorithm}{app:pseudocode_for_algo}

\appsection{E \quad Discussion}{app:discussion}
\appsubsection{E.1 \quad Relation with other Algorithms}{app:other_algo}
\appsubsection{E.2 \quad Is $\mathcal{S} = \mathcal{X} \times \mathcal{V}$ a phase space?}{app:discussion_phase_space}

\section{Proofs of Main Theorems} 

\label{app:proof_of_main_thms}

\subsection{Proof of Proposition 4.1}
\label{app:proof_optimal_traj_of_particle}
\begin{proof}
    First, we show that the unique minimizer is a cubic polynomial. Consider the energy functional defined on the interval $[0, T]$:
    \begin{equation}
        \mathcal{J}[\gamma] = \frac12 \int_0^T \|\gamma''(t)\|_2^2 \, \mathrm{d}t.
    \end{equation}
    The integrand $L(\gamma, \gamma', \gamma'') = \frac12 \|\gamma''\|_2^2$ depends on derivatives up to the second order. The necessary condition for optimality is given by the Euler--Lagrange equation:
    \begin{equation}
        \frac{\partial L}{\partial \gamma}
        - \frac{\mathrm{d}}{\mathrm{d} t}\left(\frac{\partial L}{\partial \gamma'}\right)
        + \frac{\mathrm{d}^2}{\mathrm{d} t^2}\left(\frac{\partial L}{\partial \gamma''}\right) = 0.
    \end{equation}
    Since $\partial L/\partial \gamma = 0$, $\partial L/\partial \gamma' = 0$, and $\partial L/\partial \gamma'' = \gamma''$, the equation simplifies to:
    \begin{equation}
        \frac{\mathrm{d}^2}{\mathrm{d} t^2}\gamma''(t) = \gamma^{(4)}(t) = 0.
    \end{equation}
    This implies that the optimal trajectory $\gamma(t)$ is a cubic polynomial:
    \begin{equation}
        \gamma(t) = c_3 t^3 + c_2 t^2 + c_1 t + c_0.
    \end{equation}
    Imposing the boundary conditions $\gamma(0)=\bm{x}_0, \gamma'(0)=\bm{v}_0$ and $\gamma(T)=\bm{x}_T, \gamma'(T)=\bm{v}_T$ leads to the linear system:
    \begin{equation}
        \begin{cases}
            c_0 = \bm{x}_0, \\
            c_1 = \bm{v}_0, \\
            c_3 T^3 + c_2 T^2 + c_1 T + c_0 = \bm{x}_T, \\
            3 c_3 T^2 + 2 c_2 T + c_1 = \bm{v}_T.
        \end{cases}
    \end{equation}
    Solving for the coefficients yields the unique solution:
    \begin{equation}
        \begin{cases}
            c_0 = \bm{x}_0, \\
            c_1 = \bm{v}_0, \\
            c_2 = \dfrac{3(\bm{x}_T-\bm{x}_0) - (2\bm{v}_0+\bm{v}_T)T}{T^2}, \\
            c_3 = \dfrac{(\bm{v}_0+\bm{v}_T)T - 2(\bm{x}_T-\bm{x}_0)}{T^3}.
        \end{cases}
    \end{equation}
\end{proof}

\subsection{Proof of Remark 4.3}\label{app:Non_collision}
\begin{proof}
First we propose a lemma:

\noindent\textit{Lemma.}
    Let $\mathcal{C}_{[0, t]}(\bm{x}_0, \bm{v}_0, \bm{x}_t, \bm{v}_t)$ be the least acceleration cost through $[0, t]$ ,$(\bm{x_0},\bm{v}_0), (\bm{x}_T, \bm{v}_T)\in \mathcal{X}\times\mathcal{V}$ are two different points. Then for arbitrary midpoint $\tau\in (0, 1)$ and state $(\bm{x}, \bm{v})\in \mathcal{X}\times\mathcal{V}$, the following inequality holds:
    \begin{equation}
        \mathcal{C}_{[0, \tau]}(\bm{x_0}, \bm{v_0}, \bm{x}, \bm{v})+\mathcal{C}_{[\tau, T]}(\bm{x}, \bm{v}, \bm{x}_T, \bm{v}_T)\geq\mathcal{C}_{[0, T]}(\bm{x}_0, \bm{v}_0, \bm{x}_T, \bm{v}_T)
    \end{equation}
    The "$=$" holds if and only if $(\bm{x, v})=\gamma^{*}(\tau)$, where $\gamma^{*}(t)(t\in [0, T])$ is the cubic interpolation trajectory connecting $(\bm{x_0}, \bm{v_0})$ and $(\bm{x}_T, \bm{v_T})$.  

\begin{proof}
    Fix $\tau\in(0, T)$. Define objective function $ F_{\tau}(\bm{x}, \bm{v})= \mathcal{C}_{[0, \tau]}(\bm{x_0}, \bm{v_0}, \bm{x}, \bm{v})+\mathcal{C}_{[\tau, T]}(\bm{x}, \bm{v}, \bm{x}_T, \bm{v}_T)$. By Corollary \ref{cost} , we can rewrite
    \begin{equation}
    \begin{aligned}
        F_{\tau}(\bm{x}, \bm{v})=\frac{2}{\tau^3}
        \left\{
        (\|\bm{v}_0\|^2+\langle \bm{v}_0,\bm{v}\rangle+ \|\bm{v}\|^2)\tau^2+3\langle \bm{v}_0+\bm{v}, \bm{x}_0-\bm{x}\rangle \tau+3\|\bm{x}_0-\bm{x}\|^2
        \right\}+\\
        \frac{2}{(T-\tau)^3}
        \left\{(\|\bm{v}_T\|^2+\langle \bm{v}_T,\bm{v}\rangle+ \|\bm{v}\|^2)(T-\tau)^2+3\langle \bm{v}_T+\bm{v}, \bm{x}-\bm{x}_T\rangle T+3\|\bm{x}_T-\bm{x}\|^2
        \right\}
    \end{aligned}
    \end{equation}
    Which is a quadratic form of $(\bm{x},\bm{v})$ up to a constant. $F_{\tau}(\bm{x}, \bm{v})\geq 0$ always holds, namely, the quadratic form has a finite lower bound, so it is semi-quadratic.  To minimize $F_\tau(\bm{x}, \bm{v})$, we can take differentiations:
        \begin{equation}\label{nabla_x}
        \nabla_{\bm{x}}F_{\tau}(\bm{x}, \bm{v}) = 3\left( \frac{-(\bm{v}_0+\bm{v})\tau+2(\bm{x}-\bm{x}_0)}{{\tau}^2} + \frac{(\bm{v}+\bm{v}_T)(T-\tau)+2(\bm{x}-\bm{x}_T)}{(T-{\tau})^2} \right) = 0.
    \end{equation}
    \begin{equation}\label{nabla_v}
        \nabla_{\bm{v}}F_{\tau}(\bm{x}, \bm{v}) = \frac{(\bm{v}_0+2\bm{v}){\tau} + 3(\bm{x}_0-\bm{x})}{{\tau}^2} + \frac{(\bm{v}_T+2\bm{v})(T-{\tau}) + 3(\bm{x}-\bm{x}_T)}{(T-{\tau})^2} = 0,   
    \end{equation}

    Denote the cubic interpolation from $(\bm{x_0}, \bm{v_0})$ to $(\bm{x}, \bm{v})$ on $[0, \tau]$ and $(\bm{x}, \bm{v})$ to $(\bm{x}_T, \bm{v}_T)$ on $[\tau, T]$ by $\gamma_{[0, \tau]}(t)$ and $\gamma_{[\tau, T]}({t})$, respectively.The coefficients are known according to \cref{app:proof_optimal_traj_of_particle}. Then we have:
    \begin{equation}
        \gamma_{[0, \tau]}^{''}(\tau^{-})=6\cdot\frac{(\bm{v}_0+\bm{v})\tau-2(\bm{x}_{}-\bm{x}_{0})}{\tau^{2}}+2\cdot\frac{3{(\bm{x}_{}-\bm{x}_{0})-(2\bm{v}_0+\bm{v})}\tau}{\tau^2}
        =2\cdot\frac{3(\bm{x_0}-\bm{x})+(\bm{v_0}+2\bm{v})\tau}{\tau^2},
    \end{equation}
    \begin{equation}
        \begin{aligned}
            \gamma_{[\tau, T]}^{''}(\tau^{+})= 6\cdot \frac{(\bm{v}+\bm{v}_T)\tau-2(\bm{x}_T-\bm{x})}{\tau^2}+2\cdot\frac{3(\bm{x}_T-\bm{x})-(2\bm{v}+\bm{v}_T)(T-\tau)}{(T-\tau)^2} \\ =2\cdot\frac{3(\bm{x}-\bm{x}_T)+(\bm{v}+2\bm{v}_T)(T-\tau)}{(T-\tau)^2}
        \end{aligned}
        \end{equation}
    \begin{equation}
        \gamma_{[0, \tau]}^{'''}(\tau^{-})= 6\cdot\frac{(\bm{v}_0+\bm{v})\tau-2(\bm{x}_{}-\bm{x}_{0})}{\tau^{2}}
    \end{equation}

    \begin{equation}
        \begin{aligned}
            \gamma_{[\tau, T]}^{'''}(\tau^{+})=6\cdot\frac{(\bm{v}+\bm{v}_T)(T-\tau)-2(\bm{x}_T-\bm{x})}{(T-\tau)^2}
        \end{aligned}
    \end{equation}

    Then we can find that ~\ref{nabla_x} and ~\ref{nabla_v} appear to be $ \gamma_{[0, \tau]}^{'''}(\tau^{-})=\gamma_{[\tau, T]}^{'''}(\tau^{+}), \gamma_{[0, \tau]}^{''}(\tau^{-})=\gamma_{[\tau, T]}^{''}(\tau^{+}).$ Because the two curves are cubic, and $\gamma^{(i)}_{[0, \tau]}(\tau^{-})=\gamma_{[\tau, T]}^{(i)}(\tau^{+}), i=0, 1, 2,3$, so we know that in the optimal case $\gamma_{[0, \tau]}=\gamma^{*}\big{|}_{[0, \tau]}, \gamma_{[\tau, T]}=\gamma^{*}\big{|}_{[\tau, T]}.$. That is, the unique solution of~\ref{nabla_x} and~\ref{nabla_v} is $(\gamma^{*}(\tau), {\gamma^{*}}^{'}(\tau)),$ which completes the proof.
\end{proof}

Back to our target, assume $\gamma$ and $\tilde{\gamma}$ cross at $(\bm{x}_\tau, \bm{v}_\tau)$, then we have
\begin{equation}
\begin{aligned}
    &\mathcal{C}_{[0, T]}(\bm{x}_0, \bm{v}_0, \bm{x}_T, \bm{v}_{T}) + \mathcal{C}_{[0, T]}(\tilde{\bm{x}}_0, \tilde{\bm{v}}_0, \tilde{\bm{x}}_T, \tilde{\bm{v}}_{T})\\
    =& \left\{\mathcal{C}_{[0, \tau]}(\bm{x}_0, \bm{v}_0, \bm{x}_\tau, \bm{v}_\tau) + \mathcal{C}_{[\tau, T]}(\bm{x}_\tau, \bm{v}_\tau, \bm{x}_{T}, \bm{v}_{T})\right\} + \left\{\mathcal{C}_{[0, \tau]}(\tilde{\bm{x}}_0, \tilde{\bm{v}}_0, \bm{x}_\tau, \bm{v}_\tau) + \mathcal{C}_{[\tau, T]}(\bm{x}_\tau, \bm{v}_\tau, \tilde{\bm{x}}_{T}, \tilde{\bm{v}}_{T})\right\}\\
    =& \left\{\mathcal{C}_{[0, \tau]}(\bm{x}_0, \bm{v}_0, \bm{x}_\tau, \bm{v}_\tau) + \mathcal{C}_{[\tau, T]}(\bm{x}_\tau, \bm{v}_\tau, \tilde{\bm{x}}_{T}, \tilde{\bm{v}}_{T})\right\} + \left\{\mathcal{C}_{[0, \tau]}(\tilde{\bm{x}}_0, \tilde{\bm{v}}_0, \bm{x}_\tau, \bm{v}_\tau) + \mathcal{C}_{[\tau, T]}(\bm{x}_\tau, \bm{v}_\tau, \bm{x}_{T}, \bm{v}_{T})\right\}\\
    \geq & \mathcal{C}_{[0, T]}(\bm{x}_{0}, \bm{v}_{0}, \tilde{\bm{x}}_{T}, \tilde{\bm{v}}_{T}) + \mathcal{C}_{[0, T]}(\tilde{\bm{x}}_{0}, \tilde{\bm{v}}_{0}, \bm{x}_{T}, \bm{v}_T).
\end{aligned}
\end{equation}

The equality holds if and only if $\gamma$ and $\tilde{\gamma}$ coincide. Consequently, the strict inequality holds for distinct trajectories, which implies the non-crossing property. This completes the proof.
\end{proof}

\subsection{Proof of Theorem 4.4}

\label{app:Marginal}

\begin{proof}
Consider the time derivative of the marginal distribution:
\begin{align}
    \partial_{t} \rho_{t}(\boldsymbol{x},\boldsymbol{v}) &= \partial_{t} \left\{\int \rho_{t}(\boldsymbol{x},\boldsymbol{v}|z) q(z) \mathrm{d} z \right\} \nonumber\\
    &= \int \partial_{t} \rho_{t}(\boldsymbol{x},\boldsymbol{v}|z) q(z) \mathrm{d} z \nonumber\\
    &= - \int \nabla_{\boldsymbol{x}} \cdot (\rho_{t}(\boldsymbol{x},\boldsymbol{v}|z) \boldsymbol{v}) q(z) \mathrm{d} z - \int \nabla_{\boldsymbol{v}} \cdot (\rho_{t}(\boldsymbol{x},\boldsymbol{v}|z) \boldsymbol{a}(\boldsymbol{x},\boldsymbol{v},t|z)) q(z) \mathrm{d} z \nonumber\\
    &= - \nabla_{\boldsymbol{x}} \cdot \left(\int \rho_{t}(\boldsymbol{x},\boldsymbol{v}|z) \boldsymbol{v} q(z) \mathrm{d} z \right) - \nabla_{\boldsymbol{v}} \cdot \left(\int \rho_{t}(\boldsymbol{x},\boldsymbol{v}|z) \boldsymbol{a}(\boldsymbol{x},\boldsymbol{v},t|z) q(z) \mathrm{d}z \right) \nonumber\\
    &= - \nabla_{\boldsymbol{x}} \cdot (\rho_{t}(\boldsymbol{x},\boldsymbol{v}) \boldsymbol{v} ) - \nabla_{\boldsymbol{v}} \cdot (\boldsymbol{a}_{t}(\boldsymbol{x},\boldsymbol{v}) \rho_{t}(\boldsymbol{x},\boldsymbol{v}))
\end{align}
Therefore, $\rho_{t}(\boldsymbol{x},\boldsymbol{v})$ satisfies the continuity equation:
\begin{equation}
    \partial_{t} \rho_{t}(\boldsymbol{x},\boldsymbol{v}) = - \nabla_{\boldsymbol{x}} \cdot (\rho_{t}(\boldsymbol{x},\boldsymbol{v}) \boldsymbol{v}) - \nabla_{\boldsymbol{v}} \cdot (\boldsymbol{a}_{t}(\boldsymbol{x},\boldsymbol{v}) \rho_{t}(\boldsymbol{x},\boldsymbol{v}))
\end{equation}
This completes the proof.
\end{proof}

\subsection{Proof of Theorem 4.5}

\label{app:AFM_and_CAFM}

\begin{proof}
Consider the gradient of the AFM Loss and CAFM Loss:
\begin{align}
\nabla_{\boldsymbol{\theta}}\mathcal{L}_{\text{AFM}} &=  \nabla_{\theta} \mathbb{E}_{\rho_{t}(\boldsymbol{x},\boldsymbol{v})} (\|\boldsymbol{a}_{\theta}(\boldsymbol{x},\boldsymbol{v},t)\|^{2} - 2 \boldsymbol{a}_{\theta}^{T} (\boldsymbol{x},\boldsymbol{v},t) \boldsymbol{a}(\boldsymbol{x},\boldsymbol{v},t))\\
\nabla_{\boldsymbol{\theta}}\mathcal{L}_{\text{CAFM}} &=  \nabla_{\theta} \mathbb{E}_{\rho_{t}(\boldsymbol{x},\boldsymbol{v}|\boldsymbol{z}) q(\boldsymbol{z})} (\|\boldsymbol{a}_{\theta}(\boldsymbol{x},\boldsymbol{v},t)\|^{2} - 2 \boldsymbol{a}_{\theta}^{T} (\boldsymbol{x},\boldsymbol{v},t) \boldsymbol{a}(\boldsymbol{x},\boldsymbol{v},t|\boldsymbol{z}))
\end{align}
For the first term of $\nabla_{\boldsymbol{\theta}}\mathcal{L}_{\text{AFM}}$:
\begin{align}
\mathbb{E}_{\rho_{t}(\boldsymbol{x},\boldsymbol{v})} \|\boldsymbol{a}_{\theta}(\boldsymbol{x},\boldsymbol{v},t)\|^{2} &=  \int \rho_{t}(\boldsymbol{x},\boldsymbol{v}) \| \boldsymbol{a}_{\theta}(\boldsymbol{x},\boldsymbol{v},t) \|^{2} \mathrm{d} \boldsymbol{x}  \mathrm{d} \boldsymbol{v} \nonumber\\
&= \int  \rho_{t}(\boldsymbol{x},\boldsymbol{v}|z) q(\boldsymbol{z})  \|\boldsymbol{a}_{\theta}(\boldsymbol{x},\boldsymbol{v},t) \|^{2} \mathrm{d}\boldsymbol{x}  \mathrm{d} \boldsymbol{v} \mathrm{d} z \nonumber\\
&= \mathbb{E}_{\rho_{t}(\boldsymbol{x},\boldsymbol{v}|\boldsymbol{z}) q(\boldsymbol{z})} \|\boldsymbol{a}_{\theta}(\boldsymbol{x},\boldsymbol{v},t)\|^{2} 
\end{align}
For the second term:
\begin{align}
\mathbb{E}_{\rho_{t}(\boldsymbol{x},\boldsymbol{v})} [ \boldsymbol{a}_{\theta}^{T}(\boldsymbol{x},\boldsymbol{v},t) \boldsymbol{a}(\boldsymbol{x},\boldsymbol{v},t)] &=  \int \rho_{t} (\boldsymbol{x},\boldsymbol{v}) \boldsymbol{a}_{\theta}^{T}(\boldsymbol{x},\boldsymbol{v},t)  \boldsymbol{a}(\boldsymbol{x},\boldsymbol{v},t) \mathrm{d} \boldsymbol{x}  \mathrm{d} \boldsymbol{v} \nonumber\\
&= \int \rho_{t}(\boldsymbol{x},\boldsymbol{v})  \boldsymbol{a}_{\theta}^{T}(\boldsymbol{x},\boldsymbol{v},t) \left( \int q(\boldsymbol{z}) \dfrac{\boldsymbol{a}(\boldsymbol{x},\boldsymbol{v},t|\boldsymbol{z}) \rho_{t}(\boldsymbol{x},\boldsymbol{v}|\boldsymbol{z})}{\rho_{t}(\boldsymbol{x},\boldsymbol{v})} \mathrm{d}z \right) \mathrm{d} \boldsymbol{x}  \mathrm{d} \boldsymbol{v} \nonumber\\
&= \int \int \boldsymbol{a}_{\theta}^{T}(\boldsymbol{x},\boldsymbol{v},t) \boldsymbol{a} (\boldsymbol{x},\boldsymbol{v},t|\boldsymbol{z})  \rho_{t}(\boldsymbol{x},\boldsymbol{v}|\boldsymbol{z})  q(z) \mathrm{d} \boldsymbol{x}  \mathrm{d} \boldsymbol{v} \mathrm{d} \boldsymbol{z} \nonumber\\
&= \mathbb{E}_{\rho_{t}(\boldsymbol{x},\boldsymbol{v}|\boldsymbol{z}) q(\boldsymbol{z})} (\boldsymbol{a}_{\theta}^{T}(\boldsymbol{x},\boldsymbol{v},t) \boldsymbol{a}(\boldsymbol{x},\boldsymbol{v},t|\boldsymbol{z})  )
\end{align}
\end{proof}

\subsection{Proof of Theorem 4.6}

\label{app:Marginal_prob_path}

\begin{proof}
Consider a given time $t_{i}$. Based on the expressions for $\boldsymbol{m}_{\boldsymbol{x}}, \boldsymbol{m}_{\boldsymbol{v}}$, at this moment we exactly have:
\begin{equation}
    \rho_{t_{i}}(\boldsymbol{x},\boldsymbol{v}|\boldsymbol{z})  = \mathcal{N}(\boldsymbol{x}|\boldsymbol{x}_{i}, \sigma_{x}^{2}\boldsymbol{I} ) \cdot  \mathcal{N}(\boldsymbol{v}| \boldsymbol{v}_{i} , \sigma_{v}^{2}\boldsymbol{I})
\end{equation}

For simplicity, let $\mathrm{d} \boldsymbol{s}_j  = \mathrm{d}\boldsymbol{x}_j \mathrm{d} \boldsymbol{v}_j$Therefore, the Marginal Probability Path is:
\begin{align}
\rho_{t_{i}}(\boldsymbol{x},\boldsymbol{v}) &= \int \rho_{t_{i}} (\boldsymbol{x},\boldsymbol{v}|\boldsymbol{z}) q(\boldsymbol{z}) \mathrm{d} \boldsymbol{z}  \nonumber\\
&= \int \rho_{t_{i}}(\boldsymbol{x},\boldsymbol{v}|(\boldsymbol{x}_{i},\boldsymbol{v}_{i})) \mu_{0}(\boldsymbol{x}_{0},\boldsymbol{v}_{0}) \cdot \prod_{j=0 }^{K-1}\mathcal{K}^{*}_{j \rightarrow j+1}[(\boldsymbol{x}_{j}, \boldsymbol{v}_{j}), (\boldsymbol{x}_{j+1}, \boldsymbol{v}_{j+1})]\mathrm{d} \boldsymbol{s}_{0} \cdots  \mathrm{d}  \boldsymbol{s}_{K} 
\end{align}
Using the definition of the propagator, we know that:
\begin{equation}
    \int \mathcal{K}^{*}_{j \rightarrow j+1} [(\boldsymbol{x}_{j}, \boldsymbol{v}_{j}) , (\boldsymbol{x}_{j+1},\boldsymbol{v}_{j+1}) ] \mathrm{d} \boldsymbol{x}_{j+1} \mathrm{d} \boldsymbol{v}_{j+1} =  1
\end{equation}

Thus, we can  integrate out the future terms ($j > i$) first , and then iteratively collapse the past terms ($j<i$):

\begin{align}
\rho_{t_{i}}(\boldsymbol{x},\boldsymbol{v}) &= \int \rho_{t_{i}}(\boldsymbol{x},\boldsymbol{v}|(\boldsymbol{x}_{i},\boldsymbol{v}_{i})) \mu_{0}(\boldsymbol{x}_{0},\boldsymbol{v}_{0}) \cdot \prod_{j=0 }^{K-1}\mathcal{K}^{*}_{j \rightarrow j+1}[(\boldsymbol{x}_{j}, \boldsymbol{v}_{j}), (\boldsymbol{x}_{j+1}, \boldsymbol{v}_{j+1})] \mathrm{d} \boldsymbol{s}_{0} \cdots  \mathrm{d} \boldsymbol{s}_{K}  \nonumber\\
&= \int \rho_{t_{i}}(\boldsymbol{x},\boldsymbol{v}|(\boldsymbol{x}_{i},\boldsymbol{v}_{i})) \mu_{0}(\boldsymbol{x}_{0},\boldsymbol{v}_{0}) \cdot  \prod_{j=0 }^{i-1}\mathcal{K}^{*}_{j \rightarrow j+1}[(\boldsymbol{x}_{j}, \boldsymbol{v}_{j}), (\boldsymbol{x}_{j+1}, \boldsymbol{v}_{j+1})] \mathrm{d} \boldsymbol{s}_{0} \cdots  \mathrm{d} \boldsymbol{s}_{i} \nonumber\\
&= \int \rho_{t_{i}}(\boldsymbol{x},\boldsymbol{v}|(\boldsymbol{x}_{i},\boldsymbol{v}_{i}))   \pi^{*}_{0 \rightarrow1}[(\boldsymbol{x}_{0}, \boldsymbol{v}_{0}), (\boldsymbol{x}_{1}, \boldsymbol{v}_{1})]\prod_{j=1 }^{i-1}\mathcal{K}^{*}_{j \rightarrow j+1}[(\boldsymbol{x}_{j}, \boldsymbol{v}_{j}), (\boldsymbol{x}_{j+1}, \boldsymbol{v}_{j+1})] \mathrm{d} \boldsymbol{s}_{0} \cdots  \mathrm{d} \boldsymbol{s}_{i} \nonumber\\
&= \int \rho_{t_{i}}(\boldsymbol{x},\boldsymbol{v}|(\boldsymbol{x}_{i},\boldsymbol{v}_{i}))  \mu_{1}(\boldsymbol{x}_{1},\boldsymbol{v}_{1})  \prod_{j=1 }^{i-1}\mathcal{K}^{*}_{j \rightarrow j+1}[(\boldsymbol{x}_{j}, \boldsymbol{v}_{j}), (\boldsymbol{x}_{j+1}, \boldsymbol{v}_{j+1})] \mathrm{d} \boldsymbol{s}_{1} \cdots  \mathrm{d} \boldsymbol{s}_{i} \nonumber\\
&\quad \vdots  \nonumber\\
&= \int  \rho_{t_{i}}(\boldsymbol{x},\boldsymbol{v}|(\boldsymbol{x}_{i},\boldsymbol{v}_{i})) \mu_{i}(\boldsymbol{x}_{i},\boldsymbol{v}_{i})  \mathrm{d} \boldsymbol{s}_{i} \nonumber\\
&= \int  (\mathcal{N}(\boldsymbol{x}|\boldsymbol{x}_{i}, \sigma_{x}^{2}\boldsymbol{I} ) \cdot  \mathcal{N}(\boldsymbol{v}| \boldsymbol{v}_{i} , \sigma_{v}^{2}\boldsymbol{I})) \mu_{i}(\boldsymbol{x}_{i},\boldsymbol{v}_{i})   \mathrm{d} \boldsymbol{s}_{i} \nonumber\\
&= (\mu_{i} *  g) (\boldsymbol{x},\boldsymbol{v} )
\end{align}
where $g (\boldsymbol{x} ,\boldsymbol{v}) = (\mathcal{N}(\boldsymbol{x}|\boldsymbol{0}, \sigma_{x}^{2}\boldsymbol{I} ) \cdot  \mathcal{N}(\boldsymbol{v}| \boldsymbol{0} , \sigma_{v}^{2}\boldsymbol{I}))$. When $(\sigma_{\bm{x}}, \sigma_{\bm{v}})\rightarrow \bm{0}$, we have $\rho_{t_{i}}(\boldsymbol{x}, \boldsymbol{v}) \rightarrow \mu_{i}(\boldsymbol{x},\boldsymbol{v})$ in distribution. 

If $\mu_i(\bm{x}, \bm{v})$ is absolutely continuous with respect to the Lebesgue measure and admits a density in $L^1(\mathbb{R}^{2d})$, we can guarantee strong convergence in the $L^1$ norm:
    $\lim_{\sigma \to 0} \| \rho_{t_{i}}(\bm{x}, \bm{v}) - \mu_{i}(\bm{x}, \bm{v}
    ) \|_{L^1} = 0$ for all $i=0, 1, ..., K$.

\end{proof}

\subsection{Proof of Proposition 4.7}

\label{app:exactly_solve}

First, we prove that the Marginal Probability Path constructed by TracingFlow converges to the optimal solution of the DOAT problem when $\sigma_{\boldsymbol{x}} \rightarrow 0, \sigma_{\boldsymbol{v}} \rightarrow 0$.

We investigate directly in the augmented space. Let $\boldsymbol{s} = [\boldsymbol{x}, \boldsymbol{v}]^{T} \in \mathcal{X} \times \mathcal{V}$, and consider the cost of the DOAT problem. The subsequent position of a particle starting from the initial point $\boldsymbol{s}_{0}$ can be represented by the flow map $\boldsymbol{\phi}_{t}(\boldsymbol{s}_{0})$. Therefore, the cost can be written as:

\begin{align}
\mathcal{L}_{\text{DOAT}}( \boldsymbol{a}, \rho) &=  \int_{t_{0}}^{t_{k}} \int_{\mathcal{X} \times \mathcal{V}} \dfrac{1}{2} \| \boldsymbol{a}(\boldsymbol{s}, t) \|^{2} \rho_{t}(\boldsymbol{s})  \mathrm{d} \boldsymbol{x}  \mathrm{d} \boldsymbol{v} \mathrm{d} t \nonumber\\
&= \int_{t_{0}}^{t_{k}} \int_{\mathcal{X} \times \mathcal{V}} \dfrac{1}{2} \|\boldsymbol{a}(\boldsymbol{\phi}_{t}(\boldsymbol{s}_{0}), t) \|^{2} \det\left(\dfrac{\partial \boldsymbol{\phi}}{\partial \boldsymbol{s}_{0}}\right)^{-1} \rho_{0}(\boldsymbol{s}_{0}) \det\left(\dfrac{\partial \boldsymbol{\phi}}{\partial \boldsymbol{s}_{0} }\right) \mathrm{d} \boldsymbol{s}_{0}  \mathrm{d} t \nonumber\\
&= \int_{\mathcal{X} \times \mathcal{V}} \rho_{0}(\boldsymbol{s}_{0}) \mathrm{d} \boldsymbol{s}_{0}  \int_{t_{0}}^{t_{k}} \dfrac{1}{2} \|\boldsymbol{a}(\boldsymbol{\phi}_{t}(\boldsymbol{s}_{0}), t) \|^{2} \mathrm{d} t \nonumber\\
&= \int_{\Omega}\left(\int_{t_{0}}^{t_{k}}  \dfrac{1}{2} \| \boldsymbol{\gamma} '' (t) \|^{2} \mathrm{d} t \right) \mathrm{d} \mathbb{P}[\boldsymbol{\gamma}]
\end{align}

where $\mathbb{P}$ is a probability measure in the path space, satisfying the constraint $(e_{t_{i}})_{\#} \mathbb{P} = \mu_{i}(\boldsymbol{x}, \boldsymbol{v})$. Here, $e_{t}$ is the evaluation map $e_{t}[\boldsymbol{\gamma}] = \boldsymbol{\gamma}(t)$, and $(e_{t_{i}})_{\#} \mathbb{P}$ denotes the push-forward of the probability measure by the evaluation map. According to \cref{proposition: optimal_traj_of_particle}, the minimum cost path connecting $\boldsymbol{s}_{i}$ and $\boldsymbol{s}_{i+1}$ is a cubic spline, and the minimum cost is $\mathcal{C}_{i \rightarrow i+1}[\boldsymbol{s}_{i} ,\boldsymbol{s}_{i+1}]$. Therefore, for any path $\boldsymbol{\gamma}$:
\begin{equation}
    \int_{t_{i}}^{t_{i+1}} \dfrac{1}{2} \| \boldsymbol{\gamma}''(t) \|^{2} \mathrm{d} t \ge \mathcal{C}_{i  \rightarrow i+1}[\boldsymbol{s}_{i} ,\boldsymbol{s}_{i+1}]
\end{equation}

Thus, the DOAT problem cost has a lower bound:
\begin{align}
\mathcal{L}_{\text{DOAT}} (\rho, \boldsymbol{a} ) &=  \int_{\Omega} \sum_{i=0}^{K-1} \left(\int_{t_{i}}^{t_{i+1}} \dfrac{1}{2} \| \boldsymbol{\gamma}'' (t)\|^{2} \mathrm{d} t\right) \mathrm{d} \mathbb{P}[\boldsymbol{\gamma}] \nonumber\\
&\ge  \int_{\Omega}\sum_{i=0}^{K-1} \mathcal{C}_{i  \rightarrow i+1}[\boldsymbol{s}_{i} ,\boldsymbol{s}_{i+1}] \mathrm{d} \mathbb{P}[\boldsymbol{\gamma}] \nonumber\\
&\ge  \int \sum_{i=0}^{K-1}  \mathcal{C}_{i  \rightarrow i+1}[\boldsymbol{s}_{i} ,\boldsymbol{s}_{i+1}] \mathrm{d} \pi^{*} (\boldsymbol{s}_{i}, \boldsymbol{s}_{i+1})
\end{align}

where $\pi^{*} (\boldsymbol{s}_{i}, \boldsymbol{s}_{i+1})$ is the Optimal Transport Plan for the SOAT problem between $\mu(\boldsymbol{s}_{i})$ and $\mu(\boldsymbol{s}_{i+1})$.

Consider the path constructed by TracingFlow. In the case where $\sigma_{\boldsymbol{x}}, \sigma_{\boldsymbol{v}}\rightarrow 0$, we have $\rho_{t}(\boldsymbol{x}, \boldsymbol{v} | \boldsymbol{z})  = \rho_{t} (\boldsymbol{s}| \boldsymbol{z})\rightarrow \delta(\boldsymbol{s} - \boldsymbol{\gamma}^{*}_{\boldsymbol{z}}(t))$, where $\boldsymbol{\gamma}^{*}_{\boldsymbol{z}}$ is the optimal single-particle trajectory connecting $\boldsymbol{z}=  [\boldsymbol{s}_{0}, \boldsymbol{s}_{1},\cdots ,\boldsymbol{s}_{K}]$. We calculate the transport cost of the constructed path:
\begin{align}
\mathcal{L}_{\text{TF}} &=   \int q(\boldsymbol{z}) \mathrm{d} \boldsymbol{z} \sum_{i=0}^{K-1}\int_{t_{i}}^{t_{i+1}} \|  \boldsymbol{\gamma}{''}_{\boldsymbol{z}}^{* } (t)  \|^{2} \mathrm{d} t \nonumber\\
&= \int q(\boldsymbol{z})  \mathrm{d}\boldsymbol{z}   \sum_{i=0}^{K-1}  \mathcal{C}_{i  \rightarrow i+1 } [\boldsymbol{s}_{i}, \boldsymbol{s}_{i+1}] \nonumber\\
&= \int \mu_{0}(\boldsymbol{s}_{0}) \prod_{i=0}^{K-1} \mathcal{K}^{*}_{i  \rightarrow i+1 }[\boldsymbol{s}_{i} , \boldsymbol{s}_{i+1}] \left(\sum_{j=0}^{K-1}  \mathcal{C}_{j  \rightarrow j+1 } [\boldsymbol{s}_{j}, \boldsymbol{s}_{j+1}]\right) \mathrm{d} \boldsymbol{z}
\end{align}

Here, using the definition of the propagator:

\begin{equation}
\begin{split}
    &\int \mu_{0}(\boldsymbol{s}_{0})\prod_{i=0}^{K-1} \mathcal{K}^{*}_{i  \rightarrow i+1 }[\boldsymbol{s}_{i} , \boldsymbol{s}_{i+1}] \left( \mathcal{C}_{j  \rightarrow j+1 } [\boldsymbol{s}_{j}, \boldsymbol{s}_{j+1}]\right) \mathrm{d} \boldsymbol{z} \\
    &\qquad = \int \pi^{*}_{j  \rightarrow j+1 }[\boldsymbol{s}_{j} , \boldsymbol{s}_{j+1}] \mathcal{C}_{j  \rightarrow j+1 } [\boldsymbol{s}_{j}, \boldsymbol{s}_{j+1}] \mathrm{d} \boldsymbol{s}_{j} \mathrm{d} \boldsymbol{s}_{j+1}
\end{split}
\end{equation}

Therefore:
\begin{equation}
    \mathcal{L}_{\text{TF}} =  \sum_{i=0}^{K-1} \int \pi^{*}_{i  \rightarrow i+1 }[\boldsymbol{s}_{i} , \boldsymbol{s}_{i+1}] \mathcal{C}_{i  \rightarrow i+1 } [\boldsymbol{s}_{i}, \boldsymbol{s}_{i+1}] \mathrm{d} \boldsymbol{s}_{i} \mathrm{d} \boldsymbol{s}_{i+1}
\end{equation}
This coincides exactly with the lower bound of the DOAT problem cost derived above. Thus, the path constructed by TracingFlow is exactly the optimal solution to the DOAT problem in the case where $\sigma_{\boldsymbol{x}}, \sigma_{\boldsymbol{v}}\rightarrow 0$.

Next, we prove that under the construction of TracingFlow, the acceleration field $\boldsymbol{a}(\boldsymbol{s},t)$ is single-valued with respect to $\boldsymbol{s}$ and $t$. By \cref{thm:non_collision}, under the Optimal Transport Plan, the trajectories of flow maps starting from different points must not intersect. This implies that for a given time $t$, any point $\boldsymbol{s}$ in the augmented space is traversed by at most one flow map trajectory. Therefore, the acceleration field $\boldsymbol{a}(\boldsymbol{s},t)$ possesses single-valuedness. In other words, the solution constructed by TracingFlow can indeed be expressed by a single optimal acceleration field $\boldsymbol{a}(\boldsymbol{s},t)$.

\subsection{Proof of Proposition 4.8}
\label{app:velocity}
\begin{proof}
Let $\gamma:[0,T]\to \mathbb{R}^{d}$ be a path that is $C^{1}$ on $[0, T]$ and piecewise $C^{2}$ on each interval $[t_{j-1}, t_{j}]$.

The optimal solution is given by a piecewise cubic polynomial:

\begin{equation}
\begin{split}
    \gamma(t_{j-1}+\tau) 
    &= \frac{(\bm{v}_{j-1}+\bm{v}_{j})\Delta t_j-2(\bm{x}_{j}-\bm{x}_{j-1})}{(\Delta t_j)^{3}}\tau^{3} \\
    &\quad + \frac{3(\bm{x}_{j}-\bm{x}_{j-1})-(2\bm{v}_{j-1}+\bm{v}_{j})\Delta t_j}{(\Delta t_j)^{2}}\tau^{2} + \bm{v}_{j-1}\tau+\bm{x}_{j-1},
\end{split}
\end{equation}
for $\tau\in [0, \Delta t_{j}]$, where $\Delta t_{j}=t_{j}-t_{j-1}$, 

    According to Corollary \ref{cost}, it suffices to minimize

    \begin{align}
    \mathcal{J}(\bm{v}_0, \bm{v}_1, \dots, \bm{v}_K) = 
    & 2\sum_{j=1}^{K}\frac{1}{(\Delta t_{j})^3} \bigg\{ (\|\bm{v}_{j-1}\|^2+\langle \bm{v}_{j-1},\bm{v}_{j}\rangle + \|\bm{v}_j\|^2)(\Delta t_j)^2 \notag \\
    &\qquad\qquad\qquad + 3\langle \bm{v}_{j-1}+\bm{v}_j, x_{j-1}-x_j\rangle \Delta t_j + 3\|x_{j-1}-x_j\|^2 \bigg\}
\end{align}
$\mathcal{J}\geq0$ always holds, so the the formula above is a semi-positive quadratic form.

Take differentiation, we have 
\begin{align}
        \nabla_{\bm{v}_j}\mathcal{J} &= 2\left(\frac{\bm{v}_{j-1}+2\bm{v}_j}{\Delta t_j}+\frac{2\bm{v}_j+\bm{v}_{j+1}}{\Delta t_{j+1}}\right) + 6\left(\frac{\bm{x}_{j-1}-\bm{x}_{j}}{(\Delta t_j)^2}+\frac{\bm{x}_j-\bm{x}_{j+1}}{(\Delta t_{j+1})^2}\right), \quad \text{for } j=1,\dots,K-1, \\
        \nabla_{\bm{v}_0}\mathcal{J} &= \frac{2(2\bm{v}_0+\bm{v}_1)}{\Delta t_1}+\frac{6(\bm{x}_0-\bm{x}_1)}{(\Delta t_1)^2}, 
        \nabla_{\bm{v}_K}\mathcal{J} = \frac{2(\bm{v}_{K-1}+2\bm{v}_K)}{\Delta t_K}+\frac{6(\bm{x}_{K-1}-\bm{x}_K)}{(\Delta t_K)^2}.
    \end{align}

Set the gradients to zero and rearranging the terms to separate the velocities $\bm{v}$ and positions $\bm{x}$, we obtain:
    
    For $j=0$:
    \begin{equation}
    \frac{2}{\Delta t_1}\bm{v}_0 + \frac{1}{\Delta t_1}\bm{v}_1 = \frac{3(\bm{x}_1-\bm{x}_0)}{(\Delta t_1)^2}.
    \end{equation}
    For $j=1, \dots, K-1$:
    \begin{equation}
    \frac{1}{\Delta t_j}\bm{v}_{j-1} + 2\left(\frac{1}{\Delta t_j} + \frac{1}{\Delta t_{j+1}}\right)\bm{v}_j + \frac{1}{\Delta t_{j+1}}\bm{v}_{j+1} = 3\left(\frac{\bm{x}_j-\bm{x}_{j-1}}{(\Delta t_j)^2} + \frac{\bm{x}_{j+1}-\bm{x}_j}{(\Delta t_{j+1})^2}\right).
    \end{equation}
    For $j=K$:
    \begin{equation}
    \frac{1}{\Delta t_K}\bm{v}_{K-1} + \frac{2}{\Delta t_K}\bm{v}_K = \frac{3(\bm{x}_K-\bm{x}_{K-1})}{(\Delta t_K)^2}.
    \end{equation}

Denote $V=[\bm{v}_0, \bm{v}_1, \dots, \bm{v}_K]\in \mathbb{R}^{d\times(K+1)}$ and
    writing this system in matrix form $VA=B$, 
    where 
\begin{equation}
A =
\begin{bmatrix}
\frac{2}{\Delta t_1} & \frac{1}{\Delta t_1} & 0 & \cdots & 0 & 0 \\
\frac{1}{\Delta t_1} & 2\left(\frac{1}{\Delta t_1} + \frac{1}{\Delta t_2}\right) & \frac{1}{\Delta t_2} & \cdots & 0 & 0 \\
0 & \frac{1}{\Delta t_2} & 2\left(\frac{1}{\Delta t_2} + \frac{1}{\Delta t_3}\right) & \ddots & \vdots & \vdots \\
\vdots & \vdots & \ddots & \ddots & \frac{1}{\Delta t_{K-1}} & 0 \\
0 & 0 & \cdots & \frac{1}{\Delta t_{K-1}} & 2\left(\frac{1}{\Delta t_{K-1}} + \frac{1}{\Delta t_K}\right) & \frac{1}{\Delta t_K} \\
0 & 0 & \cdots & 0 & \frac{1}{\Delta t_K} & \frac{2}{\Delta t_K}
\end{bmatrix}
\in \mathbb{R}^{(K+1)\times(K+1)}
\end{equation}

\begin{equation}
\begin{split}
    B = \Bigg[ & \frac{3(\bm{x}_1-\bm{x}_0)}{(\Delta t_1)^2}, 3\left(\frac{\bm{x}_1-\bm{x}_{0}}{(\Delta t_1)^2} + \frac{\bm{x}_{2}-\bm{x}_1}{(\Delta t_{2})^2}\right) , \dots, \\
    & 3\left(\frac{\bm{x}_{K-1}-\bm{x}_{K-2}}{(\Delta t_{K-1})^2} + \frac{\bm{x}_{K}-\bm{x}_{K-1}}{(\Delta t_{K})^2}\right), \frac{3(\bm{x}_K-\bm{x}_{K-1})}{(\Delta t_K)^2} \Bigg] \in \mathbb{R}^{d\times(K+1)}
\end{split}
\end{equation}
The solution satisfies $\gamma'(t_j) = \bm{v}_j$ for all $j=0,1,...K.$    
    
    Since $A$ is strictly diagonally dominant, the solution $V$ is unique.
\end{proof}

\subsection{Proof of Theorem A.1}

\begin{theorem}
Let $\boldsymbol{a}_{\boldsymbol{\theta}}(\boldsymbol{x}, \boldsymbol{v}, t)$ be $L_\theta$-Lipschitz continuous of $(\boldsymbol{x}, \boldsymbol{v})$. The Wasserstein-2 distance between the generated distribution $\hat \rho_{t_{i}}^{(\text{pos})} (\boldsymbol{x}) = \int \hat \rho_{t_i} (\boldsymbol{x}, \boldsymbol{v}) \mathrm{d} \boldsymbol{v}$ and the true distribution  $\mu_i^{(\text{pos})} (\boldsymbol{x})$ in the position space is bounded by
\begin{equation}
    \mathcal{W}_{2}^2(\hat \rho_{t_{i}}^{(\text{pos})},\mu_i^{(\text{pos})})  \le 2e^{\kappa(t_i-t_0)}\left(\mathcal{L}_{v_{0}}+(t_i - t_0)\mathcal{L}_{\text{AFM}}|_{[t_0, t_{i}]} \right)
\end{equation}
where the cumulative AFM loss and the initial velocity loss are defined as:
\begin{equation*}
\begin{aligned}
    \mathcal{L}_{\text{AFM}}|_{[t_0, t_{i}]} = \mathbb{E}_{t , (\boldsymbol{x}, \boldsymbol{v})\sim \rho_{t}} \|\boldsymbol{a}_{\boldsymbol{\theta}}(\boldsymbol{x}, \boldsymbol{v}, t ) -  \boldsymbol{a}(\boldsymbol{x}, \boldsymbol{v}, t)\|^{2}, 
    \mathcal{L}_{v_0} = \mathbb{E}_{(\boldsymbol{x}_0, \boldsymbol{v}_0)\sim \rho_{t_0}} \| \boldsymbol{v}_{0,\boldsymbol{\xi}}(\boldsymbol{x}_0) - \boldsymbol{v}_0 \|^2.
\end{aligned}
\end{equation*}

and $\kappa$ is a constant only depending on $L_\theta$. 
\end{theorem}

\label{app:w2_bound}

\begin{proof}

We address the problem directly within the augmented space $\mathcal{X} \times \mathcal{V}$. Let the system state be denoted by $\boldsymbol{s} = [\boldsymbol{x}, \boldsymbol{v}]^{T}$, and let $\mu_{i}(\boldsymbol{s})$ represent the data distribution in this space. We define a flow map $\boldsymbol{\phi}_{t}:[t_{0}, t_{K}] \times \mathcal{X} \times \mathcal{V} \rightarrow \mathcal{X} \times \mathcal{V}$, which is governed by the exact marginal acceleration field $\boldsymbol{a}(\boldsymbol{s},t)$ defined in \cref{eq:marginal_acc_field}. 

Using $(\boldsymbol{\phi}_{t})_{\#}$ to denote the push-forward operator, the time-dependent density evolves as $\rho_t=(\boldsymbol{\phi}_t)_{\#}\rho_{t_0}$, where $\rho_{t_0}=\mu_0$. Similarly, let $\boldsymbol{\phi}^{\boldsymbol{\theta}}_t$ be the flow map induced by the learned acceleration field $\boldsymbol{a}_{\boldsymbol{\theta}}(\boldsymbol{s}, t)$. The generated distribution $\hat{\rho}_t$ then evolves according to $\hat{\rho}_t =(\boldsymbol{\phi}^{\boldsymbol{\theta}}_t)_{\#}\hat{\rho}_{t_0}$. It is important to note that $\hat{\rho}_{t_0} \neq \rho_{t_0}$, as the conditional distribution of the initial velocity is approximated via a neural network.

We introduce a bridging measure $\nu_t = (\bm{\phi}^{\theta}_{t})_{\#}\rho_{t_0}$, then the Wasserstein-2 distance  between joint distributions can be decomposed as
\begin{equation}
    \mathcal{W}_2^2(\rho_t, \hat{\rho}_t)\leq\left( {\mathcal{W}_{2}(\rho_t, \nu_t)}+{\mathcal{W}_2(\nu_t, \hat{\rho}_t)}\right)^2\leq 2\left(\mathcal{W}_{2}^2(\rho_t, \nu_t)+\mathcal{W}_2^2(\nu_t, \hat{\rho}_t)\right)\\
\end{equation}

\textbf{Estimation of Term 1}: Define $Q_t = \int \mu_0(\bm{s}_0)\|\bm{\phi}_t^{\theta}(\bm{s}_0)-\bm{\phi}_{t}(\bm{s}_0)\|^2\mathrm{d}\bm{s}_0, $here $\mu_{0}=\rho_{t_0}$ according to definition. Noting that $Q_{t_0}=0$ since the starting points are identical. Define $A = \begin{bmatrix} 0 & I \\ 0 & 0 \end{bmatrix}$.

Taking the time derivative, we have:
\begin{align}
\frac{\mathrm{d}Q_t}{\mathrm{d}t} &= 2\int \mu_0(\bm{s}_0)(\bm{\phi}_t^{\theta}(\bm{s}_0)-\bm{\phi}_t(\bm{s}_0))^T \left(\frac{\mathrm{d}\bm{\phi}_t^{\theta}(\bm{s}_0)}{\mathrm{d}t}-\frac{\mathrm{d}\bm{\phi}_t(\bm{s}_0)}{\mathrm{d}t}\right)\mathrm{d}\bm{s}_0  \nonumber\\
&= 2\int \mu_0(\bm{s}_0)(\bm{\phi}_t^{\theta}(\bm{s}_0)-\bm{\phi}_t(\bm{s}_0))^T \left(A\bm{\phi}_t^{\theta}(\bm{s}_0)-A\bm{\phi}_t(\bm{s}_0)+
\begin{bmatrix}
    \bm{0}\\ \bm{a}_{\theta}(\bm{\phi}_t^{\theta}(\bm{s}_0), t)
\end{bmatrix}-
\begin{bmatrix}
    \bm{0}\\ \bm{a}(\bm{\phi}_t(\bm{s}_0), t)
\end{bmatrix}\right)\mathrm{d}\bm{s}_0 \nonumber\\
&= 2\int \mu_0(\bm{s}_0) \Bigg[ (\bm{\phi}_t^{\theta}(\bm{s}_0)-\bm{\phi}_t(\bm{s}_0))^T A (\bm{\phi}_t^{\theta}(\bm{s}_0)-\bm{\phi}_t(\bm{s}_0)) \nonumber\\
&\quad + (\bm{\phi}_t^{\theta}(\bm{s}_0)-\bm{\phi}_t(\bm{s}_0))^T \begin{bmatrix}
    \bm{0}\\ \bm{a}_{\theta}(\bm{\phi}_t^{\theta}(\bm{s}_0), t)-\bm{a}_{\theta}(\bm{\phi}_t(\bm{s}_0), t)
\end{bmatrix} \nonumber \\
&\quad + (\bm{\phi}_t^{\theta}(\bm{s}_0)-\bm{\phi}_t(\bm{s}_0))^T \begin{bmatrix}
    \bm{0}\\ \bm{a}_{\theta}(\bm{\phi}_t(\bm{s}_0), t)-\bm{a}(\bm{\phi}_t(\bm{s}_0), t)
\end{bmatrix} \Bigg] \mathrm{d}\bm{s}_0
\end{align}

We bound the three terms inside the integral separately. Let $\Delta \bm{\phi}_t(\bm{s}_0) = \bm{\phi}_t^{\theta}(\bm{s}_0)-\bm{\phi}_t(\bm{s}_0)$.
\begin{enumerate}
    \item Since $\|A\| \le 1$ , we have $\Delta \bm{\phi}_t(\bm{s}_0)^T A \Delta \bm{\phi}_t (\bm{s}_0)\le \|\Delta \bm{\phi}_t(\bm{s}_0)\|^2$.
    \item Using the Lipschitz continuity of the network $\bm{a}_{\theta}$ with constant $L_{\theta}$:
  
    \begin{align}
    \Delta \bm{\phi}_t(\bm{s}_0)^T \begin{bmatrix} \bm{0} \\ \bm{a}_{\theta}(\bm{\phi}_t^{\theta}(\bm{s}_0), t)-\bm{a}_{\theta}(\bm{\phi}_t(\bm{s}_0), t) \end{bmatrix} 
    &\le \|\Delta \bm{\phi}_t(\bm{s}_0)\| \cdot \|\bm{a}_{\theta}(\bm{\phi}_t^{\theta}(\bm{s}_0), t)-\bm{a}_{\theta}(\bm{\phi}_t(\bm{s}_0), t)\| \notag \\
    &\le L_{\theta} \|\Delta \bm{\phi}_t(\bm{s}_0)\|^2.
\end{align}
    \item Using Cauchy-Schwarz and the inequality $2xy \le x^2 + y^2$:
    \begin{align}
    \Delta \bm{\phi}_t(\bm{s}_0)^T \begin{bmatrix} \bm{0} \\ \bm{a}_{\theta}(\bm{\phi}_t(\bm{s}_0), t)-\bm{a}(\bm{\phi}_t(\bm{s}_0), t) \end{bmatrix} &\le \|\Delta \bm{\phi}_t(\bm{s}_0)\| \cdot\|\bm{a}_{\theta}(\bm{\phi}_t(\bm{s}_0), t)-\bm{a}(\bm{\phi}_t(\bm{s}_0), t)\| \nonumber\\
    &\le \frac{1}{2}\|\Delta \bm{\phi}_t(\bm{s}_0)\|^2 + \frac{1}{2}\|\bm{a}_{\theta}(\bm{\phi}_t(\bm{s}_0), t)-\bm{a}(\bm{\phi}_t(\bm{s}_0), t)\|^2.
    \end{align}
\end{enumerate}

Substituting these bounds back into the derivative:
\begin{align}
\frac{\mathrm{d}Q_t}{\mathrm{d}t} &\le 2\int \mu_0(\bm{s}_0) \left[ (1 + L_{\theta} + \frac{1}{2}) \|\Delta \bm{\phi}_t(\bm{s}_0)\|^2 + \frac{1}{2} \|\bm{a}_{\theta}(\bm{\phi}_t(\bm{s}_0), t)-\bm{a}(\bm{\phi}_t(\bm{s}_0), t)\|^2 \right] \mathrm{d}\bm{s}_0 \nonumber\\
&= (2L_{\theta} + 3) Q_t + \int \mu_0(\bm{s}_0) \|\bm{a}_{\theta}(\bm{\phi}_t(\bm{s}_0), t)-\bm{a}(\bm{\phi}_t(\bm{s}_0), t)\|^2 \mathrm{d}\bm{s}_0
\end{align}

By Gronwall's inequality, since $Q_{t_0}=0$:
\begin{equation}
Q_t \le e^{(2L_\theta+3)(t-t_0)} \int_{t_0}^{t}  \left( \int \mu_0(\bm{s}_0) \|\bm{a}_{\theta}(\bm{\phi}_{\tau}(\bm{s}_0), \tau)-\bm{a}(\bm{\phi}_{\tau}(\bm{s}_0), \tau)\|^2 \mathrm{d}\bm{s}_0 \right) \mathrm{d}\tau
\end{equation}
By the definition of Wasserstein-2 distance, we have:
\begin{align}t_0
\mathcal{W}_{2}^{2} (\rho_{t}, \nu_{t}) &= \min_{\pi} \int \|\bm{s}_1-\bm{s}_2\|^2\mathrm{d}\pi\quad\text{s.t.} \int \pi\mathrm{d}\boldsymbol{s}_{1} =\rho_{t}, \int \pi \mathrm{d}\boldsymbol{s}_{2} =\nu_t
\nonumber\\ 
&\leq  \int \mu_{0}(\boldsymbol{s}_{0}) \| \boldsymbol{\phi}^{\boldsymbol{\theta}}_{t} (\boldsymbol{s}_{0}) - \boldsymbol{\phi}_{t} (\boldsymbol{s}_{0})  \|^{2} \mathrm{d} \boldsymbol{s}_{0}\nonumber\\
& = Q_t \le e^{(2L_\theta+3)(t-t_0)} \int_{t_0}^{t} \mathbb{E}_{\bm{s} \sim \mu_{\tau}} \|\bm{a}_{\theta}(\bm{s}, \tau)-\bm{a}(\bm{s}, \tau)\|^2 \mathrm{d}\tau\\  
&= e^{(2L_\theta+3)(t-t_0)}(t-t_0)\mathcal{L}_{\text{AFM}}|_{[t_0, t]}. 
\end{align}

\textbf{Estimation of Term 2}:Take $\pi_0(\bm{s}_0, \bm{s}_0^{\xi})$ as the optimal coupling that attains $\mathcal{W}_2^2(\nu_{t_0}, \hat{\rho}_{t_0}) = \int \|\bm{s}_0-\bm{s}_{0}^{\xi}\|^2\mathrm{d}\pi_0(\bm{s_0}, \bm{s}_0^{\xi})$. Note that $\nu_{t_0} = \mu_0$ (the true data distribution).

From a pair of starting points $(\bm{s}_0, \bm{s}_0^{\xi})\sim\pi_0$, define the square distance between trajectories under the same approximate flow as $\Delta(t)=\|\bm{\phi}_t^{\theta}(\bm{s}_0)-\bm{\phi}_t^{\theta}(\bm{s}_0^{\xi})\|^2$. Taking the time derivative, we have:
\begin{align}
    \frac{\mathrm{d}\Delta(t)}{\mathrm{d}t} &=  2{(\bm{\phi}_t^{\theta}(\bm{s}_0)-\bm{\phi}_t^{\theta}(\bm{s}_0^{\xi}))^T}
    \left( \frac{\mathrm{d}\bm{\phi}_t^{\theta}(\bm{s}_0)}{\mathrm{d}t} - \frac{\mathrm{d}\bm{\phi}_t^{\theta}(\bm{s}_0^{\xi})}{\mathrm{d}t} \right) \nonumber\\
    &= 2{(\bm{\phi}_t^{\theta}(\bm{s}_0)-\bm{\phi}_t^{\theta}(\bm{s}_0^{\xi}))^T}\left( A(\bm{\phi}_t^{\theta}(\bm{s}_0)-\bm{\phi}_t^{\theta}(\bm{s}_0^{\xi})) + \begin{bmatrix} \bm{0} \\ \bm{a}_{\theta}(\bm{\phi}_t^{\theta}(\bm{s}_0), t) - \bm{a}_{\theta}(\bm{\phi}_t^{\theta}(\bm{s}_0^{\xi}), t) \end{bmatrix} \right) \nonumber\\
    &\le 2\left( \|A\| \cdot \|\bm{\phi}_t^{\theta}(\bm{s}_0)-\bm{\phi}_t^{\theta}(\bm{s}_0^{\xi})\|^2 + \|\bm{\phi}_t^{\theta}(\bm{s}_0)-\bm{\phi}_t^{\theta}(\bm{s}_0^{\xi})\| \cdot \|\bm{a}_{\theta}(\bm{\phi}_t^{\theta}(\bm{s}_0), t) - \bm{a}_{\theta}(\bm{\phi}_t^{\theta}(\bm{s}_0^{\xi}), t)\| \right) \nonumber\\
    &\le 2(1 + L_\theta) \Delta(t).
\end{align}
Again by Gronwall's inequality, we have:
\begin{equation}
    \Delta(t) \le e^{2(1+L_\theta)(t-t_0)} \Delta(t_0) = e^{2(1+L_\theta)(t-t_0)} \|\bm{s}_0 - \bm{s}_0^{\xi}\|^2.
\end{equation}
Since $\nu_t = (\bm{\phi}_t^{\theta})_{\#} \nu_{t_0}$ and $\hat{\rho}_{t_0} = (\bm{\phi}_t^{\theta})_{\#} \hat{\rho}_{t_0}$, the push-forward of the optimal coupling $\pi_t = (\bm{\phi}_t^{\theta}, \bm{\phi}_t^{\theta})_{\#} \pi_0$ is a valid coupling for $(\nu_t, \hat{\rho}_t)$.
Integrate the squared distance over all starting pairs with respect to $\pi_0$, we obtain:
\begin{align}
    \mathcal{W}_2^2(\nu_t, \hat{\rho}_t) &\le \int \|\bm{\phi}_t^{\theta}(\bm{s}_0)-\bm{\phi}_t^{\theta}(\bm{s}_0^{\xi})\|^2 \mathrm{d}\pi_0(\bm{s}_0, \bm{s}_0^{\xi}) \nonumber\\
    &\le \int e^{2(1+L_\theta)(t-t_0)} \|\bm{s}_0 - \bm{s}_0^{\xi}\|^2 \mathrm{d}\pi_0(\bm{s}_0, \bm{s}_0^{\xi}) \nonumber\\
    &= e^{2(1+L_\theta)(t-t_0)} \mathcal{W}_2^2(\nu_{t_0}, \hat{\rho}_{t_0})
\end{align}
Let $\tilde{\pi}(\bm{s_0}, \bm{s}_0^{\xi})$ be the coupling between $\nu_0$ and $\hat{\rho}_0$ induced by the identity mapping on the position space, i.e., pairing $\boldsymbol{s}_0=(\boldsymbol{x}_0, \boldsymbol{v}_0)$ with $\boldsymbol{s}_0^{\xi}=(\boldsymbol{x}_0, \boldsymbol{v}_0^{\xi})$. 
We can bound Wasserstein distance by $L_2$ loss:
\begin{align}
\mathcal{W}_2^2(\nu_{t_0}, \hat{\rho}_{t_0})
&\leq \int\|\bm{s}_0-\bm{s}_0^{\xi}\|^2\mathrm{d}\tilde{\pi}(\bm{s}_0, \bm{s}_0^{\xi})\nonumber \\
& =\int \left(\|\bm{x}_0-\bm{x}_0^{\xi}\|^2+\|\bm{v}_0-\bm{v}_0^{\xi}\|^2\right)\mathrm{d}\tilde{\pi}(\bm{x}_0, \bm{v}_0, \bm{x}_0^{\xi}, \bm{v}_0^{\xi})\nonumber \\
&=\int \|\bm{v}_0-\bm{v}^{\xi}_0\|^2\mathrm{d}\tilde{\pi}(\bm{x}_0, \bm{v}_0, {\bm{x}}_0^{\xi}, {\bm{v}}_0^{\xi})=
\mathcal{L}_{v_0}
\end{align}

Thus $\mathcal{W}_2^2(\nu_t, \hat{\rho}_t)\leq e^{2(1+L_\theta)(t-t_0)}\mathcal{L}_{v_0}.$

Combining the bounds above, we have

\begin{align}
\mathcal{W}_2^2(\hat{\rho}_{t_i}, \rho_{t_i}) 
&\leq 2\left[\mathcal{W}_2^2(\hat{\rho}_{t_i}, \nu_{t_i})+\mathcal{W}_2^2(\nu_{t_i}, \rho_{t_i})\right] \notag \\
&\leq 2\left[e^{(2L_\theta+3)(t_i-t_0)}\cdot (t_i-t_0)\mathcal{L}_{\text{AFM}}|_{[t_0, t_i]} + e^{(2L_\theta+2)(t_i-t_0)}\mathcal{L}_{v_0}\right].
\end{align}

Define $\kappa={(2L_\theta+3)}$,then the final bound is obtained by:
\begin{align}
\mathcal{W}_{2}^2(\hat \rho^{\text{(pos)}}_{t_i}, \rho^{(\text{pos})}_{t_i} = \mu_i^{\text{(pos)}})
& =\min_{\pi_{\boldsymbol{x}}} \int \|\boldsymbol{x}_{1} - \boldsymbol{x}_{2} \|^2 \mathrm{d} \pi_{\boldsymbol{x}}(\boldsymbol{x}_{1} , \boldsymbol{x}_{2}) \nonumber 
\quad \text{s.t.} \int \pi_{\boldsymbol{x}} \mathrm{d}\boldsymbol{x}_{2} =\rho^{(\text{pos})}_{t_i}, \quad \int \pi_{\boldsymbol{x}}\mathrm{d}\boldsymbol{x}_{1}  =\hat \rho_{t_i}^{\text{(pos)}} \nonumber \\
&\le \min_{\pi} \int \|\boldsymbol{x}_{1} - \boldsymbol{x}_{2} \|^2 \mathrm{d} \pi(\boldsymbol{s}_{1} , \boldsymbol{s}_{2}) 
\quad \text{s.t.} \int \pi\mathrm{d}\boldsymbol{s}_{2}  =\rho_{t_i}, \quad \int \pi\mathrm{d}\boldsymbol{s}_{1}  =\hat \rho_{t_i} \nonumber \\
&\le \min_{\pi} \int \| \boldsymbol{s}_{1} - \boldsymbol{s}_{2} \|^2 \mathrm{d} \pi(\boldsymbol{s}_{1}, \boldsymbol{s}_{2}) \label{eq:norm_ineq} \\
&=\mathcal{W}_2^2(\hat{\rho}_{t_i}, \rho_{t_i}) 
\leq 2e^{\kappa(t_i-t_0)}\left((t_i- t_0)\mathcal{L}_{\text{AFM}}|_{[t_0, t_i]}+\mathcal{L}_{v_0}\right).
\label{eq:final_bound}
\end{align}

This completes the proof. Note that our derivation holds even for multi-valued or stochastic cases, since the upper bound on the Wasserstein distance is established via a valid coupling, which remains well-defined regardless of whether the mapping is deterministic.

\end{proof}

\section{Datasets and Evaluation Metric}

\label{app:datasets_and_eval_metric}

\subsection{Experiment Setup}

\label{app:experimental setup}

The experiments were conducted on a shared high-performance computing (HPC) cluster. All experimental runs utilized GPUs equipped with 80GB of memory. Both the network $\boldsymbol{a}_{\theta}(\boldsymbol{x}, \boldsymbol{v}, t)$, employed to approximate acceleration, and the network $\boldsymbol{v}_{\phi}(\boldsymbol{x}, \boldsymbol{v}, t)$, employed to approximate velocity, utilize residual architectures. Specifically, $\boldsymbol{a}_{\theta}(\boldsymbol{x}, \boldsymbol{v}, t)$ consists of 4 hidden layers, while $\boldsymbol{v}_{\phi}(\boldsymbol{x}, \boldsymbol{v}, t)$ comprises 2 hidden layers. The Optimal Transport (OT) calculations were implemented using the Python Optimal Transport (POT) library \citep{flamary2021pot}, and the training process of neural networks were implemented using Pytorch \citep{pytorch}.

In all experiments, TracingFlow employs a completely consistent set of hyperparameters. When solving for OT using the Sinkhorn algorithm, the regularization coefficient is set to $\epsilon = 1\times 10^{-3}$. When incorporating biological priors, the penalty coefficient for barcode mismatches is set to $p_{0}=25$. Since the number of cells at each time point in the datasets processed in this paper is on the order of $10^{3}$, we do not utilize mini-batch OT; instead, we compute the Transport Plan on the full dataset.

\subsection{Evaluation Metrics}

\label{app:evaluation metrics}

We employ the 1-Wasserstein Distance ($\mathcal{W}_{1}$) and the 2-Wasserstein Distance ($\mathcal{W}_{2}$) to quantify the discrepancy between the predicted distribution and the ground truth distribution. These metrics are defined as follows:
\begin{equation}
    \mathcal{W}_{1}(\mu, \nu) = \min_{\pi \in \Pi(\mu, \nu)} \int \|\boldsymbol{x} - \boldsymbol{y}\| \mathrm{d} \pi(\boldsymbol{x}, \boldsymbol{y})
\end{equation}
\begin{equation}
    \mathcal{W}_{2}(\mu, \nu) = \left( \min_{\pi \in \Pi(\mu, \nu)} \int \|\boldsymbol{x} - \boldsymbol{y}\|_{2}^{2} \mathrm{d} \pi(\boldsymbol{x}, \boldsymbol{y}) \right)^{1/2}
\end{equation}
For the evaluation of datasets containing lineage information, we utilize the lineage-weighted $\mathcal{W}_{1}$ and $\mathcal{W}_{2}$ distances. These are defined based on the decomposition of the distributions. Consider the distribution generated by the model, denoted as $\mu(\boldsymbol{x})$, and the reference distribution provided by the dataset, denoted as $\nu(\boldsymbol{x})$. Both consist of $L$ lineages and can be decomposed as:
\begin{equation}
    \mu(\boldsymbol{x}) = \sum_{i=1}^{L} w_{i} \mu^{(i)}(\boldsymbol{x}), \quad \nu(\boldsymbol{x}) = \sum_{i=1}^{L} v_{i} \nu^{(i)}(\boldsymbol{x})
\end{equation}
where $\mu^{(i)}(\boldsymbol{x})$ and $\nu^{(i)}(\boldsymbol{x})$ represent the distributions of the $i$-th lineage within $\mu(\boldsymbol{x})$ and $\nu(\boldsymbol{x})$, respectively, satisfying $\int \mu^{(i)}(\boldsymbol{x}) \mathrm{d}\boldsymbol{x} = 1$ and $\int \nu^{(i)}(\boldsymbol{x}) \mathrm{d}\boldsymbol{x} = 1$. The terms $w_{i}$ and $v_{i}$ denote the weights of the $i$-th lineage in $\mu(\boldsymbol{x})$ and $\nu(\boldsymbol{x})$, respectively, subject to the constraint $\sum_{i=1}^{L} w_{i} = \sum_{i=1}^{L} v_{i} = 1$. The lineage-weighted $\mathcal{W}_{1}$ and $\mathcal{W}_{2}$ distances are defined as:
\begin{equation}
    \mathcal{W}_{1}^{(L)}(\mu \| \nu) = \sum_{i=1}^{L} v_{i} \mathcal{W}_{1}(\mu^{(i)}, \nu^{(i)}), \quad \mathcal{W}_{2}^{(L)}(\mu \| \nu) = \sum_{i=1}^{L} v_{i} \mathcal{W}_{2}(\mu^{(i)}, \nu^{(i)})
\end{equation}
In essence, we first calculate the $\mathcal{W}_{1}$ and $\mathcal{W}_{2}$ distances between the distributions of corresponding lineages in the two datasets. Subsequently, we compute a weighted average of these distances using the proportions of each lineage given by the reference distribution $\nu(\boldsymbol{x})$. It is important to note that the lineage-weighted $\mathcal{W}_{1}$ and $\mathcal{W}_{2}$ metrics serve solely as indicators of the model's ability to preserve lineage prior information while learning the dynamics. They are not true distance metrics in the mathematical sense, as they are evidently asymmetric with respect to $\mu$ and $\nu$.

\subsection{Datasets}

\label{app:datasets}

\vspace{0.5em}

\noindent \textbf{3D Simulation Lineage Data.} Consider the following dynamics with three lineage barcodes and a twist structure:    

\begin{equation}
\left\{
\begin{aligned}
dX_{t}^{(0)} &= v \, dt + \sigma_{x} dB_{t}^{X} \\
dY_{t}^{(0)} &=          \sigma_{y} dB_{t}^{Y} \\
dZ_{t}^{(0)} &=  \sigma_{z} dB_{t}^{Z}
\end{aligned}
\right.
\quad
\left\{
\begin{aligned}
dX_{t}^{(1)} &= v \, dt                + \sigma_{x} dB_{t}^{X} \\
dY_{t}^{(1)} &= -\mu Z_{t}^{(1)} \, dt + \sigma_{y} dB_{t}^{Y} \\
dZ_{t}^{(1)} &= \mu  Y_{t}^{(1)} \, dt + \sigma_{z} dB_{t}^{Z}
\end{aligned}
\right.
\quad
\left\{
\begin{aligned}
dX_{t}^{(2)} &= v \, dt                 + \sigma_{x} dB_{t}^{X} \\
dY_{t}^{(2)} &= \mu  Z_{t}^{(2)} \, dt  + \sigma_{y} dB_{t}^{Y} \\
dZ_{t}^{(2)} &= - \mu Y_{t}^{(2)} \, dt + \sigma_{z} dB_{t}^{Z}
\end{aligned}
\right.
\end{equation}
We take $v=3.0, \mu=3.14, \sigma_x=\sigma_y=\sigma_z=0.1$ and time points $t=0, 1, 2$. Initialize $(X_0^{(0)}, Y_0^{(0)}, Z_0^{(0)})\overset{}{\sim}\mathcal{N}([0, 0, 0], 0.5I)$, $(X_0^{(1)}, Y_0^{(1)}, Z_0^{(1)})\overset{}{\sim}\mathcal{N}([0,-2,0], 0.5I), (X_0^{(2)}, Y_0^{(2)}, Z_0^{(2)}){\sim}\mathcal{N}([0,2,0], 0.5I)$ and sample $500$ particles for each.

\begin{figure}[hbt!]
    \centering
    \begin{subfigure}{0.32\linewidth}
        \centering
        \includegraphics[width=\linewidth]{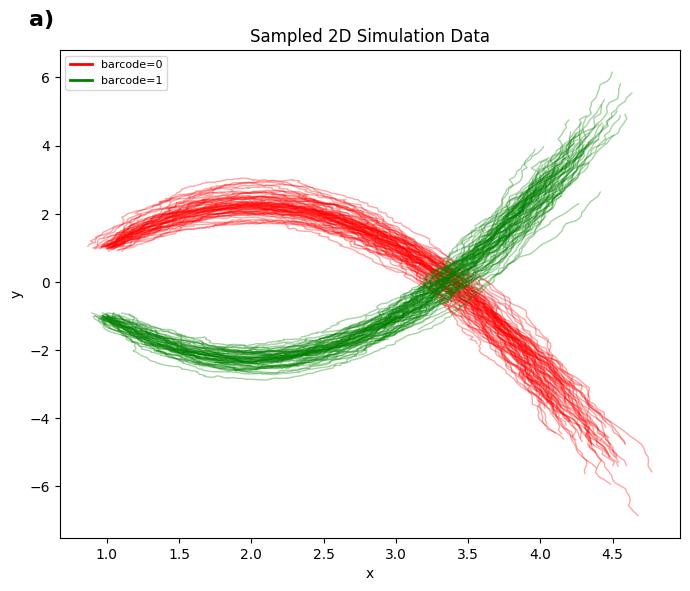}
        
        \label{2D}
    \end{subfigure}%
    \begin{subfigure}{0.6\linewidth}
        \centering
        \includegraphics[width=\linewidth]{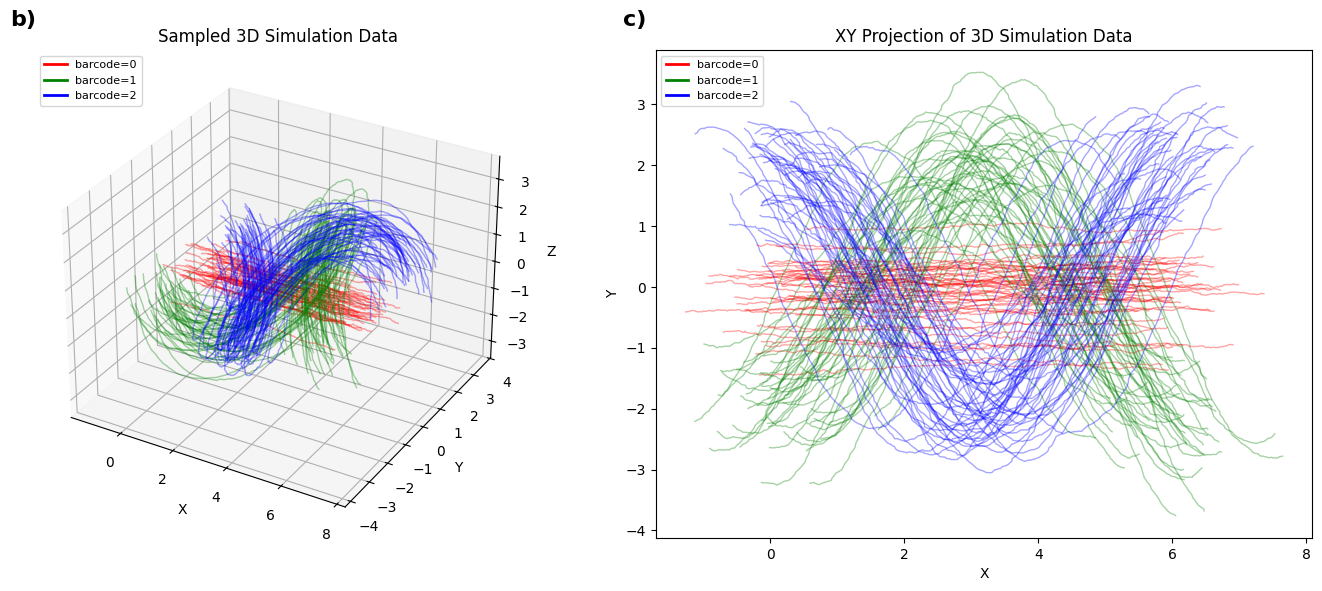}
        
        \label{3D}
    \end{subfigure}
    
    \caption{Dynamics of simulation datasets. (a)2D simulation data, (b)3D simulation data and (c)3D simulation data projected to XY plane. }
    \label{fig:2D_3D_combined}
\end{figure}

\vspace{0.5em}

\noindent \textbf{Cite Data.} We utilized the Cite-seq dataset introduced in \cite{lance2022multimodal}, which comprises 31,240 cells collected across 4 distinct time points. To evaluate the performance of TracingFlow on real-world biological data, we applied Principal Component Analysis (PCA) to reduce the dimensionality of the gene expression profiles to 100 and 5 dimensions, respectively.

\vspace{0.5em}

\noindent \textbf{Embryoid Bodies Data.} We employed the Embryoid Bodies (EB) dataset from \citep{moon2019visualizing}, consisting of 16,819 cells sampled at 5 time points. We utilized PCA to reduce the dimensionality of the gene expression data to 5 dimensions.

\vspace{0.5em}

\noindent \textbf{Hematopoiesis Data.} We employed the Hematopoiesis dataset from \citep{weinreb2020lineage}, consisting of 130,887 cells sampled at 3 time points. Part of the cells are recorded with barcodes. We selected 22,329 cells with barcodes from the dataset and  utilized PCA to reduce the dimensionality of gene expression data to 50 dimensions. When processing this dataset, we filtered for data where barcodes were present at all three time points. The barcodes for the remaining data were set to undefined, incurring no additional transport cost during the incorporation of biological priors (\cref{article:lineage}).

\noindent \textbf{Mexico Gulf Data} To assess the capability and robustness of Tracing Flow in capturing long-term system dynamics, we utilize the Mexico Gulf dataset, following \cite{theodoropoulos2025momentum}. This two-dimensional dataset comprises 9 time points.

\section{Experiment Details}

\label{app:experiments}

\subsection{Experiment Details across various datasets}

\label{app:exp_on_datasets}

\noindent \textbf{Simulation 2D Data}    \cref{tab:APP_2D_Simlulation}  presents the performance of TracingFlow and other algorithms on the 2D Simulation dataset. On average, TracingFlow reconstructs the data distribution at each time point with higher accuracy. 

\begin{table}[ht]
    \centering
    \caption{\(\mathcal{W}_{1}\) and \(\mathcal{W}_{2}\) distances between the generated and ground-truth distributions at each time point on the 2D Simulation dataset.}
    \label{tab:APP_2D_Simlulation}
    \resizebox{\linewidth}{!}{ 
    \begin{tabular}{llcccccccc}
        \toprule
        \multirow{2}{*}{\textbf{Method}} & \multirow{2}{*}{\textbf{Dynamics}} & \multicolumn{2}{c}{\textbf{\(t=1\)}} & \multicolumn{2}{c}{\textbf{\(t=2\)}} & \multicolumn{2}{c}{\textbf{\(t=3\)}} & \multicolumn{2}{c}{\textbf{\(t=4\)}} \\
        \cmidrule(lr){3-4} \cmidrule(lr){5-6} \cmidrule(lr){7-8} \cmidrule(lr){9-10}
         & & \(\mathcal{W}_{1}\) & \(\mathcal{W}_{2}\) & \(\mathcal{W}_{1}\) & \(\mathcal{W}_{2}\) & \(\mathcal{W}_{1}\) & \(\mathcal{W}_{2}\) & \(\mathcal{W}_{1}\) & \(\mathcal{W}_{2}\) \\
        \midrule
        OT-CFM & \multirow{2}{*}{1st Order} & \textbf{0.1215} {\scriptsize \(\pm\) 0.0164} & \textbf{0.1441} {\scriptsize \(\pm\) 0.0164} & 0.3644 {\scriptsize \(\pm\) 0.1177} & 0.4143 {\scriptsize \(\pm\) 0.1168} & 0.5863 {\scriptsize \(\pm\) 0.2615} & 0.6238 {\scriptsize \(\pm\) 0.2657} & 1.0356 {\scriptsize \(\pm\) 0.2325} & 1.1867 {\scriptsize \(\pm\) 0.2890} \\
        SF\(^2\)M & & 0.2389 {\scriptsize \(\pm\) 0.1217} & 0.2668 {\scriptsize \(\pm\) 0.1305} & 0.5700 {\scriptsize \(\pm\) 0.2574} & 0.6108 {\scriptsize \(\pm\) 0.2654} & 1.0630 {\scriptsize \(\pm\) 0.2769} & 1.3557 {\scriptsize \(\pm\) 0.2770}& 2.0986 {\scriptsize \(\pm\) 1.3198} & 2.2938 {\scriptsize \(\pm\) 1.3193} \\
        \midrule
        3MSBM & \multirow{5}{*}{2nd Order} & 1.4029 {\scriptsize \(\pm\) 0.6892} & 1.4459 {\scriptsize \(\pm\) 0.6596} & 0.9318 {\scriptsize \(\pm\) 0.2553} & 1.0181 {\scriptsize \(\pm\) 0.2050} & 1.8079 {\scriptsize \(\pm\) 0.7761} & 1.9162 {\scriptsize \(\pm\) 0.7529} & \textbf{0.5993} {\scriptsize \(\pm\) 0.6176} & \textbf{0.7265} {\scriptsize \(\pm\) 0.6935} \\
        MMFM & & 1.3226 {\scriptsize \(\pm\) 0.1774} & 1.9287 {\scriptsize \(\pm\) 0.3002} & 1.8387 {\scriptsize \(\pm\) 0.5624} & 1.1779 {\scriptsize \(\pm\) 0.5114} & 1.8158 {\scriptsize \(\pm\) 0.1794} & 1.8740 {\scriptsize \(\pm\) 0.1701} & 4.4984 {\scriptsize \(\pm\) 0.2937} & 4.6059 {\scriptsize \(\pm\) 0.2977} \\
        HRF & & 1.2036 {\scriptsize \(\pm\) 0.0425} & 1.2096 {\scriptsize \(\pm\) 0.0244} & 0.9955 {\scriptsize \(\pm\) 0.0921}& 1.0009 {\scriptsize \(\pm\) 0.0873}& 1.9829 {\scriptsize \(\pm\) 0.1311}& 2.0040 {\scriptsize \(\pm\) 0.1582}& 1.9740 {\scriptsize \(\pm\) 0.1873}& 2.4379 {\scriptsize \(\pm\) 0.2157}\\
        CAF & & 0.4051 {\scriptsize \(\pm\) 0.1712} & 0.5117 {\scriptsize \(\pm\) 0.1738} & 0.8857 {\scriptsize \(\pm\) 0.3630}& 0.8996 {\scriptsize \(\pm\) 0.3689} & 1.3274 {\scriptsize \(\pm\) 0.1194} & 1.3731 {\scriptsize \(\pm\) 0.1107} & 3.9358 {\scriptsize \(\pm\) 0.4572} & 3.7880 {\scriptsize \(\pm\) 0.4530}\\
        \textbf{TF(Ours)} & & 0.1609 {\scriptsize \(\pm\) 0.0232} & 0.1787 {\scriptsize \(\pm\) 0.0258} & \textbf{0.1753} {\scriptsize \(\pm\) 0.0299} & \textbf{0.2045} {\scriptsize \(\pm\) 0.0277} & \textbf{0.4091} {\scriptsize \(\pm\) 0.0523} & \textbf{0.4821} {\scriptsize \(\pm\) 0.0698} & {0.6404} {\scriptsize \(\pm\) 0.2661}& 1.1436 {\scriptsize \(\pm\) 0.1577}\\
        \bottomrule
    \end{tabular}
    }
\end{table}


\noindent \textbf{Cite 5D \& Cite 100D Data}  \ We present the distribution reconstruction performance of each algorithm on the Cite 5D and Cite 100D datasets in \cref{tab:exp_APP_cite_5D} and \cref{tab:exp_APP_cite_100D}, respectively. On these datasets TracingFlow achieves superior reconstruction accuracy at all time steps except for $t=3$. The trajectories learned by TracingFlow on these two datasets are illustrated in  \cref{fig:exp_APP_cite_5D_and_100D}.

\begin{table}[H]
    \centering
    \caption{\(\mathcal{W}_{1}\) and \(\mathcal{W}_{2}\) distances between the generated and ground-truth distributions at each time point on the Cite 5D dataset}
    \label{tab:exp_APP_cite_5D}
    \resizebox{\linewidth}{!}{ 
    \begin{tabular}{llcccccc}
        \toprule
        \multirow{2}{*}{\textbf{Method}} & \multirow{2}{*}{\textbf{Dynamics}} & \multicolumn{2}{c}{\textbf{\(t=1\)}} & \multicolumn{2}{c}{\textbf{\(t=2\)}} & \multicolumn{2}{c}{\textbf{\(t=3\)}} \\
        \cmidrule(lr){3-4} \cmidrule(lr){5-6} \cmidrule(lr){7-8}
         & & \(\mathcal{W}_{1}\) & \(\mathcal{W}_{2}\) & \(\mathcal{W}_{1}\) & \(\mathcal{W}_{2}\) & \(\mathcal{W}_{1}\) & \(\mathcal{W}_{2}\) \\
        \midrule
        OT-CFM & \multirow{2}{*}{1st Order} & 0.5615  {\scriptsize \(\pm\) 0.0156} & 0.6484  {\scriptsize \(\pm\) 0.0268} & 0.8472  {\scriptsize \(\pm\) 0.0463}& 0.9416  {\scriptsize \(\pm\) 0.0443} & 1.0536  {\scriptsize \(\pm\) 0.0534} & 1.1566  {\scriptsize \(\pm\) 0.0573} \\
        SF\(^2\)M & & 0.5932 {\scriptsize \(\pm\) 0.0249} & 0.6797 {\scriptsize \(\pm\) 0.0424} & 0.7768 {\scriptsize \(\pm\) 0.0451} & 0.8670 {\scriptsize \(\pm\) 0.0480} & 0.9251 {\scriptsize \(\pm\) 0.0279} & 1.0221 {\scriptsize \(\pm\) 0.0352} \\
        \midrule
        3MSBM & \multirow{5}{*}{2nd Order} & 5.3735  {\scriptsize \(\pm\) 0.3961} & 5.4244  {\scriptsize \(\pm\) 0.3582}& 4.3118  {\scriptsize \(\pm\) 0.4589} & 4.5458  {\scriptsize \(\pm\) 0.4141}& 1.2254  {\scriptsize \(\pm\) 0.2998}& 1.3427  {\scriptsize \(\pm\) 0.3307}\\
        MMFM & & 0.5335  {\scriptsize \(\pm\) 0.0134} & 0.5911  {\scriptsize \(\pm\) 0.0057} & 2.2290  {\scriptsize \(\pm\) 0.2151} & 2.4557  {\scriptsize \(\pm\) 0.1613} & 4.8413  {\scriptsize \(\pm\) 0.2538} & 5.2261  {\scriptsize \(\pm\) 0.4061} \\
        HRF & & 0.4225 {\scriptsize \(\pm\) 0.0152} & 0.4622 {\scriptsize \(\pm\) 0.0185}& 0.6227 {\scriptsize \(\pm\) 0.0421} & 0.6814 {\scriptsize \(\pm\) 0.0392}& \textbf{0.7638} {\scriptsize \(\pm\) 0.0738}& \textbf{0.8394} {\scriptsize \(\pm\) 0.0713}\\
        CAF & & 5.1724{\scriptsize \(\pm\) 0.0452} & 5.6242{\scriptsize \(\pm\) 0.0763} & 5.4333 {\scriptsize \(\pm\) 0.0914} & 5.4637 {\scriptsize \(\pm\) 0.1425} & 3.0710 {\scriptsize \(\pm\) 0.1853}& 3.1354 {\scriptsize \(\pm\) 0.2672}\\
        \textbf{TF (Ours)} & & \textbf{0.2682} {\scriptsize \(\pm\) 0.0241} & \textbf{0.2916}{\scriptsize \(\pm\) 0.0312} & \textbf{0.6094} {\scriptsize \(\pm\) 0.0615} & \textbf{0.6516} {\scriptsize \(\pm\) 0.0684}& {0.8808} {\scriptsize \(\pm\) 0.1208} & {0.9725} {\scriptsize \(\pm\) 0.1329}\\
        \bottomrule
    \end{tabular}
    }
\end{table}

\begin{table}[H]
    \centering
    \caption{$\mathcal{W}_{1}$ and $\mathcal{W}_{2}$ distances between the generated and ground-truth distributions at each time point on the Cite 100D dataset}
    \label{tab:exp_APP_cite_100D}
    \resizebox{\linewidth}{!}{ 
    \begin{tabular}{llcccccc}
        \toprule
        \multirow{2}{*}{\textbf{Method}} & \multirow{2}{*}{\textbf{Dynamics}} & \multicolumn{2}{c}{\textbf{$t=1$}} & \multicolumn{2}{c}{\textbf{$t=2$}} & \multicolumn{2}{c}{\textbf{$t=3$}} \\
        \cmidrule(lr){3-4} \cmidrule(lr){5-6} \cmidrule(lr){7-8}
         & & $\mathcal{W}_{1}$ & $\mathcal{W}_{2}$ & $\mathcal{W}_{1}$ & $\mathcal{W}_{2}$ & $\mathcal{W}_{1}$ & $\mathcal{W}_{2}$ \\
        \midrule
        OT-CFM & \multirow{2}{*}{1st Order} & 10.1136 {\scriptsize \(\pm\) 0.0104} & 10.1873 {\scriptsize \(\pm\) 0.0135} & 10.2504 {\scriptsize \(\pm\) 0.0209} & 10.3145 {\scriptsize \(\pm\) 0.0201} & {11.2682} {\scriptsize \(\pm\) 0.0938} & {11.3642} {\scriptsize \(\pm\) 0.0997} \\
        SF$^2$M & & 10.3598 {\scriptsize \(\pm\) 0.0051} & 10.4448 {\scriptsize \(\pm\) 0.0053} & 11.4647 {\scriptsize \(\pm\) 0.0384} & 11.5456 {\scriptsize \(\pm\) 0.0402} & 15.2032 {\scriptsize \(\pm\) 0.1813} & 15.4367 {\scriptsize \(\pm\) 0.1949} \\
        \midrule
        3MSBM & \multirow{5}{*}{2nd Order} & 19.9677{\scriptsize \(\pm\) 0.1373} & 20.1015 {\scriptsize \(\pm\) 0.1105} & 15.3218 {\scriptsize \(\pm\) 0.1232} & 15.4089 {\scriptsize \(\pm\) 0.1343}& 12.0229 {\scriptsize \(\pm\) 0.2895} & 12.1395 {\scriptsize \(\pm\) 0.3621} \\
        MMFM & & 9.8726 {\scriptsize \(\pm\) 0.0383} & 9.9305 {\scriptsize \(\pm\) 0.0361} & 12.1374 {\scriptsize \(\pm\) 0.2422} & 12.3076 {\scriptsize \(\pm\) 0.2495} & 15.5960 {\scriptsize \(\pm\) 1.4972} & 16.1222 {\scriptsize \(\pm\) 2.0195} \\
        HRF & & 9.9660 {\scriptsize \(\pm\) 0.0215}& 10.0580 {\scriptsize \(\pm\) 0.0341}& 9.9451 {\scriptsize \(\pm\) 0.0632}& 9.9986 {\scriptsize \(\pm\) 0.1526}& \textbf{10.5485} {\scriptsize \(\pm\) 0.1082}& \textbf{10.6442} {\scriptsize \(\pm\) 0.1254}\\
        CAF & & 18.1182 {\scriptsize \(\pm\) 0.3214}& 18.1971 {\scriptsize \(\pm\) 0.2845}& 17.4947 {\scriptsize \(\pm\) 0.4267}& 17.5767 {\scriptsize \(\pm\) 0.4912}& 14.6158 {\scriptsize \(\pm\) 0.5632}& 15.1773 {\scriptsize \(\pm\) 0.5215}\\
        \textbf{TF (Ours)} & & \textbf{2.8955} {\scriptsize \(\pm\) 0.1429}& \textbf{3.2763} {\scriptsize \(\pm\) 0.1776}& \textbf{9.8235} {\scriptsize \(\pm\) 0.1489}& \textbf{9.8773}{\scriptsize \(\pm\) 0.1827} & 11.4910 {\scriptsize \(\pm\) 0.1693}& 11.6054{\scriptsize \(\pm\) 0.2064} \\
        \bottomrule
    \end{tabular}
    }
\end{table}

\begin{figure}[hbt!]
    \centering
    \includegraphics[width=0.80\linewidth]{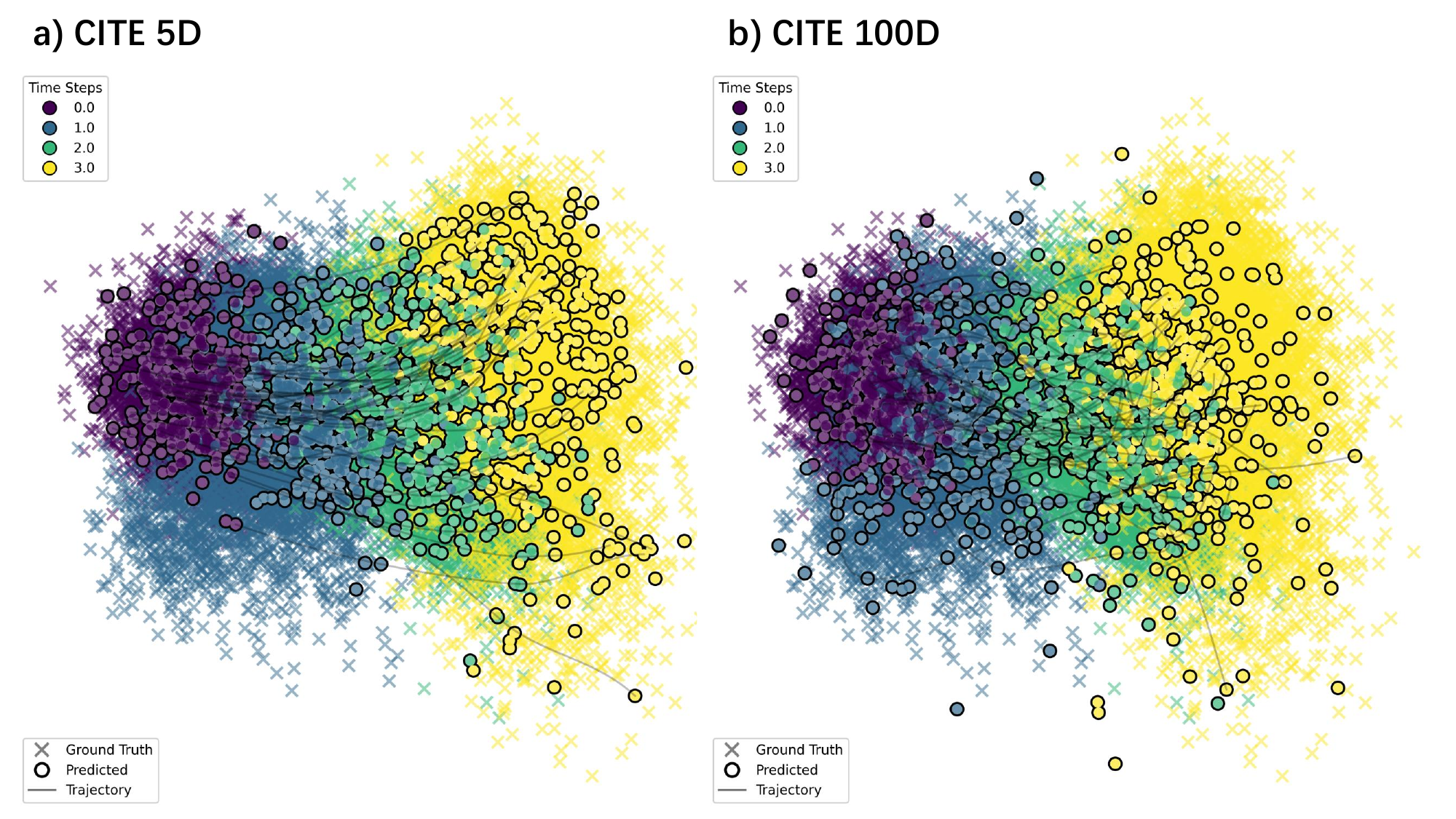}
    \caption{Trajectories learned by TracingFlow on the (a) Cite 5D and (b) Cite 100D datasets, projected to 2D using PCA.}
    \label{fig:exp_APP_cite_5D_and_100D}
\end{figure}

\noindent \textbf{EB 5D Data} \cref{tab:APP_EB_5D} displays the interpolation accuracy on the EB 5D dataset, evaluated by holding out one time point during training. TracingFlow outperformed other methods in reconstruction accuracy for all time points except $t=1$. \cref{fig:exp_appendix_EB} illustrates the interpolated distributions generated by TracingFlow across the four time points.

\begin{table}[H]
    \centering
    \caption{$\mathcal{W}_{1}$ and $\mathcal{W}_{2}$ distances between predicted and ground-truth distributions at held-out time points on the EB 5D dataset.}
    \label{tab:APP_EB_5D}
    \resizebox{\linewidth}{!}{ 
    \begin{tabular}{llcccccccc}
        \toprule
        \multirow{2}{*}{\textbf{Method}} & \multirow{2}{*}{\textbf{Dynamics}} & \multicolumn{2}{c}{\textbf{$t=1$}} & \multicolumn{2}{c}{\textbf{$t=2$}} & \multicolumn{2}{c}{\textbf{$t=3$}} & \multicolumn{2}{c}{\textbf{$t=4$}} \\
        \cmidrule(lr){3-4} \cmidrule(lr){5-6} \cmidrule(lr){7-8} \cmidrule(lr){9-10}
         & & $\mathcal{W}_{1}$ & $\mathcal{W}_{2}$ & $\mathcal{W}_{1}$ & $\mathcal{W}_{2}$ & $\mathcal{W}_{1}$ & $\mathcal{W}_{2}$ & $\mathcal{W}_{1}$ & $\mathcal{W}_{2}$ \\
        \midrule
        OT-CFM & \multirow{2}{*}{1st Order} 
        & 3.4335 {\scriptsize $\pm$ 0.1465} & 3.8462 {\scriptsize $\pm$ 0.1449} 
        & 4.8404 {\scriptsize $\pm$ 0.1496} & 5.5792 {\scriptsize $\pm$ 0.1358} 
        & 5.0713 {\scriptsize $\pm$ 0.2839} & 5.6448 {\scriptsize $\pm$ 0.2861} 
        & 7.1368 {\scriptsize $\pm$ 0.3799} & 7.6938 {\scriptsize $\pm$ 0.4574} \\
        
        SF$^2$M & 
        & \textbf{3.4204} {\scriptsize $\pm$ 0.0637} & \textbf{3.6713} {\scriptsize $\pm$ 0.0589} 
        & \textbf{3.5511} {\scriptsize $\pm$ 0.0394} & 3.9441 {\scriptsize $\pm$ 0.0418} 
        & 4.3175 {\scriptsize $\pm$ 0.2517} & 4.8738 {\scriptsize $\pm$ 0.3247} 
        & \textbf{6.9516} {\scriptsize $\pm$ 0.0616} & 7.4727 {\scriptsize $\pm$ 0.0795} \\
        \midrule
        
        3MSBM & \multirow{3}{*}{2nd Order} 
        & 6.1086 {\scriptsize $\pm$ 1.1215} & 6.5434 {\scriptsize $\pm$ 1.2227} 
        & 7.3626 {\scriptsize $\pm$ 0.8043} & 8.0596 {\scriptsize $\pm$ 1.0252} 
        & 7.5862 {\scriptsize $\pm$ 1.1170} & 8.1734 {\scriptsize $\pm$ 0.7900} 
        & 7.7029 {\scriptsize $\pm$ 0.4241} & \textbf{7.1022} {\scriptsize $\pm$ 0.3971} \\
        
        MMFM & 
        & 10.7803 {\scriptsize $\pm$ 0.6718} & 11.8376 {\scriptsize $\pm$ 0.7823} 
        & 5.7932 {\scriptsize $\pm$ 0.2162} & 6.2242 {\scriptsize $\pm$ 0.1828} 
        & 9.3326 {\scriptsize $\pm$ 0.5617} & 9.9837 {\scriptsize $\pm$ 0.4994} 
        & 13.2046 {\scriptsize $\pm$ 2.1528} & 14.3464 {\scriptsize $\pm$ 2.2982} \\
        
        \textbf{TF (Ours)} & 
        & 3.5611 {\scriptsize $\pm$ 0.0362} & 3.8930 {\scriptsize $\pm$ 0.0243} 
        & 3.5002 {\scriptsize $\pm$ 0.0983} & \textbf{3.8035} {\scriptsize $\pm$ 0.1086} 
        & \textbf{3.5643} {\scriptsize $\pm$ 0.1140} & \textbf{3.7947} {\scriptsize $\pm$ 0.4123} 
        & 7.3639 {\scriptsize $\pm$ 0.1316} & 7.6523 {\scriptsize $\pm$ 0.1976} \\
        \bottomrule
    \end{tabular}
    }
\end{table}

\begin{figure}[hbt!]
    \centering
    \includegraphics[width=\linewidth]{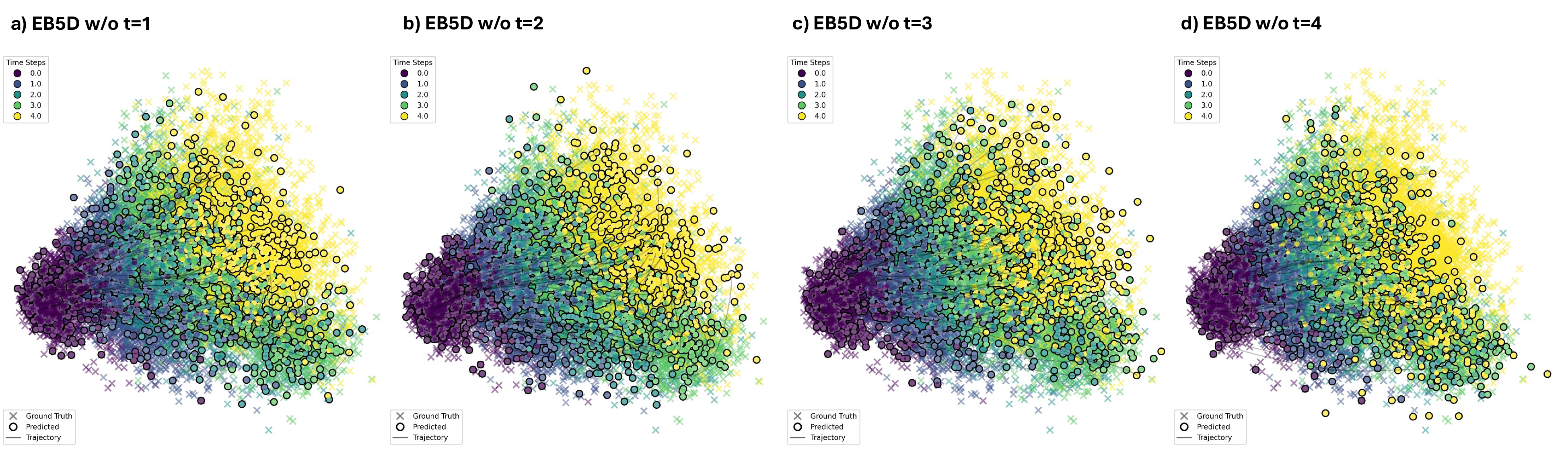}
    \caption{Predicted distributions by TracingFlow at four held-out time points, projected to 2D using PCA.}
    \label{fig:exp_appendix_EB}
\end{figure}

\noindent \textbf{3D Simulation Lineage Dataset}  \cref{tab:APP_TAB_3D_SimLineage} demonstrates the ability of various algorithms to preserve biological priors on the 3D Simulation Lineage dataset. Since the trajectories learned by TracingFlow align with biological priors (see \cref{fig:exp_lineage_tracing_1}), it achieves lower lineage-weighted $\mathcal{W}{1}$ and $\mathcal{W}{2}$ scores on average.

\begin{table}[H]
    \centering
    \caption{lineage-weighted $\mathcal{W}_{1}$ and $\mathcal{W}_{2}$ distances between the generated and ground-truth distributions at each time point on the 3D Simulation Lineage dataset}
    \label{tab:APP_TAB_3D_SimLineage}
    \resizebox{0.9\linewidth}{!}{ 
    \begin{tabular}{llcccc}
        \toprule
        \multirow{2}{*}{\textbf{Method}} & \multirow{2}{*}{\textbf{Dynamics}} & \multicolumn{2}{c}{\textbf{$t=1$}} & \multicolumn{2}{c}{\textbf{$t=2$}} \\
        \cmidrule(lr){3-4} \cmidrule(lr){5-6}
         & & $\mathcal{W}_{1}$ & $\mathcal{W}_{2}$ & $\mathcal{W}_{1}$ & $\mathcal{W}_{2}$ \\
        \midrule
        OT-CFM & \multirow{2}{*}{1st Order} & 2.6080 {\scriptsize $\pm$ 0.0029} & 2.6869 {\scriptsize $\pm$ 0.0030}& 1.8686 {\scriptsize $\pm$ 0.0382}& 1.9456 {\scriptsize $\pm$ 0.0403}\\
        SF$^2$M & & 2.8262 {\scriptsize $\pm$ 0.0409}& 2.8373 {\scriptsize $\pm$ 0.0462}& 0.2247 {\scriptsize $\pm$ 0.0617}& 0.2533 {\scriptsize $\pm$ 0.0568}\\
        \midrule
        3MSBM & \multirow{6}{*}{2nd Order} & 4.0141 {\scriptsize $\pm$ 0.8788}& 4.2262 {\scriptsize $\pm$ 0.8596}& 0.7548 {\scriptsize $\pm$ 0.0726}& 0.6751 {\scriptsize $\pm$ 0.0573}\\
        MMFM & & 1.8585 {\scriptsize $\pm$ 0.0350}& 2.1618 {\scriptsize $\pm$ 0.0496}& 3.5695 {\scriptsize $\pm$ 0.5692}& 4.1028 {\scriptsize $\pm$ 0.6037}\\
        HRF & & 2.9150 {\scriptsize $\pm$ 0.2595}& 2.9265 {\scriptsize $\pm$ 0.3901}& 0.6717 {\scriptsize $\pm$ 0.5491}& 0.6905 {\scriptsize $\pm$ 0.4143}\\
        CAF & & 1.6284 {\scriptsize $\pm$ 0.2249}& 1.6623 {\scriptsize $\pm$ 0.3726}& 3.7328 {\scriptsize $\pm$ 0.3305}& 3.8471 {\scriptsize $\pm$ 0.3316}\\
        TF w/o Bio. Prior & & 2.8267 {\scriptsize $\pm$ 0.2896} & 2.8376 {\scriptsize $\pm$ 0.3714} & {0.2317} {\scriptsize $\pm$ 0.0512} & {0.2687} {\scriptsize $\pm$ 0.0892} \\
        \textbf{TF(Ours)} & & \textbf{0.4282} {\scriptsize $\pm$ 0.1606}& \textbf{0.4989} {\scriptsize $\pm$ 0.1585}& 0.4795 {\scriptsize $\pm$ 0.2680}& 0.5667 {\scriptsize $\pm$ 0.2654}\\
        \bottomrule
    \end{tabular}
    }
\end{table}


\noindent \textbf{Hematopoiesis Dataset}  \cref{tab:APP_Real_Lineage} presents the performance of different algorithms in retaining biological priors on real-world datasets, measured by the lineage-weighted $\mathcal{W}{1}, \mathcal{W}{2}$ distance. TracingFlow surpassed other algorithms at all time points. We also evaluated the algorithms without introducing biological priors by Vanilla $\mathcal{W}_1$, $\mathcal{W}_2$ distance, and as shown in \cref{tab:APP_Real_Lineage_2}, TracingFlow maintained superior performance. \cref{fig:exp_appendix_real_lineage} visualizes the data points generated by TracingFlow on the real-world dataset under both conditions (with and without biological priors).

\begin{table}[H]
    \centering
    \caption{lineage-weighted $\mathcal{W}_{1}$ and $\mathcal{W}_{2}$ distances between the generated and ground-truth distributions at each time point on the Hematopoiesis dataset}
    \label{tab:APP_Real_Lineage}
    \resizebox{0.9\linewidth}{!}{ 
    \begin{tabular}{llcccc}
        \toprule
        \multirow{2}{*}{\textbf{Method}} & \multirow{2}{*}{\textbf{Dynamics}} & \multicolumn{2}{c}{\textbf{$t=1$}} & \multicolumn{2}{c}{\textbf{$t=2$}} \\
        \cmidrule(lr){3-4} \cmidrule(lr){5-6}
         & & $\mathcal{W}_{1}$ & $\mathcal{W}_{2}$ & $\mathcal{W}_{1}$ & $\mathcal{W}_{2}$ \\
        \midrule
        OT-CFM & \multirow{2}{*}{1st Order} & 13.2972 {\scriptsize $\pm$ 0.0804}& 13.7134 {\scriptsize $\pm$ 0.0831}& 16.6918 {\scriptsize $\pm$ 0.1973}& 17.0807 {\scriptsize $\pm$ 0.2025}\\
        SF$^2$M & & 14.1427 {\scriptsize $\pm$ 0.0341}& 14.4199 {\scriptsize $\pm$ 0.0303}& 23.2913 {\scriptsize $\pm$ 0.2590}& 23.5314 {\scriptsize $\pm$ 0.2640}\\
        \midrule
        3MSBM & \multirow{6}{*}{2nd Order} & 17.6487 {\scriptsize $\pm$ 0.5250} & 17.8710 {\scriptsize $\pm$ 0.5204} & 19.0139 {\scriptsize $\pm$ 1.4418}& 19.4144 {\scriptsize $\pm$ 1.3622}\\
        MMFM & & 18.4040 {\scriptsize $\pm$ 2.8129}& 19.8538 {\scriptsize $\pm$ 2.8906}& 36.0412 {\scriptsize $\pm$ 3.0429}& 40.2282 {\scriptsize $\pm$ 3.3934}\\
        HRF & & 14.1007 {\scriptsize $\pm$ 0.0804}& 14.3122 {\scriptsize $\pm$ 0.0831}& 16.8135 {\scriptsize $\pm$ 0.1973}& 17.1154 {\scriptsize $\pm$ 0.2025}\\
        CAF & & 13.1098 {\scriptsize $\pm$ 0.0628}&  13.3263 {\scriptsize $\pm$ 0.0742}& 17.3232 {\scriptsize $\pm$ 0.0580} & 17.6685 {\scriptsize $\pm$ 0.0572}\\
        TF w/o Bio. Prior & & 14.1260 {\scriptsize $\pm$ 0.0526}& 14.5958 {\scriptsize $\pm$ 0.0537} & 19.1073 {\scriptsize $\pm$  0.1932}& 19.4262 {\scriptsize $\pm$ 0.2375}\\
        \textbf{TF} & & \textbf{11.5203} {\scriptsize $\pm$ 0.0695}& \textbf{12.0151} {\scriptsize $\pm$ 0.0809}& \textbf{16.5736} {\scriptsize $\pm$ 0.2603}& \textbf{16.9027} {\scriptsize $\pm$ 0.2999}\\
        \bottomrule
    \end{tabular}
    }
\end{table}

\begin{table}
\centering

\begin{tabular}{l l l l}
\hline

\end{tabular}

\end{table}

\begin{table}[H]
    \centering
    \caption{$\mathcal{W}{1}$ and $\mathcal{W}{2}$ distances between the generated and ground-truth distributions at each time point on the Hematopoiesis dataset}
    \label{tab:APP_Real_Lineage_2}
    \resizebox{0.75\linewidth}{!}{ 
    \begin{tabular}{llcccccc}
        \toprule
        \multirow{2}{*}{\textbf{Method}} & \multirow{2}{*}{\textbf{Dynamics}} & \multicolumn{2}{c}{\textbf{$t=1$}} & \multicolumn{2}{c}{\textbf{$t=2$}} & \multicolumn{2}{c}{\textbf{Average}} \\
        \cmidrule(lr){3-4} \cmidrule(lr){5-6} \cmidrule(lr){7-8} 
         & & $\mathcal{W}_{1}$ & $\mathcal{W}_{2}$ & $\mathcal{W}_{1}$ & $\mathcal{W}_{2}$ & $\mathcal{W}_{1}$ & $\mathcal{W}_{2}$ \\
        \midrule
        OT-CFM & \multirow{2}{*}{First-order} & 9.8780 & 10.6132 & 10.7621 & 11.4466 & 10.3201 & 11.0299 \\
        SF$^2$M & & 10.2352 & 10.8017 & 13.2507 & 13.6879 & 11.7430 & 12.2428 \\
        \midrule
        3MSBM & \multirow{5}{*}{Second-order} & 12.8214 & 13.2501 & 20.0262 & 20.6374 & 16.4238 & 16.9438 \\
        MMFM & & 10.4305 & 10.8695 & 22.3269 & 23.5322 & 16.3787 & 17.2008 \\
        HRF & & 10.6233 & 11.4510 & 10.9570 & 11.6201 & 10.7902 & 11.5356 \\
        CAF & & 21.8716 & 22.3309 & 17.1251 & 18.3914 & 19.4984 & 20.3612 \\
        TF w/o Bio. Prior & & \textbf{7.0139} & \textbf{7.4140} & \textbf{10.1642} & \textbf{10.5925} & \textbf{8.5891} & \textbf{9.0032} \\
        \bottomrule
    \end{tabular}
    }
\end{table}

\begin{figure}[hbt!]
    \centering
    \includegraphics[width=0.85\linewidth]{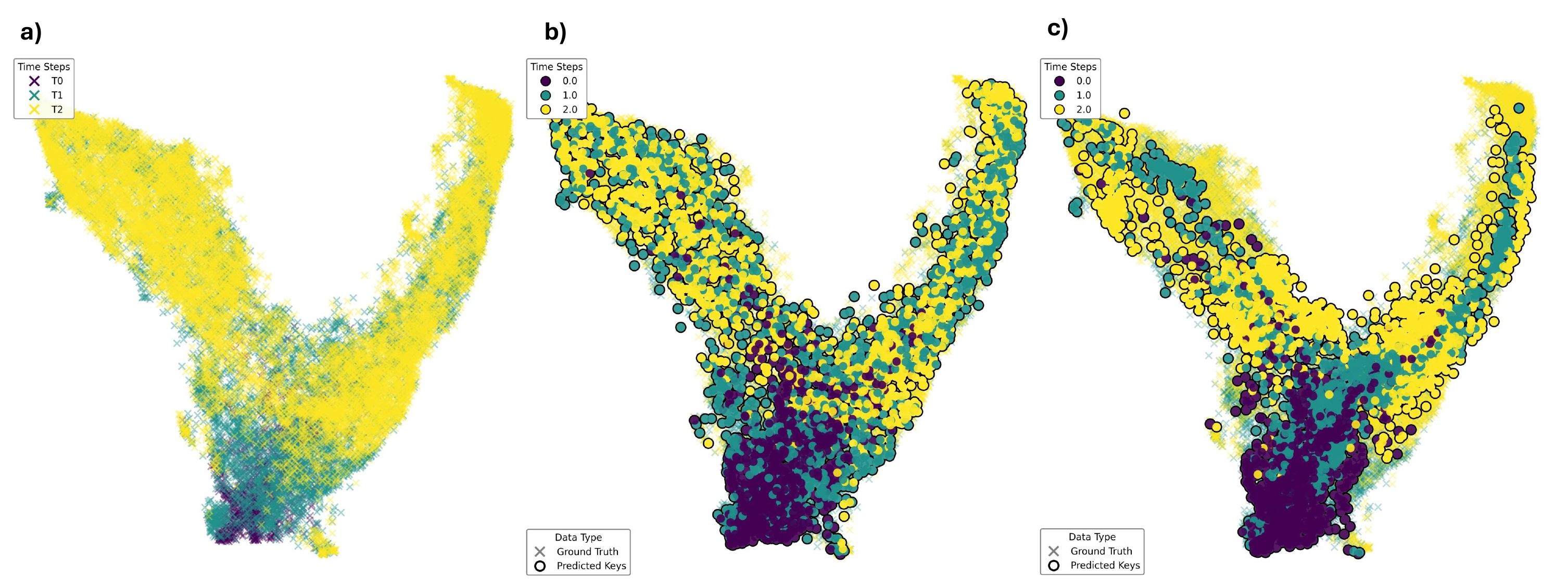}
    \caption{Original data (a), and data points generated by TracingFlow on the Hematopoiesis dataset  without biological priors (b) and with biological priors (c), projected to 2D using UMAP.}
    \label{fig:exp_appendix_real_lineage}
\end{figure}

\noindent \textbf{Mexico Gulf Dataset} \  \cref{tab:mexico_gulf} presents the performance of various algorithms on the Mexico Gulf dataset. Tracing Flow achieves the best distribution reconstruction accuracy at all time points, yielding smooth trajectories.

\begin{table}[H]
    \centering
    \caption{$\mathcal{W}_{1}$ and $\mathcal{W}_{2}$ distances between the generated and ground-truth distributions at each time point on the Mexico Gulf dataset}
    \label{tab:mexico_gulf}
    
    \resizebox{\linewidth}{!}{ 
    \begin{tabular}{llcccccccc}
        \toprule
        \multirow{2}{*}{\textbf{Method}} & \multirow{2}{*}{\textbf{Dynamics}} & \multicolumn{2}{c}{\textbf{$t=1$}} & \multicolumn{2}{c}{\textbf{$t=2$}} & \multicolumn{2}{c}{\textbf{$t=3$}} & \multicolumn{2}{c}{\textbf{$t=4$}} \\
        \cmidrule(lr){3-4} \cmidrule(lr){5-6} \cmidrule(lr){7-8} \cmidrule(lr){9-10}
         & & $\mathcal{W}_{1}$ & $\mathcal{W}_{2}$ & $\mathcal{W}_{1}$ & $\mathcal{W}_{2}$ & $\mathcal{W}_{1}$ & $\mathcal{W}_{2}$ & $\mathcal{W}_{1}$ & $\mathcal{W}_{2}$ \\
        \midrule
        OT-CFM & \multirow{2}{*}{First-order} & 0.0336 & 0.0375 & 0.0791 & 0.0887 & 0.1367 & 0.1415 & 0.1409 & 0.1474 \\
        SF$^2$M & & 0.0768 & 0.0917 & 0.0802 & 0.0964 & 0.0819 & 0.0968 & 0.0809 & 0.0987 \\
        \midrule
        3MSBM & \multirow{5}{*}{Second-order} & 0.4407 & 0.6305 & 0.5565 & 0.7617 & 0.5887 & 0.9113 & 0.5231 & 1.0714 \\
        MMFM & & 0.0346 & 0.0364 & 0.2961 & 0.3008 & 1.0538 & 1.0571 & 1.8755 & 1.8766 \\
        HRF & & 0.6205 & 0.6217 & 1.2004 & 1.2011 & 1.5334 & 1.5337 & 1.4986 & 1.5013 \\
        CAF & & 1.7659 & 1.7663 & 2.4713 & 2.4714 & 2.0742 & 2.0755 & 1.3040 & 1.3071 \\
        TF & & \textbf{0.0128} & \textbf{0.0145} & \textbf{0.0301} & \textbf{0.0339} & \textbf{0.0407} & \textbf{0.0458} & \textbf{0.0348} & \textbf{0.0456} \\
        \bottomrule
    \end{tabular}
    }
    
    \vspace{0.1em} 
    
    \resizebox{\linewidth}{!}{ 
    \begin{tabular}{llcccccccccc}
        \toprule
        \multirow{2}{*}{\textbf{Method}} & \multirow{2}{*}{\textbf{Dynamics}} & \multicolumn{2}{c}{\textbf{$t=5$}} & \multicolumn{2}{c}{\textbf{$t=6$}} & \multicolumn{2}{c}{\textbf{$t=7$}} & \multicolumn{2}{c}{\textbf{$t=8$}} & \multicolumn{2}{c}{\textbf{Average}} \\
        \cmidrule(lr){3-4} \cmidrule(lr){5-6} \cmidrule(lr){7-8} \cmidrule(lr){9-10} \cmidrule(lr){11-12}
         & & $\mathcal{W}_{1}$ & $\mathcal{W}_{2}$ & $\mathcal{W}_{1}$ & $\mathcal{W}_{2}$ & $\mathcal{W}_{1}$ & $\mathcal{W}_{2}$ & $\mathcal{W}_{1}$ & $\mathcal{W}_{2}$ & $\mathcal{W}_{1}$ & $\mathcal{W}_{2}$ \\
        \midrule
        OT-CFM & \multirow{2}{*}{First-order} & 0.1318 & 0.1370 & 0.1364 & 0.1425 & 0.1610 & 0.1799 & 0.1871 & 0.2236 & 0.1258 & 0.1373 \\
        SF$^2$M & & 0.1030 & 0.1191 & 0.0851 & 0.0956 & 0.1401 & 0.1494 & 0.1965 & 0.2068 & 0.1056 & 0.1193 \\
        \midrule
        3MSBM & \multirow{5}{*}{Second-order} & 0.2185 & 1.2783 & 0.5362 & 2.0436 & 0.8474 & 4.5953 & 1.3606 & 13.2434 & 0.6340 & 3.0669 \\
        MMFM & & 2.3926 & 2.3928 & 2.7214 & 2.7215 & 2.9191 & 2.9199 & 2.9930 & 2.9945 & 1.7857 & 1.7875 \\
        HRF & & 1.3022 & 1.3075 & 0.9476 & 0.9518 & 0.6480 & 0.6535 & 0.3296 & 0.3341 & 1.0101 & 1.0131 \\
        CAF & & 0.7086 & 0.7124 & 0.4669 & 0.4751 & 0.6672 & 0.6914 & 0.9641 & 0.9914 & 1.3028 & 1.3113 \\
        TF & & \textbf{0.0545} & \textbf{0.0657} & \textbf{0.0778} & \textbf{0.0824} & \textbf{0.1039} & \textbf{0.1133} & \textbf{0.1206} & \textbf{0.1309} & \textbf{0.0594} & \textbf{0.0665} \\
        \bottomrule
    \end{tabular}
    }
\end{table}

\subsection{Training time and scalability of TracingFlow} 

\label{app:scalability}

To evaluate the scalability of TracingFlow, we report the training time required by each method across various datasets, as presented in Table~\ref{tab:APP_training_time}. Experimental results indicate that Flow Matching algorithms based on second-order dynamics (e.g., 3 MSBM and our TracingFlow) generally require more training time compared to first-order dynamics algorithms (e.g., OT-CFM and SF$^2$M). Notably, the training time of TracingFlow is slightly lower than that of 3 MSBM, demonstrating that our method does not introduce additional computational overhead.

\begin{table}[H]
    \centering
    \caption{Training time comparison (in seconds) of different methods across various datasets.}
    \label{tab:APP_training_time}
    \resizebox{0.75\linewidth}{!}{ 
    \begin{tabular}{lccccc}
        \toprule
        \textbf{Method} & \textbf{Simulation} & \textbf{Cite 5D} & \textbf{Cite 100D} & \textbf{SimLineage-3D} & \textbf{RealLineage} \\
        \midrule
        OT-CFM & 12 s & 175 s & 218 s & 13 s & 196 s \\
        SF$^2$M & 26 s & 79 s & 84 s & 34 s & 61 s \\ 
        MMFM & 42 s & 96 s & 179 s & 68 s & 201 s \\
        3 MSBM& 220 s & 495 s & 910 s & 244 s & 855 s \\
        CAF & 94 s & 274 s & 382 s & 256 s & 264 s \\
        HRF & 247 s & 271 s & 359 s & 289 s & 347 s \\
        TF  & 225 s & 319 s & 806 s & 269 s & 742 s \\
        \bottomrule
    \end{tabular}
    }
\end{table}


%

\subsection{Sensitivity Analysis on Minibatch-OT} 

\label{app:minibatch}

In all experiments presented in this paper, the full dataset was utilized to compute the OT plan. However, as the scale of the data increases, the computational cost of the OT plan grows rapidly, making the adoption of the minibatch computation method described in \cref{article:minibatch-OT} inevitable. We conducted a sensitivity analysis on the Cite-5D dataset to evaluate the impact of batch size in minibatch-OT computation on the accuracy of distribution reconstruction. The results are presented in Table~\ref{tab:APP_sensitivity_batchsize}. Experimental results demonstrate that TracingFlow consistently and accurately reconstructs the distributions at each time point when the batch size is adjusted within a certain range ( batch size $ \sim 10^{3}$).

\begin{table}[H]
    \centering
    \caption{Sensitivity analysis of batch size for Minibatch-OT on the Cite-5D dataset.}
    \label{tab:APP_sensitivity_batchsize}
    \resizebox{0.75\linewidth}{!}{ 
    \begin{tabular}{lcccccccc}
        \toprule
        \multirow{2}{*}{\textbf{Batch Size}} & \multicolumn{2}{c}{\textbf{$t=1$}} & \multicolumn{2}{c}{\textbf{$t=2$}} & \multicolumn{2}{c}{\textbf{$t=3$}} & \multicolumn{2}{c}{\textbf{Average}} \\
        \cmidrule(lr){2-3} \cmidrule(lr){4-5} \cmidrule(lr){6-7} \cmidrule(lr){8-9}
         & $\mathcal{W}_{1}$ & $\mathcal{W}_{2}$ & $\mathcal{W}_{1}$ & $\mathcal{W}_{2}$ & $\mathcal{W}_{1}$ & $\mathcal{W}_{2}$ & $\mathcal{W}_{1}$ & $\mathcal{W}_{2}$ \\
        \midrule
        1000 & 0.4344 & 0.4798 & 0.6522 & 0.7183 & 0.7658 & 0.8479 & 0.6175 & 0.6820 \\
        2000 & 0.4196 & 0.4624 & 0.5640 & 0.6251 & 0.7972 & 0.8820 & 0.5936 & 0.6565 \\
       
        Full Dataset & 0.2868 & 0.3185 & 0.5487 & 0.5997 & 0.7689 & 0.8499 & 0.5348 & 0.5894 \\
        \bottomrule
    \end{tabular}
    }
\end{table}
\newpage
\section{Pseudocode for Algorithm}
\label{app:pseudocode_for_algo}
The pseudocode for the training process of Tracing Flow is provided in \cref{alg:training_offline}.

\begin{algorithm}[H]
\caption{Tracing-Flow Matching Training}
\label{alg:training_offline}
\begin{algorithmic}[1]
\REQUIRE 
    Snapshot datasets $\mathcal{D}_k = \{\boldsymbol{x}_k^{(i)}\}_{i=1}^{N_k}$ at times $t_0 < t_1 < \dots < t_K$; \\
    Acceleration Network $\boldsymbol{a}_{\boldsymbol{\theta}}$; Velocity Network $\boldsymbol{v}_{\boldsymbol{\xi}}$; \\
    Hyperparameters: batch size $B_{\text{traj}}$, learning rates $\eta_{\boldsymbol{a}}, \eta_{\boldsymbol{v}}$.
    
\STATE \textbf{Initialize:} 
    Estimated velocity sets $\mathcal{V}_k = \{\boldsymbol{v}_k^{(i)}\}_{i=1}^{N_k} \leftarrow \{\mathbf{0}\}$ for all $k$.

\WHILE{Velocity Estimates Not Converged}
    \STATE \textcolor{blue}{\textbf{Stage 1: Velocity Assignment (Iterative SOAT)}}
    \STATE Initialize velocity accumulators $\hat{\mathcal{V}}_k^{(i)} \leftarrow \emptyset$ for all $k, i$.
    
    \STATE \textcolor{blue}{\textit{// Step 1.1: Solve Optimal Transport}}
    \FOR{$k = 0$ to $K-1$}
        \STATE Compute cost matrix $\mathcal{C}_{i \rightarrow i+1}$ with / w.o. biological prior.
        \STATE Compute optimal coupling $\pi^*_{k \to k+1}$ between $(\boldsymbol{x}_k, \boldsymbol{v}_k)$ and $(\boldsymbol{x}_{k+1}, \boldsymbol{v}_{k+1})$ by solving the SOAT Problem.
        \STATE Compute propagator (transition) matrix $\mathcal{K}^*_{k \to k+1}$ by normalizing $\pi^*_{k \to k+1}$:
        \STATE \quad $\mathcal{K}^*_{k \to k+1}(i, j) = \frac{\pi^*_{k \to k+1}(i, j)}{\sum_{j'} \pi^*_{k \to k+1}(i, j')}$
    \ENDFOR

    \STATE \textcolor{blue}{\textit{// Step 1.2: Trajectory Sampling \& Velocity Update}}
    \STATE Sample $B_{\text{traj}}$ discrete trajectories $T_m = \{(\boldsymbol{x}_{k}^{(s_{m,k})}, \boldsymbol{v}_{k}^{(s_{m,k})})\}_{k=0}^{K}$ for $m=1 \dots B_{\text{traj}}$.
    \STATE \quad $\cdot$ Initial indices $s_{m,0}$ are sampled uniformly.
    \STATE \quad $\cdot$ Subsequent indices $s_{m,k+1} \sim \mathcal{K}^*_{k \to k+1}(s_{m,k}, \cdot)$.
    
    \FOR{$m=1$ to $B_{\text{traj}}$}   
        \STATE Compute optimal path velocities $\{\hat{\boldsymbol{v}}_{k}^{(s_{m,k})}\}_{k=0}^K$ for trajectory $T_m$ using Proposition  \ref{prop:cubic_spline}.
        \STATE Accumulate: $\hat{\mathcal{V}}_k^{(s_{m,k})} \leftarrow \hat{\mathcal{V}}_k^{(s_{m,k})} \cup \{\hat{\boldsymbol{v}}_k^{(s_{m,k})}\}$ for $k = 0 \dots K$.
    \ENDFOR

    \STATE \textcolor{blue}{\textit{// Step 1.3: Average Accumulated Velocities}}
    \FOR{$k = 0$ to $K$, $i = 1$ to $N_k$} 
        \IF{$\hat{\mathcal{V}}_k^{(i)} \neq \emptyset$}
            \STATE $\boldsymbol{v}_k^{(i)} \leftarrow \frac{1}{|\hat{\mathcal{V}}_k^{(i)}|} \sum_{\boldsymbol{v} \in \hat{\mathcal{V}}_k^{(i)}} \boldsymbol{v}$ 
        \ENDIF
    \ENDFOR  
\ENDWHILE

\STATE \textcolor{blue}{\textbf{Stage 2: Final Transport Plan Computation}}
\STATE Compute optimal couplings $\pi^*$ and propagators $\mathcal{K}^*$ using the converged velocities $\mathcal{V}_k$.

\STATE \textcolor{blue}{\textbf{Stage 3: Neural Network Training}}
\WHILE{Not Converged}
    \STATE Sample random time $t \sim \mathcal{U}[t_0, t_K]$ and coupling $\boldsymbol{z} \sim q(\boldsymbol{z})$.
    \STATE Compute Acceleration Matching Loss:
    \STATE \quad $\mathcal{L}_{\text{CAFM}} = \mathbb{E}_{(\boldsymbol{x}, \boldsymbol{v}) \sim \rho_{t}(\cdot|z) , \boldsymbol{z} \sim q(\boldsymbol{z}) , t \sim \mathcal{U}[t_0,t_K]} \Big[ \| \boldsymbol{a}_{\boldsymbol{\theta}}(\boldsymbol{x}, \boldsymbol{v}, t) - \boldsymbol{a}(\boldsymbol{x}, \boldsymbol{v}, t|z) \|^{2} \Big]$
    
    \STATE Compute Initial Velocity Loss:
    \STATE \quad $\mathcal{L}_{\boldsymbol{v}_0} = \mathbb{E}_{(\boldsymbol{x}, \boldsymbol{v}) \sim \rho_{0}} \Big[ \| \boldsymbol{v}_{\boldsymbol{\xi}}(\boldsymbol{x}) - \boldsymbol{v} \|^{2} \Big]$
    
    \STATE Update parameters (Gradient Descent):
    \STATE \quad $\boldsymbol{\theta} \leftarrow \boldsymbol{\theta} - \eta_{\boldsymbol{a}} \nabla_{\boldsymbol{\theta}} \mathcal{L}_{\text{CAFM}}$
    \STATE \quad $\boldsymbol{\xi} \leftarrow \boldsymbol{\xi} - \eta_{\boldsymbol{v}} \nabla_{\boldsymbol{\xi}} \mathcal{L}_{\boldsymbol{v}_0}$
\ENDWHILE
\end{algorithmic}
\end{algorithm}

\section{Discussion}

\label{app:discussion}

\subsection{Relation to other Algorithms}
\label{app:other_algo}

\paragraph{Relation to 3 MSBM} 3 MSBM considers the following stochastic optimal control problem:
\begin{equation}
    \min \mathbb{E}_{\rho_{t}} \int [\|\boldsymbol{a}_{t} \|^{2}] \mathrm{d} t
\end{equation}
\begin{equation}
    \mathrm{d} \boldsymbol{m}_{t} = A\boldsymbol{m}_{t} \mathrm{d} t + \boldsymbol{u}_{t} \mathrm{d} t + g \mathrm{d} \boldsymbol{W}_{t}, \quad \boldsymbol{x}_{n} \sim q_{n} = \int \pi_{n}(\boldsymbol{x}, \boldsymbol{v} ) \mathrm{d} \boldsymbol{v}_{n}
\end{equation}
where $\boldsymbol{m}_{t} = [\boldsymbol{x}_{t}, \boldsymbol{v}_{t}]^{T}$, $A= \begin{bmatrix} 0 & 0 \\ 0 & 1\end{bmatrix}$, and $g = \begin{bmatrix} 0 & 0 \\ 0 & \sigma\end{bmatrix}$. This problem can be viewed as a ``stochastic'' version of the VM-DOAT problem proposed in this paper. 3 MSBM also provides a simulation-free training method to regress the acceleration field. However, 3 MSBM does not explicitly solve this stochastic optimal control problem using the form of optimal coupling plus optimal single-particle trajectories. Its optimal coupling requires iterative solutions similar to Rectified Flow \citep{liu2022flow} : a process of repeatedly selecting the coupling and training $a_{\theta}$. Furthermore, it relies on a heuristic method to estimate the initial velocity (via normal initialization and iteratively solving forward and backward SDEs), rather than incorporating this estimation as part of the optimal control problem. Experiments show that Tracing Flow achieves better distribution reconstruction accuracy than 3MSBM.

\paragraph{Relation to MMFM} In MMFM, the single-particle path is similarly specified by:
\begin{equation}
    \int \|\gamma''(t) \|^{2} \mathrm{d} t, \quad \boldsymbol{x}_{k} = \gamma(t_{k})
\end{equation}
which corresponds to a natural spline. However, MMFM does not solve an optimal control problem over distributions. Furthermore, the velocity field in MMFM remains single-valued with respect to $\boldsymbol{x}$, which prevents it from learning trajectories that cross in the position space.

\paragraph{Relation to OAT-FM} OAT-FM introduces a DOAT problem formulation consistent with this paper; however, the objective of OAT-FM is not to genuinely solve the DOAT problem within the augmented space $\mathcal{X}\times \mathcal{V}$, but rather to fine-tune a pre-trained FM model. It employs the loss:
\begin{align}
\mathcal{L}_{\text{CAFM}}  &= \mathbb{E}_{\pi^{*}[(\boldsymbol{x}_{0}, \boldsymbol{v}_{0}), (\boldsymbol{x}_{1},\boldsymbol{v}_{1})]}\big[ \alpha \| \dfrac{\boldsymbol{x}_{t} - \boldsymbol{x}_{0}}{t} - \dfrac{\boldsymbol{v}_{0} + \boldsymbol{v}_{\boldsymbol{\theta}}}{2} \|^{2} + (1- \alpha) \|\boldsymbol{v}_{\theta}(\boldsymbol{x}_{t}, t) - \boldsymbol{v}_{0} \|^{2} \nonumber\\
&+ \alpha \| \dfrac{\boldsymbol{x}_{1} - \boldsymbol{x}_{t}}{1-t} - \dfrac{ \boldsymbol{v}_{\boldsymbol{\theta}} + \boldsymbol{v}_{1}}{2} \|^{2} + (1- \alpha) \|\boldsymbol{v}_{1} -  \boldsymbol{v}_{\theta}(\boldsymbol{x}_{t}, t)  \|^{2}\big]
\end{align}
to enforce the velocity field along the line segment connecting any two samples to be as parallel to that line as possible. As a fine-tuning approach, it effectively still addresses a distribution transport problem involving only two time points within the position space $\mathcal{X}$, rather than tackling a multi-marginal optimal control problem. Furthermore, OAT-FM does not establish a connection between its loss function and the cost along each trajectory $\dfrac{1}{2} \int_{0}^{1} \|{\ddot \gamma}\|^{2}\mathrm{d} t$, nor does it explicitly learn the acceleration field $\boldsymbol{a}(\boldsymbol{x}, \boldsymbol{v},t)$.

\subsection{Is $\mathcal{S} =  \mathcal{X} \times \mathcal{V}$  a Phase Space?}

\label{app:discussion_phase_space}


In the Conclusion and Limitation section, we noted that if $\boldsymbol{x}$ is regarded as the generalized coordinate and $\boldsymbol{v}$ as its time derivative, the term $\int \frac{1}{2} \|\boldsymbol{a} \|^{2} \mathrm{d} t$ cannot be interpreted as a physical action, as the action in classical mechanics typically does not involve second-order time derivatives. Here, we offer an alternative physical interpretation of this control objective.

\begin{proposition}
Consider a $d$-dimensional second-order dynamical system governed by the equations:
\begin{align}
\dot{\boldsymbol{x}} &= \boldsymbol{v} \\
\dot{\boldsymbol{v}} &= \boldsymbol{a} 
\end{align}
where $\boldsymbol{x}, \boldsymbol{v}, \boldsymbol{a} \in \mathbb{R}^{d}$. The optimal control objective is given by:
\begin{equation}
    \min   \frac{1}{2} \int_{0}^T \|\boldsymbol{a}\|^{2} \mathrm{d} t 
\end{equation}

The evolution of this system under the optimal control law is equivalent to a Hamiltonian system in $4d$-dimensional phase space with the following Hamiltonian:
\begin{equation}
    H = \boldsymbol{p}_{1}^{T} \boldsymbol{q}_{2} + \frac{1}{2} \|\boldsymbol{p}_{2}\|^{2}
\end{equation}
where $\boldsymbol{q}_{1}, \boldsymbol{q}_{2} \in \mathbb{R}^{d}$ are the canonical coordinates, and $\boldsymbol{p}_{1}, \boldsymbol{p}_{2} \in \mathbb{R}^{d}$ are the canonical momenta (note that $\boldsymbol{q}_{1}$ does not appear in the Hamiltonian). Here, $\boldsymbol{q}_{1}$ and $\boldsymbol{q}_{2}$ correspond to $\boldsymbol{x}$ and $\boldsymbol{v}$ in the optimal control problem, respectively, while $\boldsymbol{p}_{2}$ corresponds to $\boldsymbol{a}$.
\end{proposition}

\begin{proof}
    We derive the equations of motion directly using Hamilton's canonical equations:
\begin{equation}
    \dot{\boldsymbol{q}}_{1} = \frac{\partial H}{\partial \boldsymbol{p}_{1}} = \boldsymbol{q}_{2}, \quad \dot{\boldsymbol{q}}_{2} = \frac{\partial H}{\partial \boldsymbol{p}_{2}} = \boldsymbol{p}_{2}, \quad \dot{\boldsymbol{p}}_{1} = - \frac{\partial H}{\partial \boldsymbol{q}_{1}} = \boldsymbol{0}, \quad \dot{\boldsymbol{p}}_{2} = - \frac{\partial H}{\partial \boldsymbol{q}_{2}} = -\boldsymbol{p}_{1}
\end{equation}

This implies that $\boldsymbol{p}_{1}$ is constant. Solving these equations sequentially yields:
\begin{align}
\boldsymbol{p}_{2}(t) &= \boldsymbol{c}_{2}' - \boldsymbol{p}_{1}t \\
\boldsymbol{q}_{2}(t) &= \boldsymbol{c}_{1}'  + \boldsymbol{c}_{2}' t - \frac{1}{2} \boldsymbol{p}_{1} t^{2}   \\
\boldsymbol{q}_{1}(t) &= \boldsymbol{c}_{0}' + \boldsymbol{c}_{1}' + \frac{1}{2} \boldsymbol{c}_{2}' t^{2} - \frac{1}{6} \boldsymbol{p}_{1} t^{3}
\end{align}
Comparing this with the results in  \cref{eq:single_particle_minimizer}, we observe that $\boldsymbol{q}_{1}(t)$ and $\boldsymbol{q}_{2}(t)$ satisfy the same equations as $\boldsymbol{x}(t)$ and $\boldsymbol{v}(t)$, with the undetermined constants determined by the boundary conditions. This completes the proof.
\end{proof}

Consequently, we observe that although the augmented space $\mathcal{S} = \mathcal{X} \times \mathcal{V}$ concatenates position $\boldsymbol{x}$ and velocity (momentum) $\boldsymbol{v}$, it \textbf{cannot} be interpreted as a \textbf{phase space}. Instead, it should be viewed as the \textbf{configuration space} of the classical mechanical system described by the Hamiltonian above. We can attempt to recover the Lagrangian from this Hamiltonian. The canonical equations established that $\dot{\boldsymbol{q}}_{1} = \boldsymbol{q}_{2}$ and $\dot{\boldsymbol{q}}_{2} = \boldsymbol{p}_{2}$. Applying the Legendre transformation:
\begin{align}
L &= \boldsymbol{p}_{1}^{T}\dot{\boldsymbol{q}}_{1} + \boldsymbol{p}_{2}^{T}\dot{\boldsymbol{q}}_{2} - H \nonumber\\
&= \boldsymbol{p}_{1}^{T} (\dot{\boldsymbol{q}}_{1} - \boldsymbol{q}_{2}) + \frac{1}{2} \|\dot{\boldsymbol{q}}_{2}\|^{2}
\end{align}
We find that the term $\boldsymbol{p}_{1}^{T}$ cannot be eliminated; the Lagrangian cannot be expressed solely as a function of the canonical coordinates and their first derivatives. This is a characteristic of constrained systems (in Dirac's theory of constraints, a constrained system is typically defined as one where the Hessian matrix of the Lagrangian is singular. To see this intuitively, if we regard $\boldsymbol{p}_{1}$ in the Lagrangian as $d$ new canonical coordinates, their corresponding canonical momenta are zero, implying the motion is confined to a submanifold of the phase space).

In \cref{eq:single_particle_cost}, we derived the minimum cost for the system to transition from $(\boldsymbol{x}_{1}, \boldsymbol{v}_{1})$ to $(\boldsymbol{x}_{2}, \boldsymbol{v}_{2})$. While this cost is a symmetric positive-definite quadratic form with respect to $\boldsymbol{x}$ and $\boldsymbol{v}$, it does not constitute a metric on the augmented space $\mathcal{S} = \mathcal{X} \times \mathcal{V}$. (If it were a metric, the augmented space would be flat, and the optimal control trajectories—geodesics—would be linear functions of $t$). Maupertuis' principle suggests a close relationship between the metric in the configuration space and the system's Lagrangian \citep{arnol2013mathematical}: if the Lagrangian takes the form:
\begin{equation}
  L(\boldsymbol{q}, \dot{\boldsymbol{q}}) = \frac{1}{2} m_{ij}(\boldsymbol{q}) \dot{q}^{i} \dot{q}^{j} - V(\boldsymbol{q})  
\end{equation}
then the configuration space possesses the metric:
\begin{equation}
    g_{ij}(\boldsymbol{q}) = 2(E - V(\boldsymbol{q})) m_{ij}(\boldsymbol{q})
\end{equation}
where $E$ is the total energy of the particle. For the optimal control problem discussed herein, we have seen that the Lagrangian cannot be written in such a form; therefore, it is fundamentally impossible to equip the configuration space $\mathcal{S} = \mathcal{X} \times \mathcal{V}$ with such a metric.

\subsection{Broader Impacts}
This paper presents work whose goal is to advance the field of machine learning. There are many potential societal consequences of our work, none of which we feel must be specifically highlighted here. However, as our algorithm is applied to single-cell data analysis, the fidelity of the generated trajectories is heavily dependent on the quality of the input data. Consequently, when utilizing our method for biological or medical research purposes, it is critical to employ high-quality datasets, incorporate established and accurate biological priors, and ensure that the algorithmic outputs are rigorously validated by domain experts.

\end{document}